\documentclass[10pt,journal]{IEEEtran}
\usepackage{algpseudocode}
\usepackage{algorithmicx}
\usepackage[ruled,vlined,linesnumbered]{algorithm2e}
\usepackage{amssymb}
\usepackage{booktabs}
\usepackage{graphicx, cite}
\usepackage{subfigure}
\usepackage{amsmath,bm}
\usepackage{amsthm}
\usepackage{float}
\usepackage{stfloats}
\usepackage{url}
\usepackage{multirow}
\usepackage{xcolor}
\usepackage{CJK}
\usepackage{mathrsfs}
\usepackage{enumitem}
\usepackage{physics}
\usepackage{cleveref}

\newcommand{\hytt}[1]{\texttt{\hyphenchar\font=\defaulthyphenchar #1}}

\DeclareMathOperator{\Proj}{Proj}

\newcommand{\vctProj}[2][]{\Proj_{#1}#2}

\SetCommentSty{mycommfont}

\newtheorem{myLem}{Lemma}
\IEEEoverridecommandlockouts

\newcommand{\first}[1]{{\color{black}{#1}}}

\begin{document}
\title{{\fontsize{21}{0}\selectfont LEAF + AIO: Edge-Assisted Energy-Aware Object Detection for Mobile Augmented Reality}}
\author{Haoxin Wang\IEEEauthorrefmark{1}, BaekGyu Kim\IEEEauthorrefmark{1}, Jiang Xie\IEEEauthorrefmark{2}, and Zhu Han\IEEEauthorrefmark{3}\IEEEauthorrefmark{4}\\
\IEEEauthorblockA{\IEEEauthorrefmark{1}Toyota Motor North America (TMNA) R\&D InfoTech Labs, U.S.A.}\\
\IEEEauthorblockA{\IEEEauthorrefmark{2}University of North Carolina at Charlotte, Charlotte, NC 28223, U.S.A.}\\
\IEEEauthorblockA{\IEEEauthorrefmark{3}Department of Electrical and Computer Engineering, University of Houston, Houston, TX 77004, U.S.A.}\\
\IEEEauthorblockA{\IEEEauthorrefmark{4}Department of Computer Science and Engineering, Kyung Hee University, Seoul, South Korea, 446-701\\
E-mail:~haoxin.wang@toyota.com;~baekgyu.kim@toyota.com;~linda.xie@uncc.edu;~hanzhu22@gmail.com}

\thanks{This is a personal copy of the authors. Not for redistribution. The final version of this paper was accepted by IEEE Transactions on Mobile Computing.}}
% \thanks{This work was supported in part by the U.S. National Science Foundation (NSF) under Grant No. 1718666, 1910667, 1910891, 2025284, 2107216, 2128368, and funds from Toyota Motor North America.}}
%\IEEEaftertitletext{\vspace{-2.5\baselineskip}}

\maketitle
\thispagestyle{plain}
\pagestyle{plain}
\pagenumbering{gobble}

\begin{abstract}
Today very few deep learning-based mobile augmented reality (MAR) applications are applied in mobile devices because they are significantly energy-guzzling. In this paper, we design an edge-based energy-aware MAR system that enables MAR devices to dynamically change their configurations, such as CPU frequency, computation model size, and image offloading frequency based on user preferences, camera sampling rates, and available radio resources. Our proposed dynamic MAR configuration adaptations can minimize the per frame energy consumption of multiple MAR clients without degrading their preferred MAR performance metrics, such as latency and detection accuracy. To thoroughly analyze the interactions among MAR configurations, user preferences, camera sampling rate, and energy consumption, we propose, to the best of our knowledge, the first comprehensive analytical energy model for MAR devices. Based on the proposed analytical model, we design a LEAF optimization algorithm to guide the MAR configuration adaptation and server radio resource allocation. An image offloading frequency orchestrator, coordinating with the LEAF, is developed to adaptively regulate the edge-based object detection invocations and to further improve the energy efficiency of MAR devices. Extensive evaluations are conducted to validate the performance of the proposed analytical model and algorithms.
\end{abstract}


\section{Introduction}
\label{sc:introduction}
With the advancement in \textit{Deep Learning} in the past few years, we are able to create intelligent machine learning models to accurately detect and classify complex objects in the physical world. This advancement has the potential to make \textit{Mobile Augmented Reality} (MAR) applications highly intelligent and widely adaptable in various scenarios, such as tourism, education, and entertainment. Thus, implementing MAR applications on popular mobile architectures is a new trend in modern technologies.

However, only a few MAR applications are implemented in mobile devices and are developed based on deep learning frameworks because (i) performing deep learning algorithms on mobile devices is significantly energy-guzzling; (ii) deep learning algorithms are computation-intensive, and executing locally in resource limited mobile devices may not provide acceptable performance for MAR clients \cite{huynh2017deepmon}. To solve these issues, a promising approach is to transfer MAR input image/video frames to an edge server that is sufficiently powerful to execute the deep learning algorithms. 

\textbf{Motivations.} Although compared to running a deep learning algorithm locally on a mobile device, edge-based approaches may extend the device's battery life to certain extents, it is still considerably energy consuming due to conducting multiple pre-processes on the mobile device, such as camera sampling, screen rendering, image conversion, and data transmission \cite{wang2019globe}. For instance, based on the measurement from our developed MAR testbed, a $3000$ mAh smartphone battery is exhausted within approximately $2.3$ hours for executing our developed MAR application which continuously transmits the latest camera sampled image frames to an edge server for object detection. Therefore, the energy efficiency of MAR devices becomes a bottleneck, which impedes MAR clients to obtain better MAR performance. For example, decreasing the energy consumption of an MAR device is always at the cost of reducing the object detection accuracy. Therefore, improving the energy efficiency of MAR devices and balancing the tradeoffs between energy efficiency and other MAR performance metrics are crucial to edge-based MAR systems.

\textbf{Challenges.} An accurate analytical energy model is significantly important for understanding how energy is consumed in an MAR device and for guiding the design of energy-aware MAR systems. However, to the best of our knowledge, there is no existing energy model developed for MAR devices or applications. Developing a comprehensive MAR energy model that is sufficiently general to handle any MAR architecture and application is very challenging. This is because (i) interactions between MAR configuration parameters (e.g., client's CPU frequency and computation model size) and MAR device's energy consumption are complex and lack analytic understandings; (ii) interactions between these configurations and the device's energy consumption may also vary with different mobile architectures.

In addition, designing an energy-aware solution for mobile devices in edge-based MAR systems is also challenging, even after we obtain an analytical energy model. This is because: 
(i) complicated pre-processes on MAR devices increase the complexity of the problem. Compared to conventional computation offloading systems, besides data transmission, there are also a variety of pre-processing tasks (e.g., camera sampling, screen rendering, and image conversion) necessarily to be performed on MAR devices, which are also energy consuming. For example, over $60\%$ of the energy is consumed by camera sampling and screen rendering, based on observations from our developed testbed. Therefore, we have to take into account the energy efficiency of these pre-processing tasks while designing an energy-aware approach for MAR clients. 
(ii) Considering the user preference constraint of individual MAR clients also increases the complexity of the problem. For example, maintaining a high object detection accuracy for a client who prefers a precise MAR while decreasing its energy consumption is very challenging. As stated previously, reducing the energy consumption of the MAR device without degrading other performance metrics is no easy task.
(iii) In practical scenarios, an edge server is shared by multiple MAR clients. Individual client's energy efficiency is also coupled with the radio resource allocation at the edge server. Such a coupling makes it computationally hard to optimally allocate radio resources and improve each client's energy efficiency.

\textbf{Our Contributions.}
In this paper \footnote{\first{This work is an extension of our previous conference paper: 10.1109/INFOCOM41043.2020.9155517}}, we study these research challenges and design a user preference based energy-aware edge-based MAR system. The novel contributions of this paper are summarized as follows:
\begin{enumerate}
\item We design and implement an edge-based object detection for MAR systems to analyze the interactions between MAR configurations and the client's energy consumption. Based on our experimental study, we summarize several insights which can potentially guide the design of energy-aware object detection. 
\item We propose, to the best of our knowledge, the first comprehensive energy model which identifies (i) the tradeoffs among the energy consumption, service latency, and detection accuracy, and (ii) the interactions among MAR configuration parameters (i.e., CPU frequency and computation model size), user preferences, camera sampling rate, network bandwidth, and per frame energy consumption for a multi-user edge-based MAR system.
\item We propose an energy-efficient optimization algorithm, LEAF, which guides MAR configuration adaptations and radio resource allocations at the edge server, and minimizes the per frame energy consumption while satisfying variant clients' user preferences.
\item We develop and implement an image offloading frequency orchestrator that coordinates with the LEAF algorithm to avoid energy-consuming continuous repeated executions of object detections and further improve the energy efficiency of MAR devices.
\end{enumerate}

\section{Related Work}
\label{sc:related}

\textbf{Energy Modeling.} Energy modeling has been widely used for investigating the factors that influence the energy consumption of mobile devices. \cite{xiao2013modeling} and \cite{huang2012close} propose energy models of WiFi and LTE data transmission with respect to the network performance metrics, such as data and retransmission rates, respectively. \cite{shye2009into,walker2016accurate,devogeleer2014modeling,xu2013v,pathak2011fine} propose multiple power consumption models to estimate the energy consumption of mobile CPUs. Tail energy caused by different components, such as disk, Wi-Fi, 3G, and GPS in smartphones has been investigated in \cite{pathak2011fine,pathak2012energy}. However, none of them can be directly applied to estimate the energy consumed by MAR applications. This is because MAR applications introduce a variety of (i) energy consuming components (e.g., camera sampling and image conversion) that are not considered in the previous models and (ii) configuration variables (e.g., computation model size and camera sample rate) that also significantly influence the energy consumption of mobile devices.

\textbf{Computation Offloading.} Most existing research on computation offloading focuses on how to make offloading decisions. \cite{hu2017energy} and \cite{hu2014energy} coordinate the scheduling of offloading requests for multiple applications to further reduce the wireless energy cost caused by the long tail problem. \cite{geng2018energy} proposes an energy-efficient offloading approach for multicore-based mobile devices. \cite{miettinen2010energy} discusses the energy efficiency of computation offloading for mobile clients in cloud computing. However, these solutions cannot be applied to improving the energy efficiency of mobile devices in MAR offloading cases. This is because (i) a variety of pre-processing tasks in MAR executions, such as camera sampling, screen rendering, and image conversion, are not taken into account and (ii) besides the latency constraint that is considered in most existing computation offloading approaches, object detection accuracy is also a key performance metric, which must be considered while designing an MAR offloading solution. In addition, although some existing work proposes to study the tradeoffs between the MAR service latency and detection accuracy \cite{ran2018deepdecision,liu2018edge,hanhirova2018latency,liu2019edge,wang2018smart,apicharttrisorn2019frugal,wang2019globe,wang2019auto,Anik2022icc}, none of them considered (i) the energy consumption of the MAR device and (ii) the whole processing pipeline of MAR (i.e., starting from the camera sampling to obtaining detection results).

\textbf{CPU Frequency Scaling.} Our work is also related to CPU frequency scaling. For modern mobile devices, such as smartphones, CPU frequency and the voltage provided to the CPU can be adjusted at run-time, which is called Dynamic Voltage and Frequency Scaling (DVFS). Prior work \cite{chen2007energy,hu2017energy,kwak2014dynamic,lee2009energy} proposes various DVFS strategies to reduce the mobile device energy consumption under various applications, such as video streaming \cite{hu2017energy} and delay-tolerant applications \cite{kwak2014dynamic}. However, to the best of our knowledge, there has been little effort factoring in the energy efficiency of MAR applications in the context of mobile device DVFS.

\section{Experimental Results on Factors Affecting MAR Client Energy Efficiency}
\label{sc:Experimental Results}

\begin{figure*}[t]
\centering
\includegraphics[width=0.86\textwidth]{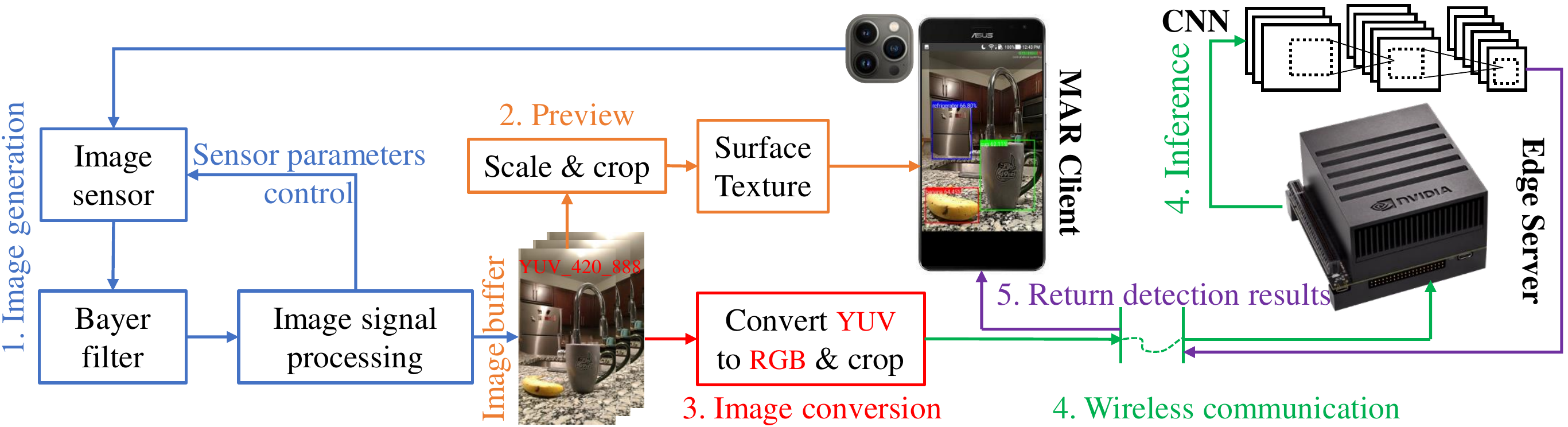}
%\vspace{-0.1in}
\caption{\first{The processing pipeline of the edge-based MAR system developed in this paper \cite{wang2020user, 9274509}.}}
\label{fig:pipeline}
%\vspace{-0.2in}
\end{figure*}

In this section, we describe our preliminary experiments to evaluate the impact of various factors on the energy efficiency of an MAR client, service latency, and detection accuracy in an edge-based MAR system. Specifically, these experimental results provide (i) observations on interactions between energy consumption and MAR configuration parameters, such as MAR client's CPU frequency, computation model size, camera sampling rate, and user preference, (ii) bases of modeling the energy consumption of an MAR client, and (iii) insights on designing an energy-efficient optimization algorithm.     

% \vspace{-0.05 in}
\subsection{Testbed Setup}
Our testbed consists of three major components: MAR client, edge server, and power monitor. Note that this paper focuses on the MAR application in which an MAR client captures physical environmental information through the camera and sends the information to an edge server for object detection. The detailed processing pipeline is shown in Fig. \ref{fig:pipeline}.

\textbf{Processing Pipeline \footnote{\first{GitHub: https://github.com/WINSAC/Mobile-AR-in-Edge-Computing-Client}}} 
\textit{Image generation (phase 1):} The input to this phase is continuous light signal and the output is an image frame. In this phase, the image sensor first senses the intensity of light and converts it into an electronic signal. A Bayer filter is responsible for determining the color information. \first{Then, an image signal processor (ISP) takes the raw data from the image sensor and converts it into a high-quality image frame.} The ISP performs a series of image signal processing operations to deliver a high-quality image, such as noise reduction, color correction, and edge enhancement. In addition, the ISP conducts automated selection of key camera control values according to the environment (e.g., auto-focus (AF), auto-exposure (AE), and auto-white-balance (AWB)). The whole image generation pipeline in our implemented application is constructed based on \hytt{android.hardware.camera2} which is a package that provides an interface to individual camera devices connected to an Android device. \hytt{CaptureRequest} is a class in \hytt{android.hardware.camera2} that constructs the configurations for the capture hardware (sensor, lens, and flash), the processing pipeline, and the control algorithms. Therefore, in our implemented application, we use \hytt{CaptureRequest} to set up image generation configurations. For example, \hytt{CaptureRequest.CONTROL\_AE\_MODE\_OFF} disables AE and \hytt{CaptureRequest.CONTROL\_AE\_TARGET\_FPS\_RANGE} sets the camera FPS (i.e., the number of frames that the camera samples per second). 

\textit{Preview (phase 2):} The input to this phase is a latest generated image frame with \hytt{YUV\_420\_888} format\footnote{For \hytt{android.hardware.camera2}, \hytt{YUV\_420\_888} format is recommended for YUV output \cite{ImageFormat}.} (i.e., the output of Phase 1) and the output is a camera preview rendered on a smartphone's screen with a pre-defined preview resolution. In this phase, the latest generated image frame is first resized to the desired preview resolution and then buffered in a \hytt{SurfaceTexture} which is a class capturing frames from an image stream (e.g., camera preview or video decode) as an OpenGL ES texture. Finally, the camera preview frame in \hytt{SurfaceTexture} is copied and sent to a dedicated drawing surface, \hytt{SurfaceView}, and rendered on the screen. In our implemented application, the preview resolution is set via method \hytt{SurfaceTexture.setDefaultBufferSize()}.

\textit{Image conversion (phase 3):} The input to this phase is a latest generated image frame with \hytt{YUV\_420\_888} format (i.e., the output of Phase 1) and the output is a cropped RGB image frame. In this phase, in order to further process camera captured images (i.e., object detection), an \hytt{ImageReader} class is implemented to acquire the latest generated image frame, where \hytt{ImageReader.OnImageAvailableListener} provides a callback interface for being notified that a new generated image frame is available and method \hytt{ImageReader.acquireLatestImage()} acquires the latest image frame from the \hytt{ImageReader}'s queue while dropping an older image. Additionally, the desired size and format of acquired image frames are configured once an \hytt{ImageReader} is created. In our implemented application, the desired size and the preview resolution are the same and the image format in \hytt{ImageReader} is set to \hytt{YUV\_420\_888}. Furthermore, an image converter is implemented to convert the \hytt{YUV\_420\_888} image to an \hytt{RGB} image, because the input to a CNN-based object detection model must be an \hytt{RGB} image. Finally, the converted \hytt{RGB} image is cropped to the size of the CNN model for object detections.

\textit{Wireless communication \& inference (phase 4):} The input to this phase is a converted and cropped image frame (i.e., the output of Phase 3) and the output is an object detection result. In our implemented application, the object detection result contains one or multiple bounding boxes with labels that identify the locations and classifications of the objects in an image frame. Each bounding box consists of 5 predictions: (x, y, w, h) and a confidence score \cite{yolov3}. The (x, y) coordinates represent the center of the box relative to the bounds of the grid cell. The (h, w) coordinates represent the height and width of the bounding box relative to (x, y). The confidence score reflects how confident the CNN-based object detection model is on the box containing an object and also how accurate it thinks the box is what it predicts. Our implemented application transmits the converted and cropped image frame to the edge server through a wireless TCP socket connection in real time. To avoid having the server process stale frames, the application always sends the latest generated frame to the server and waits to receive the detection result before sending the next frame for processing.

\textit{Detection result rendering (phase 5):} The input to this phase is the object detection result of an image frame (i.e., the output of Phase 4) and the output is a view with overlaid augmented objects (specifically, overlaid bounding boxes and labels in this paper) on top of the physical objects (e.g., a cup).

\textbf{Edge Server.} The edge server is developed to process received image frames and to send the detection results back to the MAR client. We implement an edge server on an Nvidia Jetson AGX Xavier, which connects to a WiFi access point (AP) through a $1$Gbps Ethernet cable. The transmission latency between the server and AP can be ignored. Two major modules are implemented on the edge server \footnote{\first{GitHub: https://github.com/WINSAC/Mobile-AR-in-Edge-Computing-Server}}: (i) the \textit{communication handler} which establishes a TCP socket connection with the MAR device and (ii) the \textit{analytics handler} which performs object detection for the MAR client. In this paper, the analytics handler is designed based on a custom framework called Darknet \cite{darknet13} with GPU acceleration and runs YOLOv3 \cite{yolov3}, a large Convolutional Neural Networks (CNN) model. The YOLOv3 model used in our experiments is trained on COCO dataset \cite{lin2014microsoft} and can detect $80$ classes.

\textbf{MAR Client.} We implement an MAR client on a rooted Android smartphone, Nexus $6$, which is equipped with Qualcomm Snapdragon $805$ SoC (System-on-Chip). The CPU frequency ranges from $0.3$ GHz to $2.649$ GHz.

\textbf{Power Monitor.} The power monitor is responsible for measuring the power consumption of the MAR client. We use Monsoon Power Monitor \cite{Monsoon}, which can sample at $5,000$ Hz, to provide power supply for the MAR device. The power measurements are taken with the screen on, with the Bluetooth/LTE radios disabled, and with minimal background application activity, ensuring that the smartphone's \emph{base power} is low and does not vary unpredictably over time. The base power is defined as the power consumed when the smartphone is connected to the AP without any data transmission activity. The detailed energy measurement methodology is presented in our previous paper \cite{9274509}.

\textbf{Key Performance Metrics.} We define three performance metrics to evaluate the MAR system:
\begin{itemize}
    \item \textit{Per frame energy consumption:} The per frame energy consumption is the total amount of energy consumed in an MAR client by successfully performing the object detection on one image frame. It includes the energy consumed by camera sampling (i.e., image generation), screen rendering (i.e., preview), image conversion, communication, and operating system.
    \item \textit{Service latency:} The service latency is the total time needed to derive the detection result on one image frame. It includes the latency of image conversion, transmission, and inference.
    \item \textit{Accuracy:} The mean average precision (mAP) is a commonly used performance metric to evaluate the detection accuracy of a visual object detection algorithm \cite{everingham2010pascal}, where a greater accuracy is indicated by a higher mAP.
\end{itemize}

\begin{figure}[t]
\centering
\subfigure[]
{\includegraphics[width=0.24\textwidth]{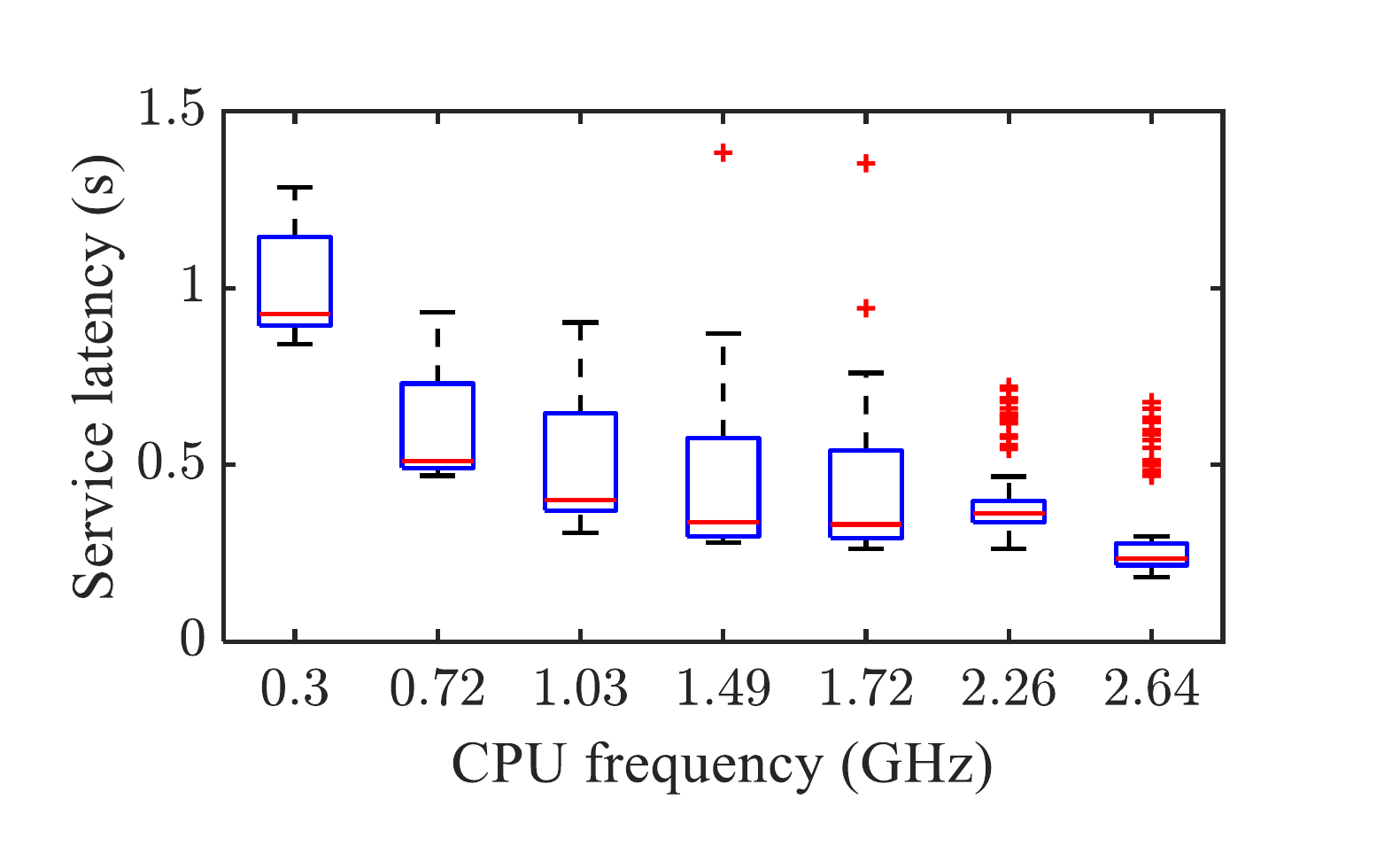}\label{fig:CPUlatency}}
\subfigure[]
{\includegraphics[width=0.232\textwidth]{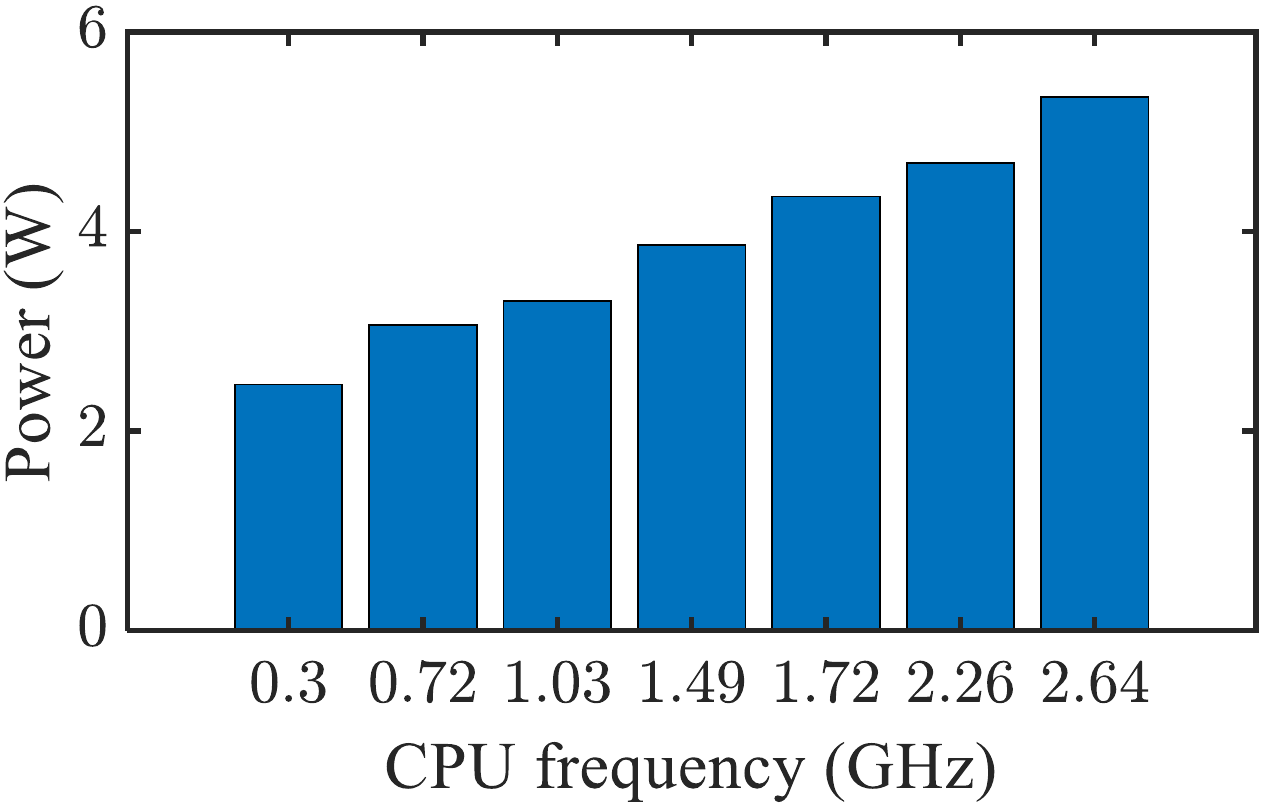}\label{fig:CPUpower}}
\vspace{-0.1 in}
\caption{CPU frequency vs. service latency and power (computation model size: $320^2$ pixels).}
\label{fig:CPUimpact}
\vspace{-0.15in}
\end{figure}

\subsection{The Impact of CPU Frequency on Power Consumption and Service Latency}

\first {In this experiment, we seek to investigate how the CPU frequency impacts the power consumption of the MAR device and the service latency. We set the test device to the \textit{Userspace} Governor and change its CPU frequency manually by writing files in the \hytt{/sys/devices/system/cpu/[cpu\#]/cpufreq} virtual file system with root privilege. The camera FPS is set to 15 and the computation model size is $320^2$ pixels.}
The results are shown in Fig. \ref{fig:CPUimpact}. The lower the CPU frequency, the longer service latency the MAR client derives and the less power it consumes. However, the reduction of the service latency and the increase of the power consumption is disproportional. For example, as compared to $1.03$ GHz, $1.72$ GHz reduces about $2\%$ service latency but increases about $15\%$ power consumption. As compared to $0.3$ GHz, $0.72$ GHz reduces about $60\%$ service latency, but only increases about $20\%$ power consumption. 

\textbf{Insight:} \textit{
\first{This result advocates adapting the client's CPU frequency for the service latency reduction by trading as little increase of the power consumption as possible.}
} 

\begin{figure}[t]
\centering
\subfigure[]
{\includegraphics[width=0.232\textwidth]{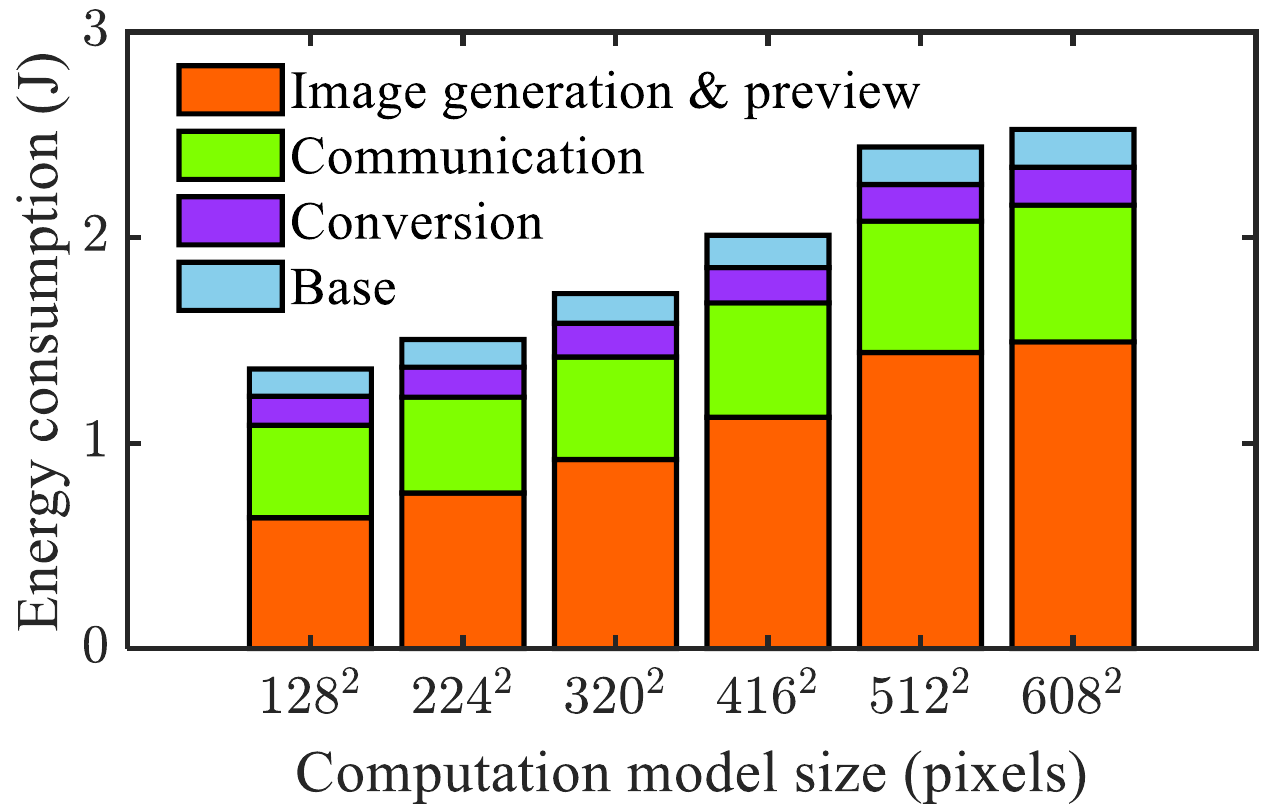}\label{fig:remoteen}}
\subfigure[]
{\includegraphics[width=0.24\textwidth]{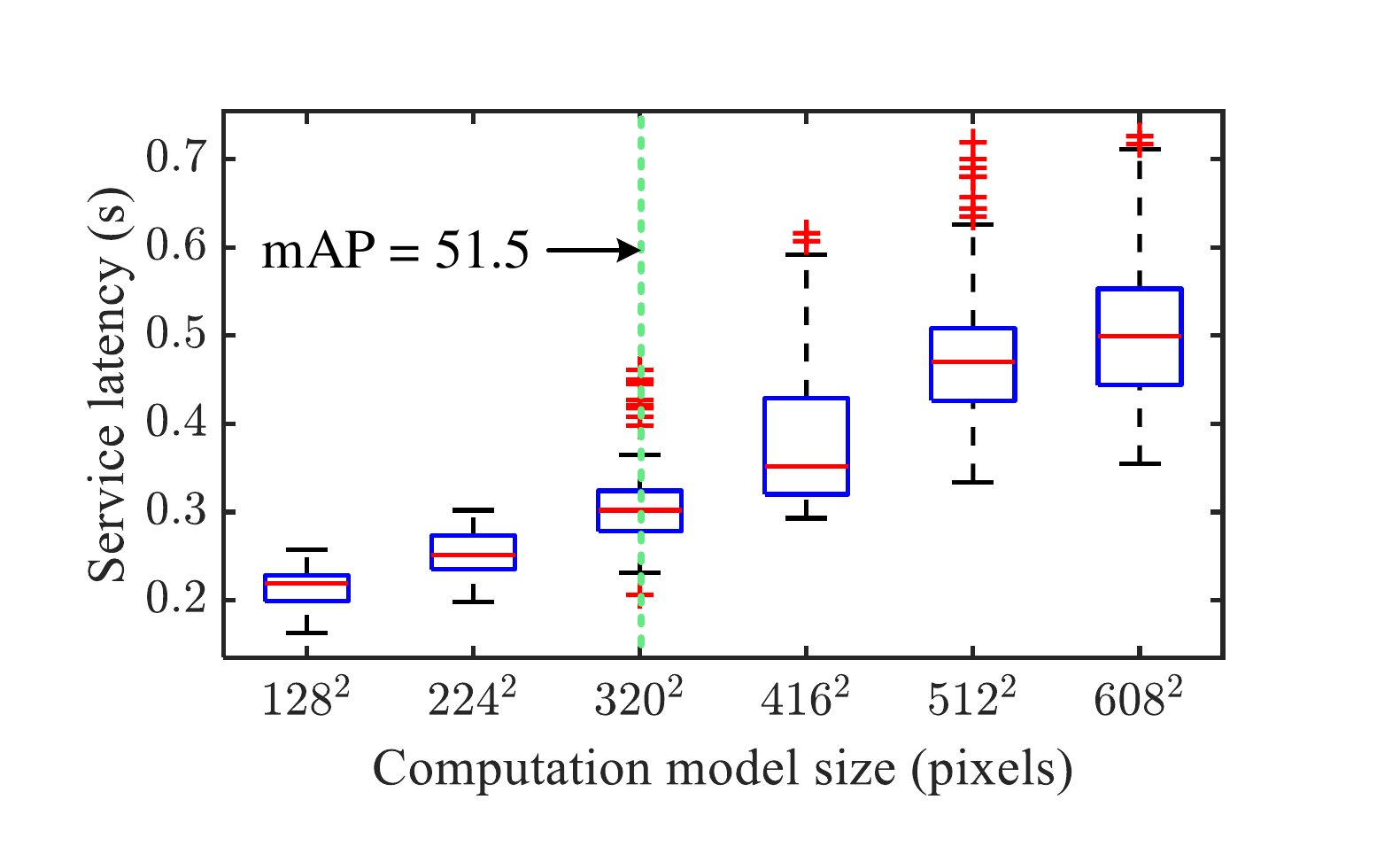}\label{fig:remotela}}
\vspace{-0.1 in}
\caption{Computation model size vs. energy consumption and service latency.}
\label{fig:remotevslocal}   
\vspace{-0.15in}
\end{figure}

\subsection{The Impact of Computation Model Size on Energy Consumption and Service Latency}
\first {In this experiment, we implement object detection based on the YOLOv3 framework with six different CNN model sizes. The test device works on the default CPU governor, \textit{Interactive} and its camera FPS is set to 15.} Increasing the model size always results in a gain of mAP. However, the gain on mAP becomes smaller as the increase of the model sizes \cite{liu2018edge}. In addition, the per frame energy consumption and the service latency boost $85\%$ and $130\%$, respectively, when the model size increases from $128^2$ to $608^2$ pixels, as shown in Figs. \ref{fig:remoteen} and \ref{fig:remotela}. 

\textbf{Insight:} \textit{This result inspires us to trade mAP for the per frame energy consumption and service latency reduction when the model size is large.}

\begin{figure}[t]
\centering
\subfigure[]
{\includegraphics[width=0.232\textwidth]{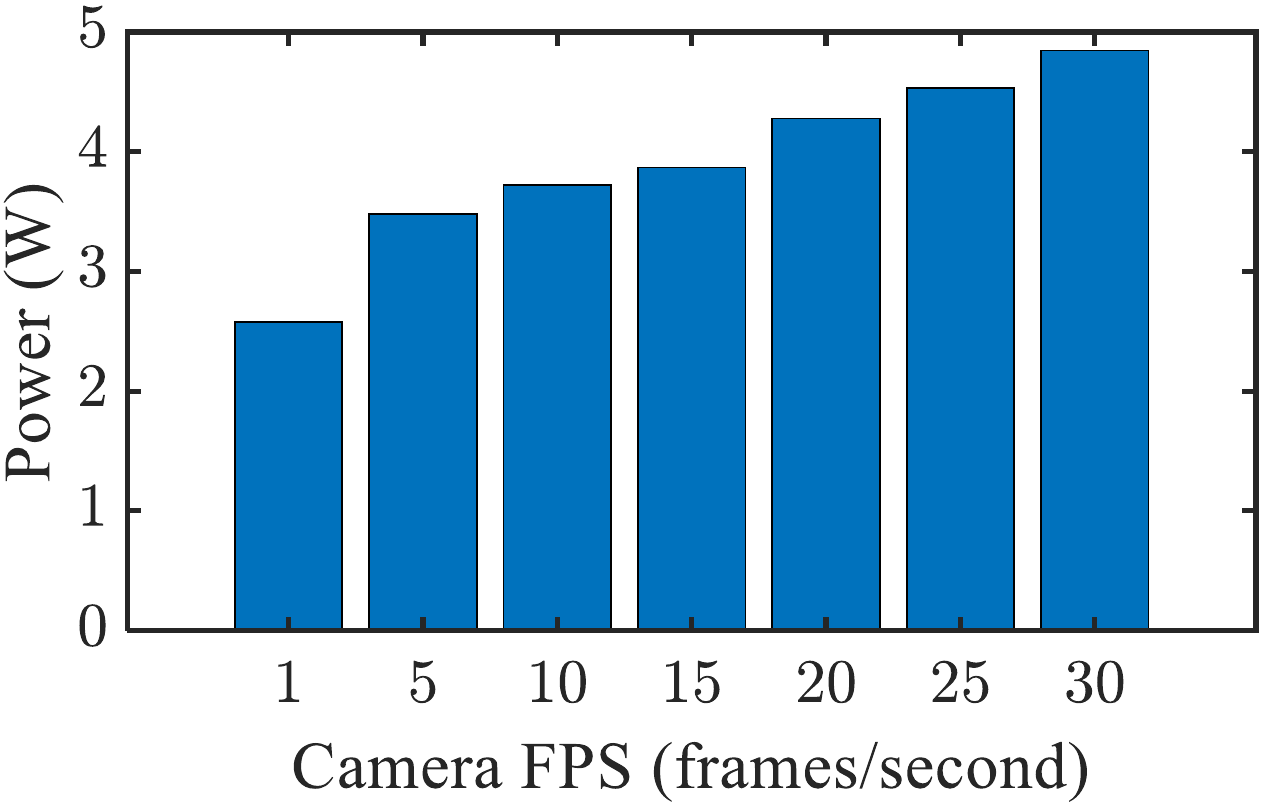}\label{fig:cameraFPSpower}}
\subfigure[]
{\includegraphics[width=0.24\textwidth]{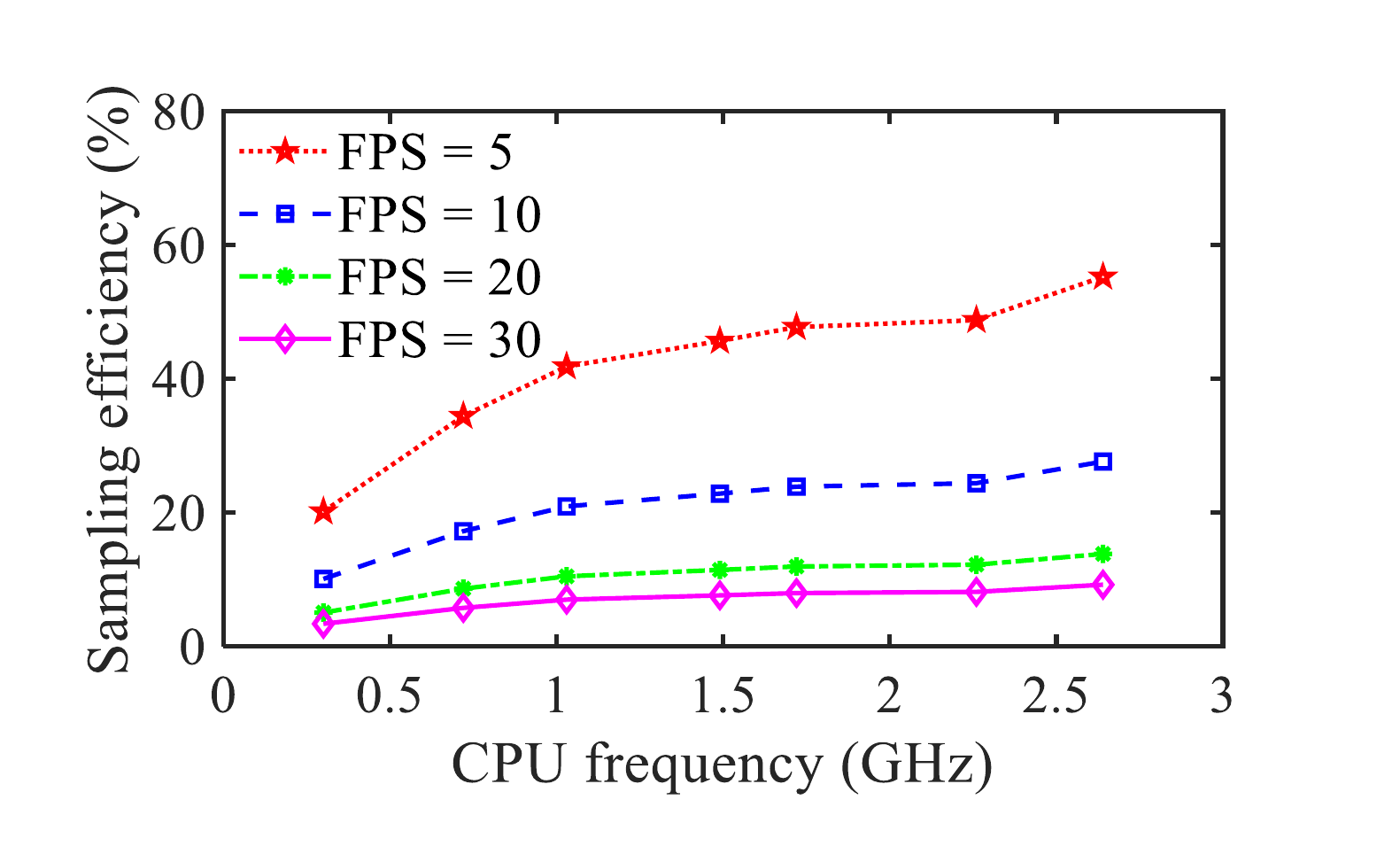}\label{fig:cameraFPSeff}}
\caption{Camera FPS vs. power and sampling efficiency (computation model size: $320^2$ pixels).}
\label{fig:FPSimpact}   
\end{figure}

\subsection{The Impact of Camera FPS on Power Consumption}

\first {In this experiment, we vary the MAR client's camera FPS to explore how it impacts the device's power consumption, where the camera FPS is defined as the number of frames that the camera samples per second. The computation model size is $320^2$ pixels and the default CPU frequency is $1.49$ GHz.}
Fig. \ref{fig:cameraFPSpower} shows that a large camera FPS leads to a high power consumption. However, as shown in Fig. \ref{fig:pipeline}, not every camera captured image frame is sent to the edge server for detection. Because of the need (i) to avoid the processing of stale frames and (ii) to decrease the transmission energy consumption, only the latest camera sampled image frame is transmitted to the server. This may result in the MAR client expending significant reactive power for sampling non-detectable image frames. In Fig. \ref{fig:cameraFPSeff}, we quantify the sampling efficiency with the variation of the camera FPS. As we expected, a large camera FPS leads to a lower sampling efficiency (e.g., less than $2\%$ of the power is consumed for sampling the detectable image frames when the camera FPS is set to $30$). However, in most MAR applications, users usually request a high camera FPS for a smoother preview experience, which is critical for tracking targets in physical environments. Interestingly, increasing CPU frequency can reduce the reactive power for sampling, as shown in Fig. \ref{fig:cameraFPSeff}. 

\textbf{Insight:} \textit{This result demonstrates that when a high camera FPS is requested, increasing CPU frequency can promote the sampling efficiency but may also boost the power consumption. Therefore, finding a CPU frequency that can balance this tradeoff is critical.}

\subsection{User Preference}
\label{ssc:preference}
An MAR client may have variant preferences in different implementation cases, including:
\begin{itemize}
\item \textbf{Latency-preferred.} The MAR application of cognitive assistance \cite{ha2014towards}, where a wearable device helps visually impaired people navigate on a street, may require a low service latency but can tolerate a relatively high number of false positives (i.e., false alarms are fine but missing any potential threats on the street is costly).
\item \textbf{Accuracy-preferred.} An MAR application for recommending products in shopping malls or supermarkets may tolerate a long latency but requires a high detection accuracy and preview smoothness.
\item \textbf{Preview-preferred.} The MAR drawing assistant application \cite{SketchAR}, where a user is instructed to trace virtual drawings from the phone, may tolerate a long latency (i.e., only needs to periodically detect the position of the paper where the user is drawing on) but requires a smooth preview to track the lines that the user is drawing. 
\end{itemize}

\textbf{Insight:} \textit{This observation infers that the user preference's diversity may significantly affect the tradeoffs presented above. For instance, for the accuracy-preferred case, trading detection accuracy for the per frame energy consumption or service latency reduction works against the requirement of the user.}

\section{Proposed System Architecture}
\label{sc:design}

\begin{figure}[t]
\centering
\includegraphics[width=0.48\textwidth]{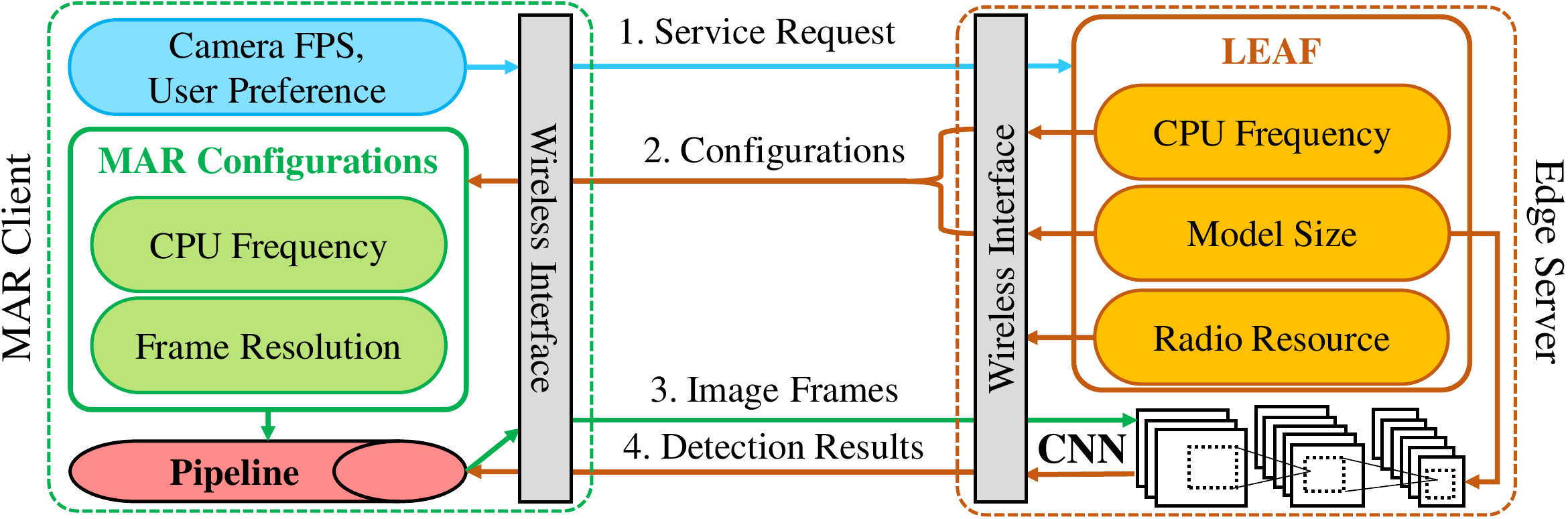}
\caption{Overview of the proposed edge-based MAR system.}
\label{fig:overview}
\end{figure}

\first{Based on the above insights, we propose an edge-based MAR system for object detection to reduce the per frame energy consumption of MAR clients by dynamically selecting the optimal combination of MAR configurations. To derive the optimal MAR configurations, we propose an optimization algorithm (LEAF) that supports \underline{l}ow-\underline{e}nergy, \underline{a}ccurate, and \underline{f}ast MAR applications.}

Fig. \ref{fig:overview} shows the overview of our proposed system. In the first step, MAR clients send their service requests and selected camera FPS and user preferences to an edge server. In the second step, according to the received camera FPS and user preferences, the edge server determines the optimal CPU frequency, computation model size, and allocated radio resource for each MAR client using our proposed LEAF algorithm. The determined CPU frequency and computation model size are then sent back to corresponding MAR clients as MAR configuration messages. In the third step, MAR clients set their CPU frequency to the optimal value and resize their latest camera sampled image frames based on the received optimal computation model size. After the CPU frequency adaptation and image frame resizing, MAR clients transmit their image frames to the edge server for object detection. In the final step, the edge server returns detection results to corresponding MAR clients. The LEAF will be executed when (i) a new MAR client joins the system; (ii) an MAR client leaves the system; or (iii) an MAR client re-sends the service request with a new user preference.

However, designing such a system is challenging. From the presented insights in the previous section, the interactions among the MAR system configuration variables, user preference, camera FPS, and the per frame energy consumption are complicated. (i) Some configuration variables improve one performance metric but impair another one. For example, a lower computation model size reduces the service latency but decreases the detection accuracy. (ii) Some configuration variables may affect the same metric in multiple ways. For example, selecting a higher CPU frequency can decrease the per frame energy consumption by increasing the sampling efficiency, but it increases the CPU power, which conversely increases the per frame energy consumption. Unfortunately, there is no analytical model for characterizing these interactions in the MAR system and it is not possible to design a prominent optimization algorithm without thoroughly analyzing these interactions.

\section{Proposed Analytical Model and Problem Formulation}
\label{sc:models}
In this section, we thoroughly investigate the complicated interactions among the MAR configuration parameters, user preference, camera FPS, and the key performance metrics presented in Section \ref{sc:Experimental Results}. We first propose a comprehensive analytical model to theoretically dissect the per frame energy consumption and service latency. The proposed model is general enough to handle any MAR device and application. Then, using the proposed model, we further model multiple fine-grained interactions, whose theoretical properties are complex and hard to understand, via a data-driven methodology. Finally, based on the above proposed models, we formulate the MAR reconfiguration as an optimization problem. 
\subsection{Analytics-based Modeling Methodology}
\label{ssc:Analytics-based}

\begin{figure*}[t]
\centering
\includegraphics[width=0.98\textwidth]{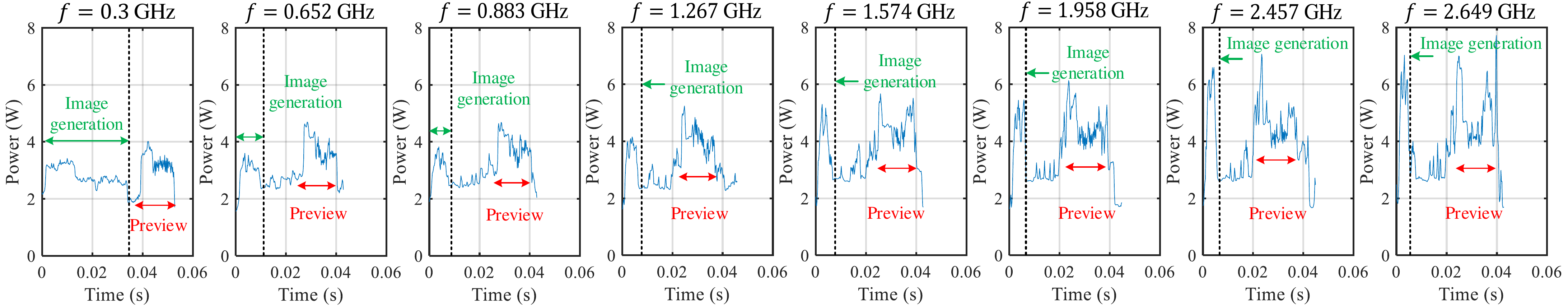}
\vspace{-0.15 in}
\caption{The impact of CPU frequency on the power consumption of image generation and preview.}
\label{fig:imagegeneration}
\vspace{-0.2 in}
\end{figure*}

We consider an edge-based MAR system with $K$ MAR clients and one edge server, where clients are connected to the edge server via a single-hop wireless network. Denote $\mathcal{K}$ as the set of MAR clients. The per frame service latency of the $k$th MAR client can be defined as

\begin{small}
\begin{equation}
    L^{k} = L^{k}_{cv} + L^{k}_{tr} + L^{k}_{inf},
% \vspace{-0.05 in}
\end{equation}
\end{small}\noindent
where $L^{k}_{cv}$ is the image conversion latency caused by converting a buffered camera captured image frame from YUV to RGB; $L^{k}_{tr}$ is the transmission latency incurred by sending the converted RGB image frame from the $k$th client to its connected edge server; and $L^{k}_{inf}$ is the inference latency of the object detection on the server. According to the MAR pipeline depicted in Fig. \ref{fig:pipeline}, the per frame energy consumption of the $k$th MAR client can be defined as

\begin{small}
\begin{equation}
    E^{k} = E^{k}_{img} + E^{k}_{cv} + E^{k}_{com} + E^{k}_{bs},
    % \vspace{-0.05 in}
\end{equation}
\end{small}\noindent
where $E^{k}_{img}$ is the image generation and preview energy consumption incurred by image sampling, processing, and preview rendering; $E^{k}_{cv}$ is the image conversion energy consumption; $E^{k}_{com}$ is the wireless communication energy consumption, which includes four phases: \textit{promotion, data transmission, tail}, and \textit{idle}; and $E^{k}_{bs}$ is the MAR device base energy consumption.

\textbf{The Model of Image Generation and Preview.} Image generation is the process that an MAR client transfers its camera sensed continuous light signal to a displayable image frame. Preview is the process of rendering the latest generated image frame on the client's screen. As these two processes are executed in parallel with the main thread, their execution delays are not counted in the per frame service latency. 

As depicted in Fig. \ref{fig:remoteen}, the energy consumption of image generation and preview is the largest portion of the per frame energy consumption. To understand how energy is consumed in image generation and preview and what configuration variables impact it, we conduct a set of experiments. We find that \textit{the power consumption of image generation and preview highly depends on the CPU frequency.} Fig. \ref{fig:imagegeneration} shows the power consumption of image generation and preview under different CPU frequencies, where the camera FPS is set to $15$. A higher CPU frequency results in a higher average power consumption. In addition, the image generation delay is also closely related to the CPU frequency, where a higher CPU frequency always leads to a shorter delay. However, the delay of rendering a preview is only related to the GPU frequency, which is out of the scope of this paper. Thus, we consider the preview delay as a fixed value with any CPU frequencies. We model the energy consumption of the $k$th MAR client's image generation and preview within a service latency as
\begin{small}
\begin{equation}
    E^{k}_{img} = \left(\int_{0}^{t^{k}_{gt}(f_{k})} P^{k}_{gt}(f_{k})\, dt + \int_{0}^{t_{prv}} P^{k}_{prv}(f_{k})\, dt\right)\cdot fps_{k}\cdot L^{k},
    \vspace{-0.05 in}
\end{equation}
\end{small}\noindent
where $P^{k}_{gt}$, $P^{k}_{prv}$, $t^{k}_{gt}$, $t_{prv}$ are the power consumption of image generation, preview, the delay of image generation, and preview, respectively. $f_{k}$ is the CPU frequency. $fps_{k}$ is the camera FPS. $P^{k}_{gt}$, $P^{k}_{prv}$, and $t^{k}_{gt}$ are functions of $f_{k}$.

\textbf{The Model of Image Conversion.} Image conversion is processed through the MAR client's CPU; and hence, the conversion latency and power consumption highly depend on the CPU frequency. We define $L^{k}_{cv}$ and $E^{k}_{cv}$ a function of $f_{k}$. Therefore, the major source of the power consumption of the image conversion is the CPU computation. The power consumption of mobile CPUs can be divided into two components, $P^{k}_{cv} = P_{leak} + P^{k}_{dynamic}$ \cite{devogeleer2014modeling}, where $P_{leak}$ is independent and $P^{k}_{dynamic}$ is dependent upon the CPU frequency. (i) $P_{leak}$ is the power originating from leakage effects and is in essence not useful for the CPU's purpose. In this paper, we consider $P_{leak}$ a constant value $\epsilon$. (ii) $P^{k}_{dynamic}$ is the power consumed by the logic gate switching at $f_{k}$ and is proportional to $V_{k}^{2}f_{k}$, where $V_{k}$ is the supply voltage for the CPU. Due to the DVFS for the power saving purpose, e.g., a higher $f_{k}$ will be supplied by a larger $V_{k}$, each $f_{k}$ matches with a specific $V_{k}$, where $V_{k} \propto (\alpha_{1}f_{k}+\alpha_{2})$. $\alpha_{1}$ and $\alpha_{2}$ are two positive coefficients. Thus, the energy consumption of converting a single image frame of the $k$th MAR client can be modeled as
\begin{small}
\begin{equation}
E^{k}_{cv} = P^{k}_{cv}L^{k}_{cv} = (\alpha^{2}_{1}f^{3}_{k}+2\alpha_{1}\alpha_{2}f^{2}_{k}+\alpha_{2}f_{k}+\epsilon)\cdot L^{k}_{cv}(f_{k}).
\vspace{-0.05 in}
\label{eq:conversionPower}
\end{equation}
\end{small}\noindent

\textbf{The Model of Wireless Communication and Inference.}
Intuitively, the wireless communication latency is related to the data size of the transmitted image frame (determined by the frame resolution) and wireless data rate. As the data size of detection results is usually small, we do not consider the latency caused by returning the detection results \cite{liu2018edge}. In this paper, we use $s_{k}^{2}$ (pixels) to represent the computation model size of the $k$th MAR client. The client must send image frames whose resolutions are not smaller than $s_{k}^{2}$ to the edge server to obtain the corresponding detection accuracy. Thus, the most efficient way is to transmit the image frame with the resolution of $s_{k}^{2}$ to the server. Denote $\sigma$ as the number of bits required to represent the information carried by one pixel. The data size of an image frame is calculated as $\sigma s_{k}^{2}$ bits. Let $B_{k}$ be the wireless bandwidth derived by the $k$th MAR client. We model the transmission latency of the $k$th client as
\begin{small}
\begin{equation}
    L^{k}_{tr} = \frac{\sigma s_{k}^{2}}{R_{k}},
    %\vspace{-0.05 in}
\end{equation}
\end{small}\noindent
where $R_{k}$ is the average wireless data rate of the $k$th client, which is a function of $B_{k}$.

In addition to the computation model size and wireless bandwidth, the transmission latency is also determined by the MAR client's CPU frequency. This is because the image transmission uses TCP as the transport layer protocol, and TCP utilizes substantial CPU capacity to handle congestion avoidance, buffer, and retransmission requests. For example, when the CPU frequency is low, the remaining CPU capacity may not be adequate to process the TCP task; and thus, the TCP throughput is decreased. Therefore, $R_{k}$ is also a function of $f_{k}$, i.e., $R_{k}(B_{k},f_{k})$. In this paper, $R_{k}(B_{k},f_{k})$ is defined as
\begin{small}
\begin{equation}
    R_{k}(B_{k},f_{k}) = r_{k}^{max}(B_{k})\cdot r_{k}^{*}(f_{k}),
    %\vspace{-0.05 in}
\end{equation}
\end{small}\noindent
where $r_{k}^{max}(B_{k})$ is the network throughput, which is not affected by the variation of the MAR client's CPU frequency, and is only determined by the bandwidth (more comprehensive model of this part can be found in \cite{xiao2013modeling}, which is out of the scope of this paper); $r_{k}^{*}(f_{k})$ represents the impact of the CPU frequency on the TCP throughput. 

\begin{figure}[t]
\centering
\subfigure[]
{\includegraphics[width=0.23\textwidth]{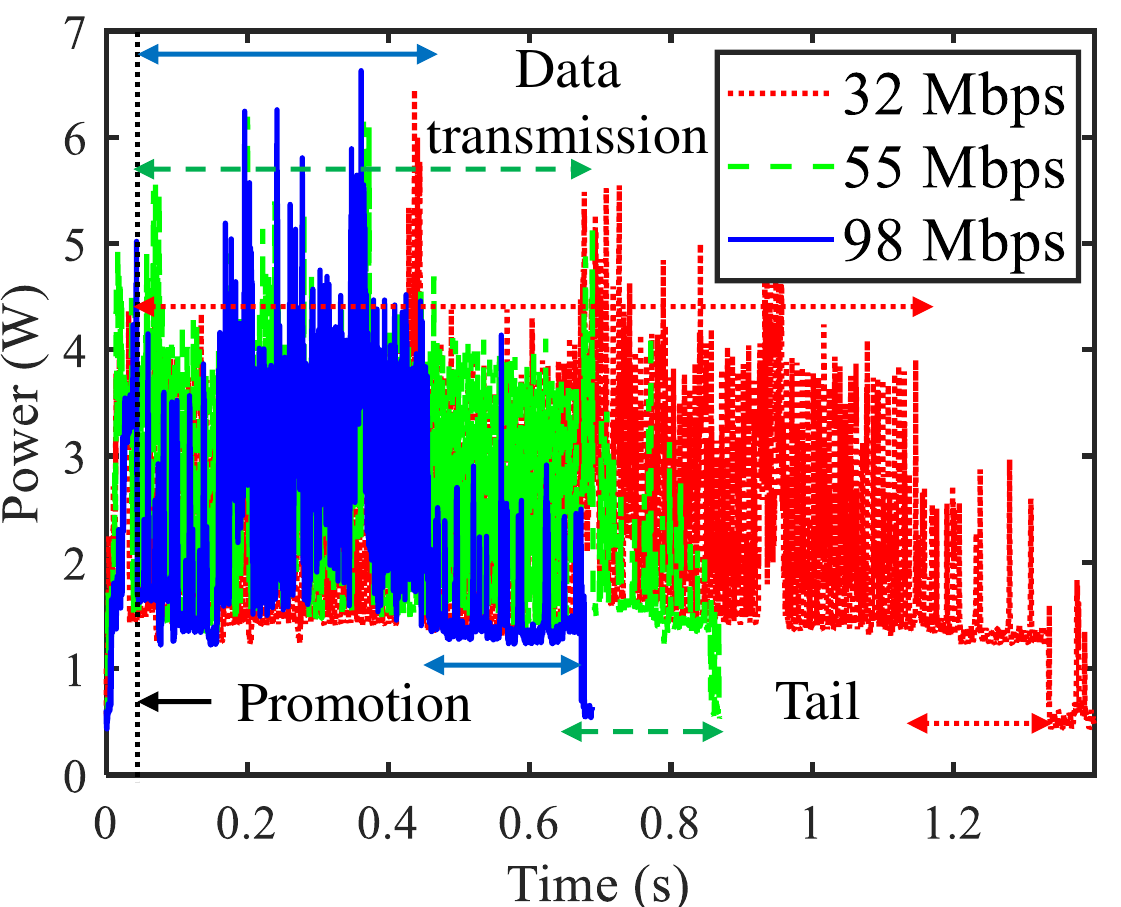}\label{fig:multitrans}}
\subfigure[]
{\includegraphics[width=0.24\textwidth]{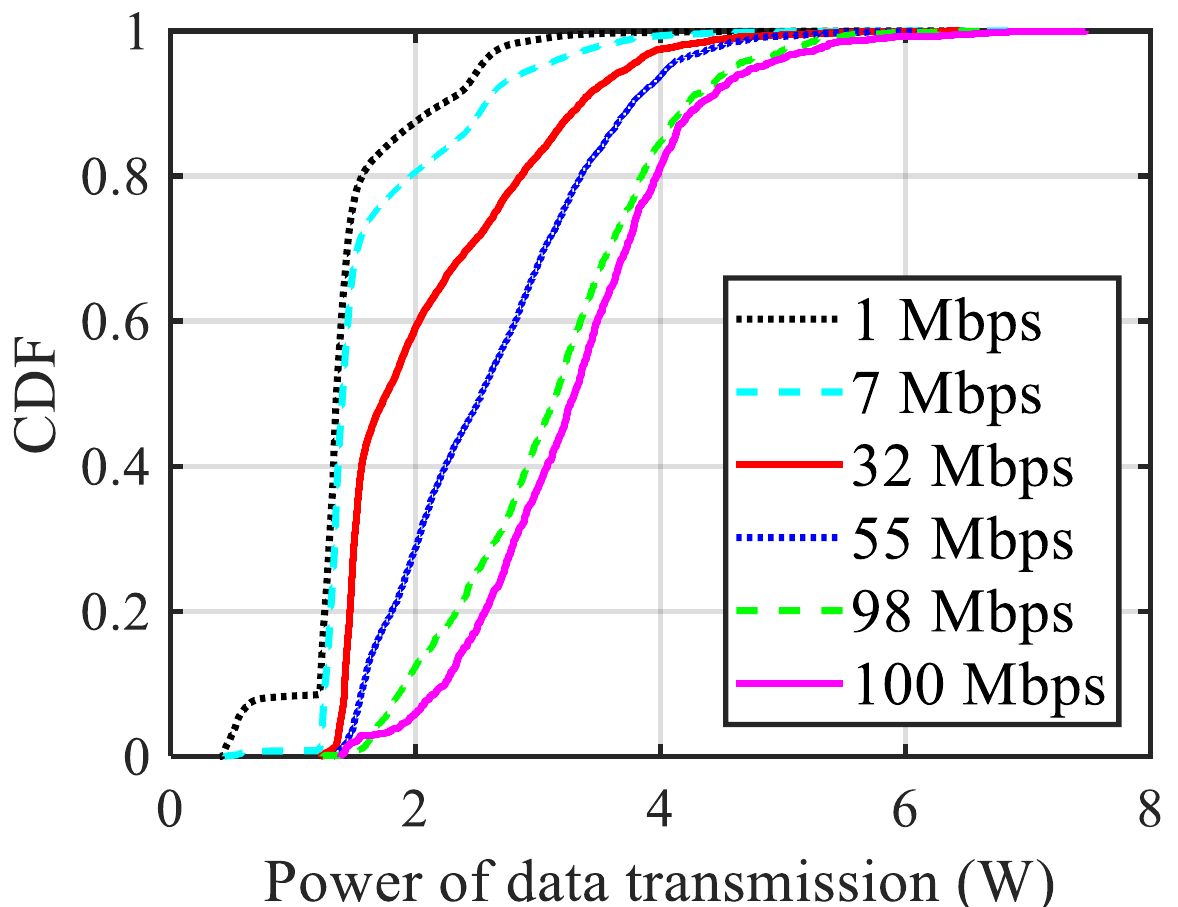}\label{fig:throughputpowerCDF}}
\vspace{-0.1in}
\caption{MAR client's wireless interface power consumption.}
\label{fig:wirelessmodel}   
\vspace{-0.25in}
\end{figure}

In WiFi networks, when transmitting a single image frame, the MAR client's wireless interface experiences four phases: promotion, data transmission, tail, and idle. When an image transmission request comes, the wireless interface enters the promotion phase. Then, it enters the data transmission phase to send the image frame to the edge server. After completing the transmission, the wireless interface is forced to stay in the tail phase for a fixed duration and waits for other data transmission requests and the detection results. If the MAR client does not receive the detection result in the tail phase, it enters the idle phase and waits for the feedback from its associated edge server. Fig. \ref{fig:wirelessmodel} depicts the measured power consumption of the MAR client that transmits a $3,840\times2,160$ pixel image with different throughput. We find that the average power consumption of the data transmission phase increases as the throughput grows. However, the average power consumption and the duration of promotion and tail phases are almost constant. Therefore, we model the energy consumption of the $k$th MAR client in the duration that starts from the promotion phase to obtaining the object detection result as
\begin{small}
\begin{equation}
    E^{k}_{com} = P^{k}_{tr}(R_{k}(B_{k},f_{k}))L^{k}_{tr} + P^{k}_{idle}t^{k}_{idle} + P_{pro}t_{pro} + P_{tail}t_{tail},
    \vspace{-0.05 in}
\end{equation}
\end{small}\noindent
where $P^{k}_{tr}$, $P^{k}_{idle}$, $P_{pro}$, and $P_{tail}$ are the average power consumption of the data transmission, idle, promotion, and tail phases, respectively; $t^{k}_{idle}$, $t_{pro}$, and $t_{tail}$ are the durations of the idle, promotion, and tail phases, respectively,
\begin{small}
\begin{equation}
P_{idle}^{k}t_{idle}^{k} = 
\begin{cases}
    0, & L_{inf}^{k}(s_{k}^{2}) \leq t_{tail},\\
    P_{bs}^{k}\cdot(L_{inf}^{k}(s_{k}^{2})-t_{tail}), & L_{inf}^{k}(s_{k}^{2}) > t_{tail},
    \vspace{-0.05 in}
\end{cases}
\end{equation}
\end{small}\noindent
where $P_{bs}^{k}$ is the MAR device's base power consumption; $L_{inf}^{k}(s_{k}^{2})$ is the inference latency on the edge server, which is determined by the computation model size \cite{liu2018edge}. Note that our proposed wireless communication model can also be used in other wireless networks (e.g., LTE).

\textbf{The Model of Base Energy.} In this paper, the base energy consumption is defined as the energy consumed by the MAR clients' CPU without any workloads, except running its operating system, and the energy consumed by the screen without any rendering. Because the screen's brightness is not a critical factor that affects the object detection performance, it is considered as a constant value in our proposed power model. Thus, the base power consumption is only a function of the CPU frequency. We model the base energy consumption of the $k$th MAR client within a service latency as
\begin{small}
\begin{equation}
    E^{k}_{bs} = 
    \begin{cases}
    P^{k}_{bs}(f_{k})\cdot L^{k}, &  L_{inf}^{k}(s_{k}^{2}) \leq t_{tail},\\
    P^{k}_{bs}(f_{k})\cdot (L^{k}-L_{inf}^{k}(s_{k}^{2})+t_{tail}), &  L_{inf}^{k}(s_{k}^{2}) > t_{tail}.
    \vspace{-0.05 in}
    \end{cases}
    \label{eq:basepower}
\end{equation}
\end{small}\noindent

\subsection{Regression-based Modeling Methodology}

\begin{figure*}[t]
%\vspace{-0.2in}
\centering
\subfigure[]
{\includegraphics[width=0.195\textwidth]{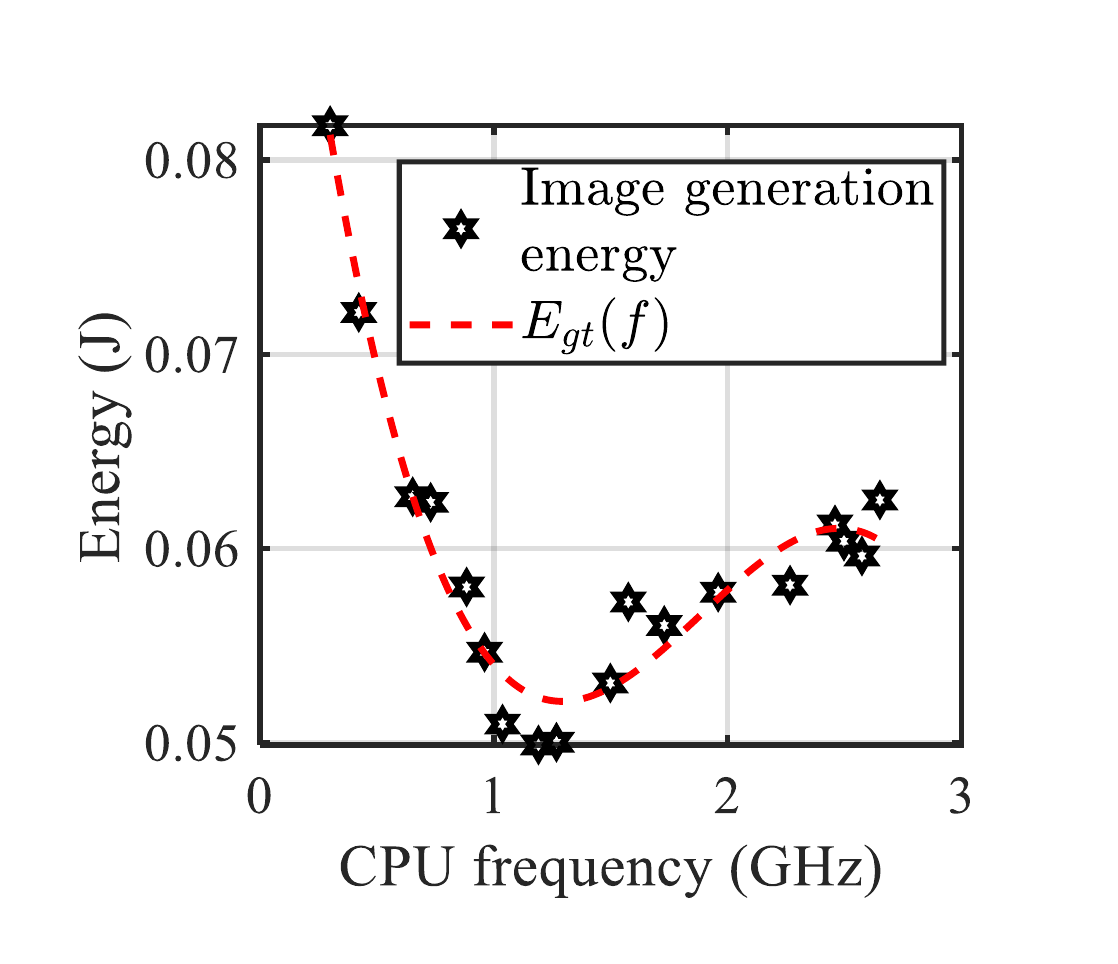}\label{fig:Egt}}
\subfigure[]
{\includegraphics[width=0.195\textwidth]{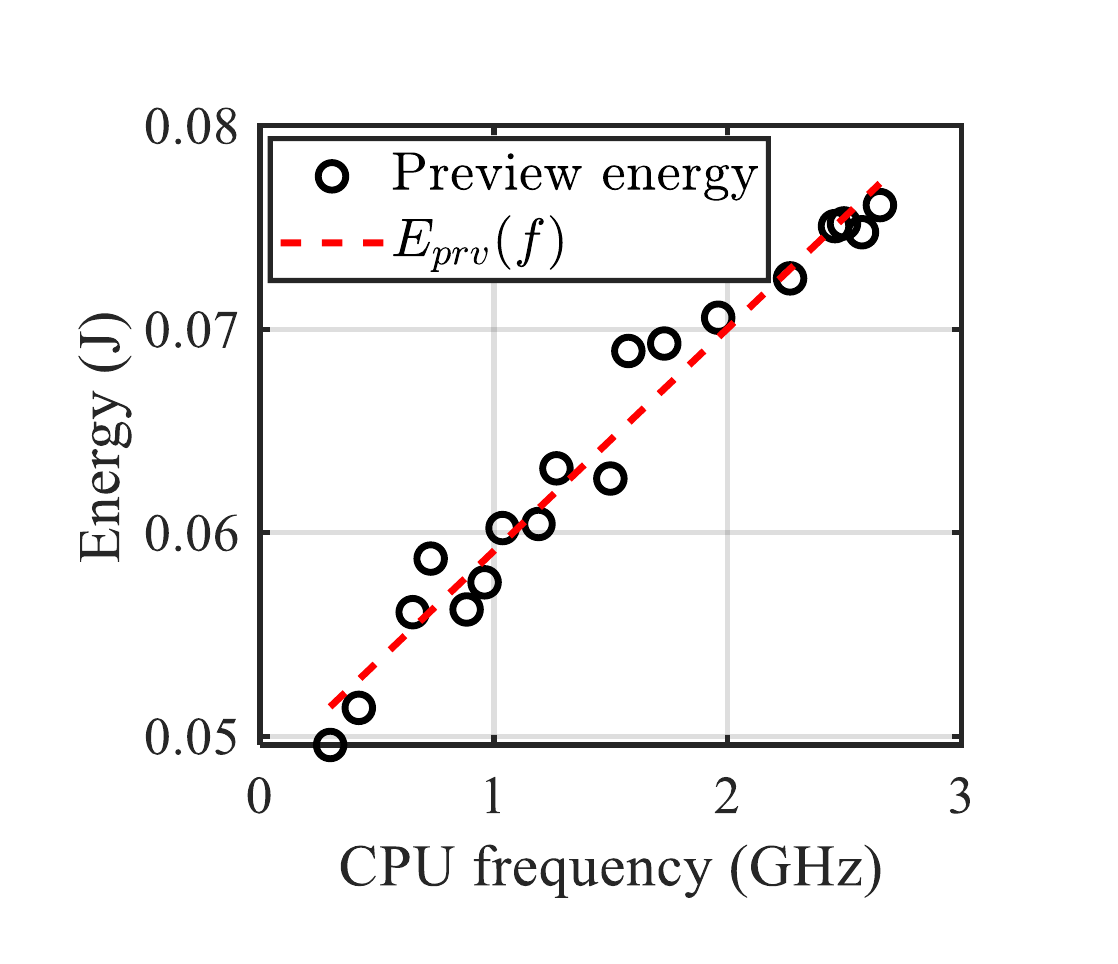}\label{fig:Eprv}}
\subfigure[]
{\includegraphics[width=0.195\textwidth]{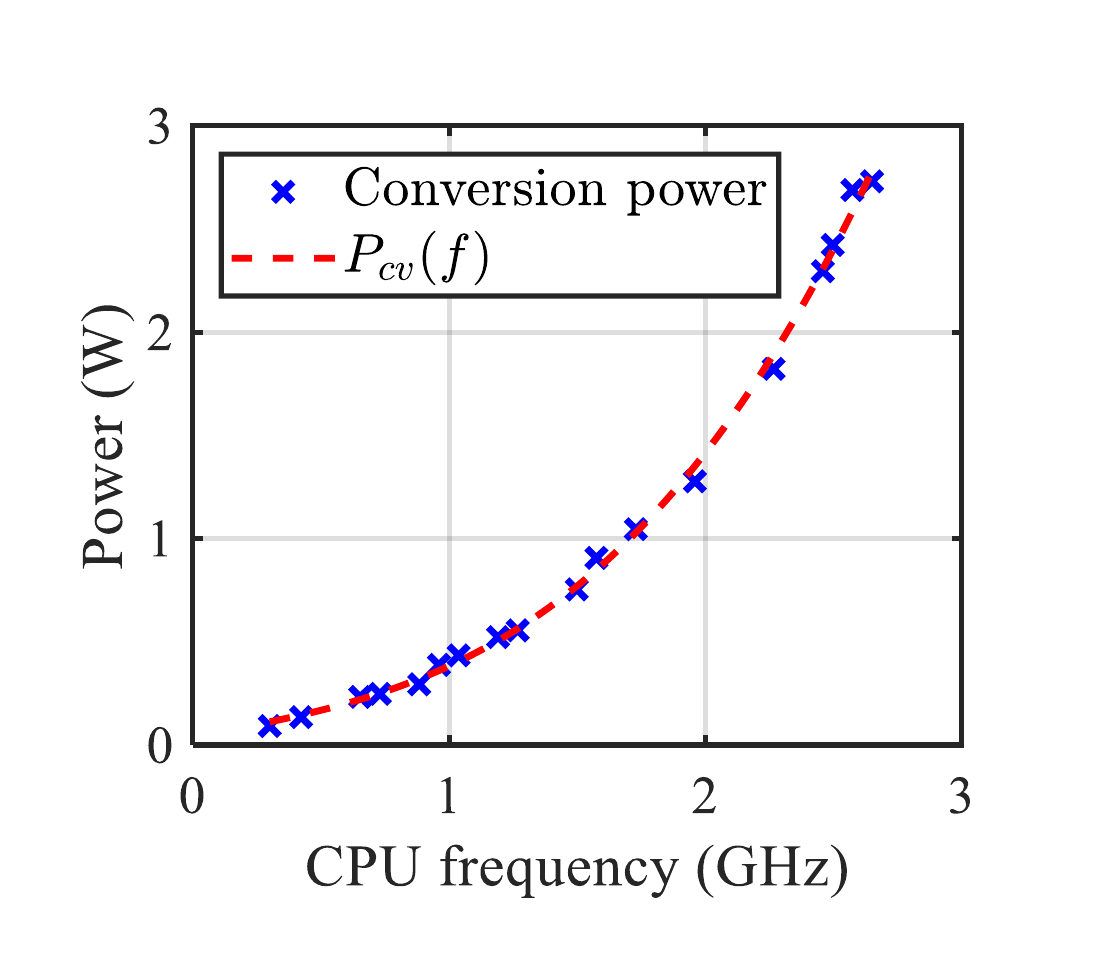}\label{fig:Pcv}}
\subfigure[]
{\includegraphics[width=0.195\textwidth]{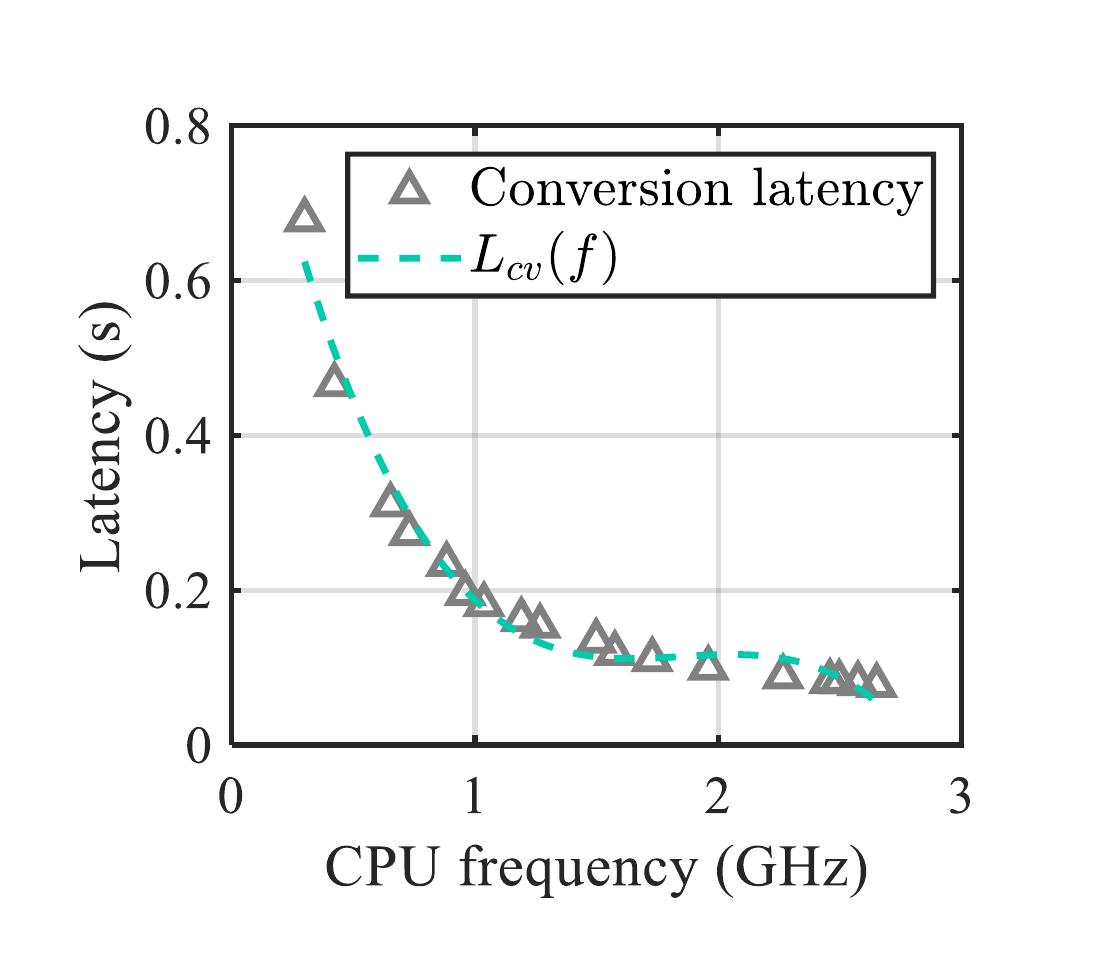}\label{fig:Lcv}}
\subfigure[]
{\includegraphics[width=0.195\textwidth]{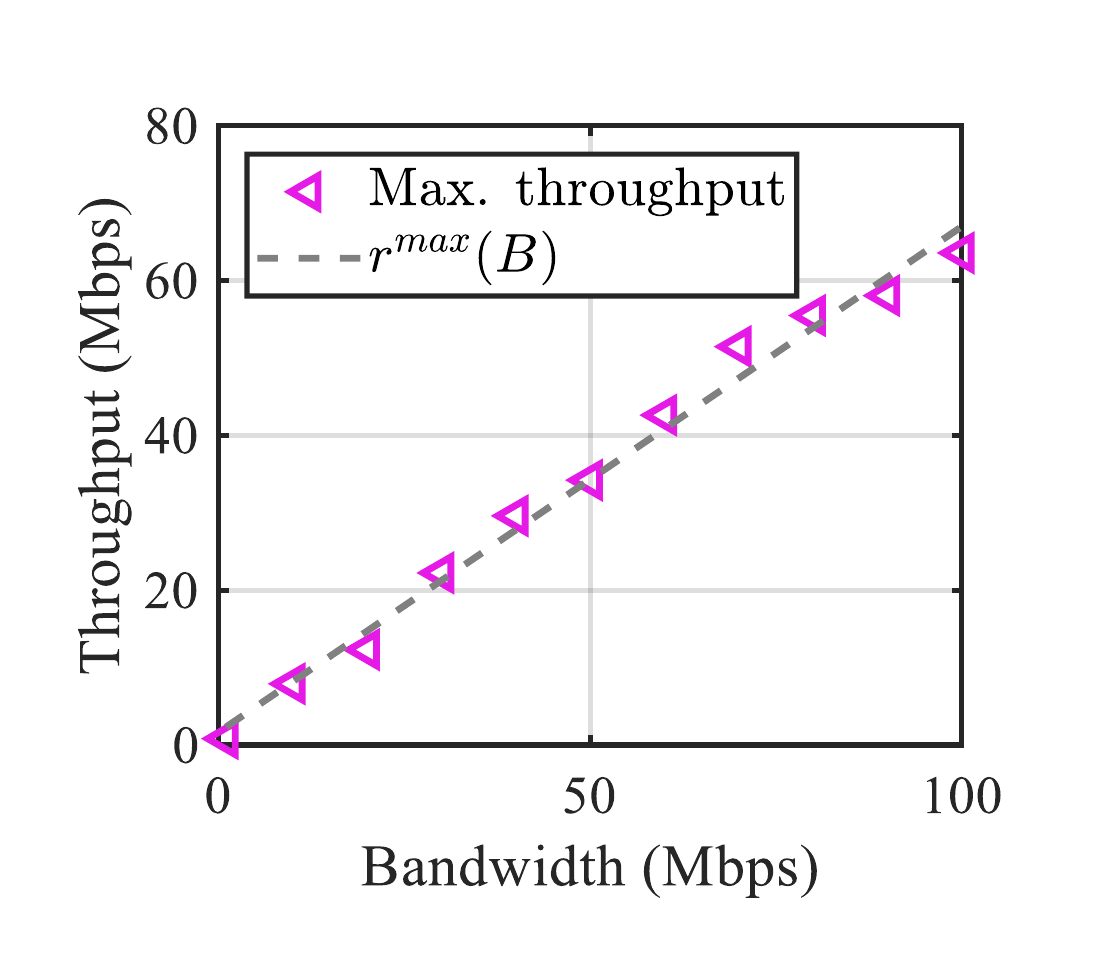}\label{fig:rb}}
\subfigure[]
{\includegraphics[width=0.195\textwidth]{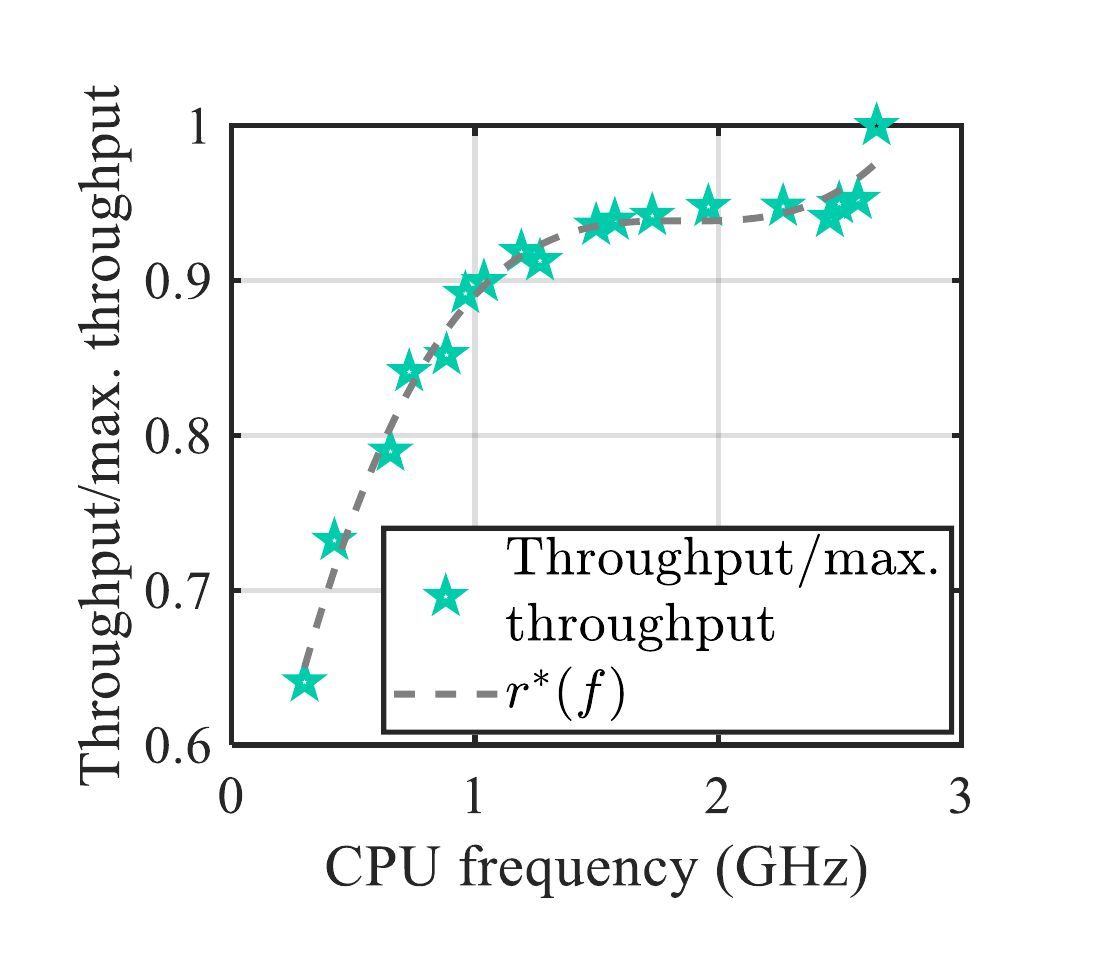}\label{fig:rf}}
\subfigure[]
{\includegraphics[width=0.195\textwidth]{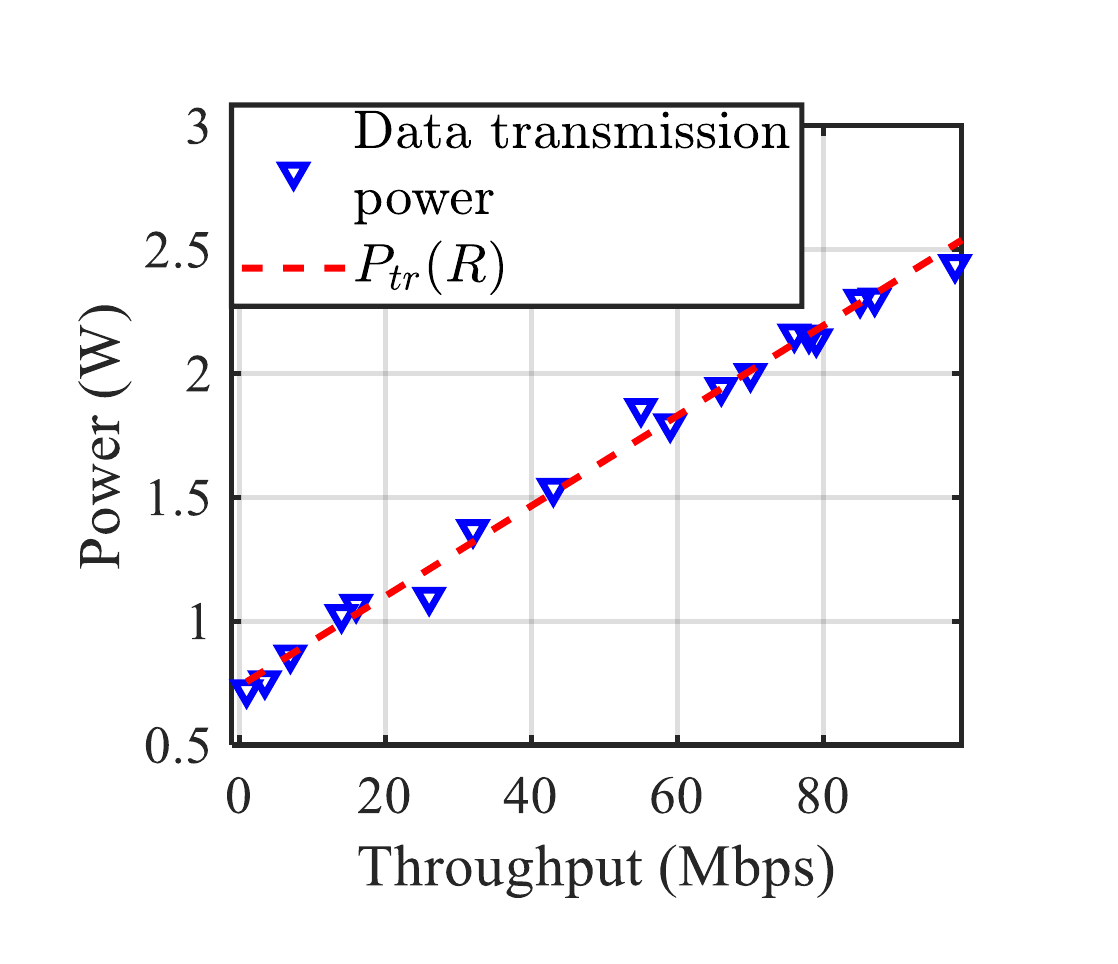}\label{fig:Ptr}}
\subfigure[]
{\includegraphics[width=0.195\textwidth]{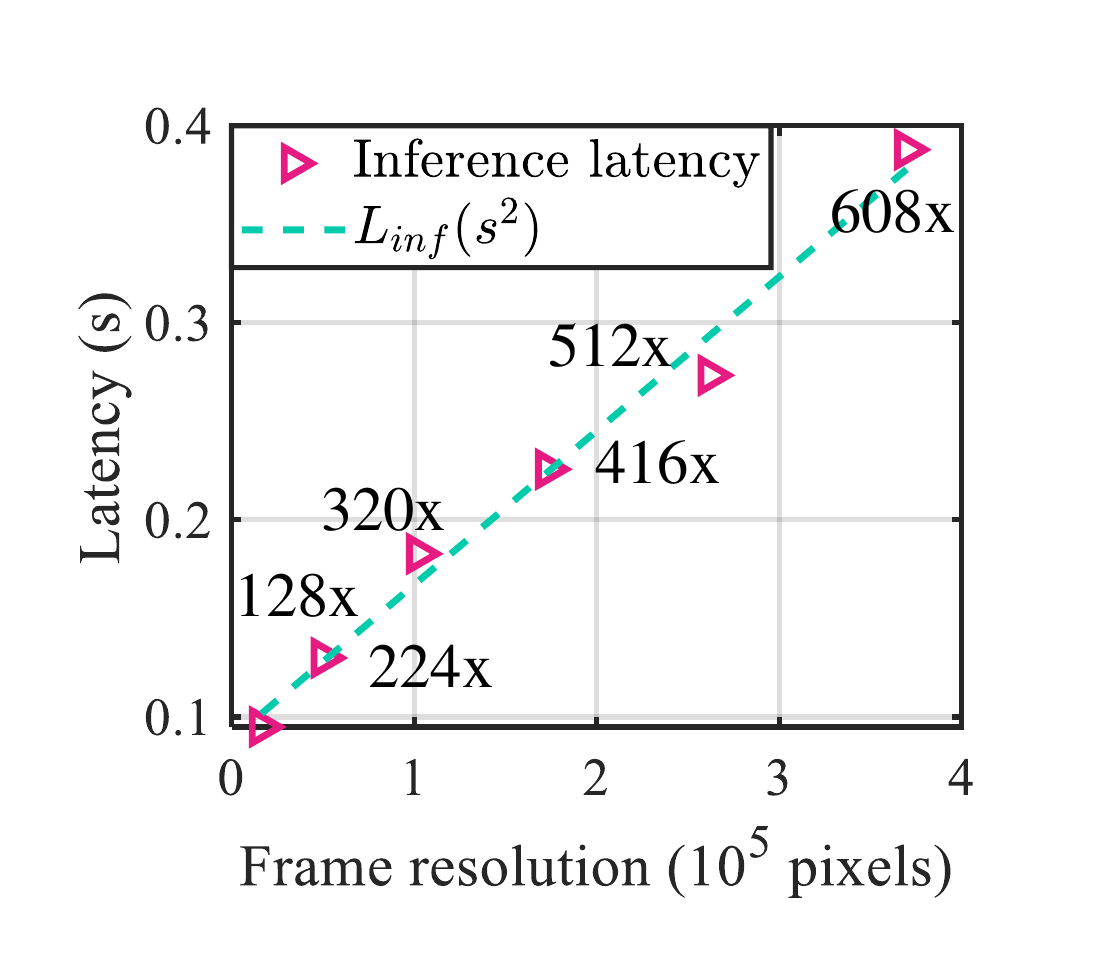}\label{fig:Linf}}
\subfigure[]
{\includegraphics[width=0.195\textwidth]{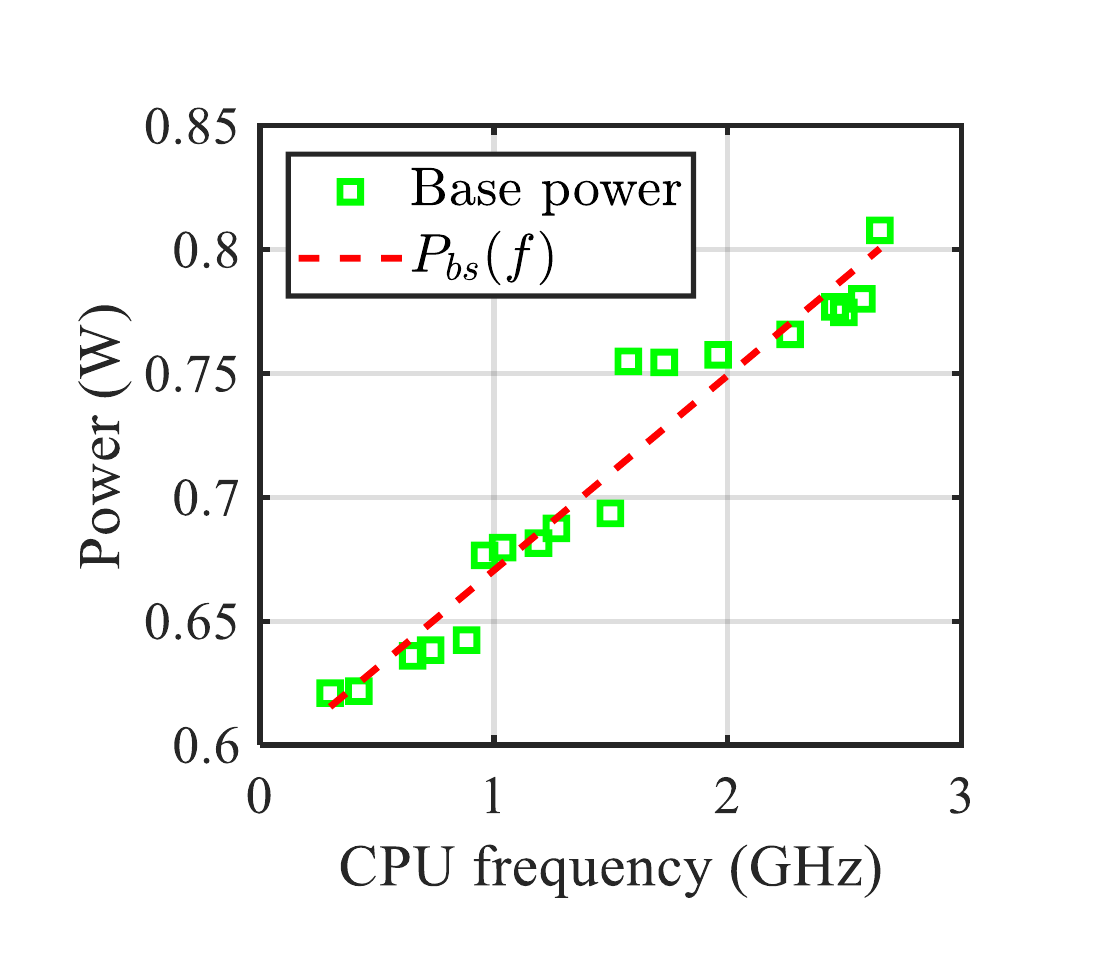}\label{fig:Pbs}}
%\vspace{-0.1in}
\caption{\first{The proposed data-driven analytic models for MAR devices, where each function, presented correspondingly in Table I, is trained offline via empirical measurements and regression analyses.}}
\label{fig:regressionmodel}   
%\vspace{-0.25in}
\end{figure*}

As shown in Subsection \ref{ssc:Analytics-based}, some interactions or functions in our proposed analytical models still cannot be expressed clearly in an analytic form. This is because of (i) the lack of analytic understandings of some interactions and (ii) specific coefficients/functions that may vary with different MAR device models. For example, in (\ref{eq:conversionPower}), the specific coefficients in $P^{k}_{cv}(f_{k})$ are unknown due to the lack of theoretical knowledge and vary with different MAR device models.

Therefore, we propose a data-driven methodology to address the above challenge, where those interactions with inadequate analytic understandings can be modeled and trained offline via empirical measurements and regression analyses. Note that regression-based modeling methodology is one of the most widely used approaches in developing mobile CPU's property models (e.g., CPU power and temperature variation modeling) and has shown to be effective in estimating CPU properties \cite{walker2016accurate,shye2009into,hu2017energy}. We use our testbed to collect measurements. The test MAR device is selected to work at $18$ different CPU frequencies ranging from $0.3$ to $2.649$ GHz. In addition, in order to obtain fine-grained regression models and eliminate the interference among different workloads on the device power consumption, we develop three Android applications; each is applied with a specific function of the MAR client, which includes image generation and preview, image conversion, and image transmission applications. The developed regression models are shown in Fig. \ref{fig:regressionmodel} and Table \ref{tb:regression}. Note that to obtain a statistical confidence in the experimental results, each data point in Fig. \ref{fig:regressionmodel} is derived by generating, transmitting, and detecting $1,000$ image frames and calculating the average values. The root mean square error (RMSE) is applied for calculating the average model-prediction error in the units of the variable of interest \cite{willmott2005advantages}.

\begin{table}[t]
 \begin{center}
    \caption{The Proposed Regression-based Models.}
    \label{tb:regression}
    \vspace{-0.2 in}
  \begin{tabular}{ |l||l|l| }
    \hline
     & Proposed models & RMSE\\ \hline
    $E_{gt}(f)$ & $-0.01071 f^{3}+0.06055 f^{2}-0.1028f+0.107$ & $0.002$ \\ \hline
    $E_{prv}(f)$ & $0.01094f+0.04816$ & $0.002$ \\ \hline
    $P_{cv}(f)$ & $0.1124f^{3}+0.01f^{2}+0.2175f+0.04295$ & $0.041$ \\\hline
    $L_{cv}(f)$ & $-0.145f^{3}+0.8f^{2}-1.467f+0.996$ & $0.025$\\\hline
    $r^{max}(B)$ & $0.677B$ & $2.403$ \\\hline
    $r^{*}(f)$ & $0.07651f^{3}-0.4264f^{2}+0.7916f+0.4489$ & $0.013$\\\hline
    $P_{tr}(R)$ & $0.01821R+0.7368$& $0.052$ \\\hline
    $L_{inf}(s^{2})$ & $0.07816s^{2}+0.08892$ & $0.838$ \\ \hline
    $P_{bs}(f)$ & $0.07873f+0.5918$ & $0.015$ \\\hline
  \end{tabular}
   \end{center}
\vspace{-0.25in}
\end{table}

\subsection{Problem Formulation}
\label{ssc:formulation}
Based on the above proposed models, we formulate the MAR reconfiguration as a multi-objective optimization problem \cite{deb2001multi}. We aim to \textit{minimize the per frame energy consumption of multiple MAR clients in the system while satisfying the user preference} (as stated in Section \ref{ssc:preference}) \textit{of each}. We introduce two positive weight parameters $\lambda^{k}_{1}$ and $\lambda^{k}_{2}$ to characterize the user preference of the $k$th MAR client, where $\lambda^{k}_{1}$ and $\lambda^{k}_{2}$ can be specified by the client. We adopt the weighted sum method \cite{marler2010weighted} to express the multi-object optimization problem as
\begin{small}
\begin{equation}
\begin{aligned}
\mathscr{P}_{0}: \min_{\{f_{k},s_{k},B_{k},\forall k\in \mathcal{K}\}} \quad & 
Q = \sum_{k\in \mathcal{K}} (E^{k}+ \lambda^{k}_{1}L^{k}-\lambda^{k}_{2}A_{k})\\
s.t. \quad 
  & C_{1}: \sum_{k\in \mathcal{K}}B_{k} \leq B_{max};\\
  & C_{2}: L^{k} \leq L_{max}^{k}, \forall k\in \mathcal{K};\\
  & C_{3}: F_{min} \leq f_{k}\leq F_{max},\forall k\in \mathcal{K};\\
  & C_{4}: s_{k}\in \{s_{min}, ..., s_{max}\},\forall k\in \mathcal{K};\\
\end{aligned}
\vspace{-0.05 in}
\end{equation}
\end{small}\noindent
where $A_{k}$ is an object detection accuracy function in terms of the $k$th MAR client selected computation model size $s_{k}^{2}$ (e.g., $A(s^{2}_{k}) = 1-1.578e^{-6.5\times10^{-3}s_{k}}$ \cite{liu2018edge}); $L^{k}_{max}$ is the maximum tolerable service latency of the $k$th client; $B_{max}$ is the maximum wireless bandwidth that an edge server can provide for its associated MAR clients. In practical scenarios, an edge server may simultaneously offer multiple different services for its associated users, e.g., video streaming, voice analysis, and content caching. Hence, the edge server may reallocate its bandwidth resource based on the user distribution. In this paper, we assume that $B_{max}$ varies with time randomly. The constraint $C_{1}$ represents that MAR clients' derived bandwidth cannot exceed the total bandwidth allocated for the MAR service on the edge server; the constraint $C_{2}$ guarantees that the service latency of MAR clients are no larger than their maximum tolerable latency; the constraints $C_{3}$ and $C_{4}$ are the constraints of the MAR device's CPU frequency and computation model size configurations, where $s_{k}$ is a discrete variable and its values depend on the available computation models in the MAR system.

\section{Proposed LEAF Optimization Algorithm}
\label{sc:leaf}
As shown in the previous section, problem $\mathscr{P}_{0}$ is a mixed-integer non-linear programming problem (MINLP) which is difficult to solve \cite{belotti2013mixed}. In order to solve this problem, we propose the LEAF algorithm based on the block coordinate descent (BCD) method \cite{grippo2000convergence}.

To solve problem $\mathscr{P}_{0}$, we relax the discrete variable $s_{k}$ into continuous variable $\hat{s_{k}}$. The problem is relaxed as
\begin{small}
\begin{equation}
\begin{aligned}
\mathscr{P}_{1}: \min_{\{f_{k},\hat{s_{k}},B_{k},\forall k\in \mathcal{K}\}} \quad & 
Q = \sum_{k\in \mathcal{K}} (E^{k}+ \lambda^{k}_{1}L^{k}-\lambda^{k}_{2}A_{k})\\
s.t. \quad 
  & C_{1}, C_{2}, C_{3}\\
  & \hat{C_{4}}: s_{min} \leq \hat{s_{k}}\leq s_{max}, \forall k\in \mathcal{K}.\\
\end{aligned}
\vspace{-0.05 in}
\end{equation}
\end{small}\noindent

According to the BCD method, we propose the LEAF algorithm which solves Problem $\mathscr{P}_{1}$ by sequentially fixing two of three variables and updating the remaining one. We iterate the process until the value of each variable converges.

$\nabla y(x)$ is denoted as the partial derivative of function $y$ corresponding to variable $x$. Denote $\vctProj[\mathcal{X}]{(x)}$ as the Euclidean projection of $x$ onto $\mathcal{X}$; $\vctProj[\mathcal{X}]{(x)} \triangleq \mathop{\arg\min}_{v\in\mathcal{X}}\lVert x-v\rVert^2$.

The procedure of our proposed solution is summarized as:
\begin{itemize}
\item Given $\hat{s_{k}}$ and $B_{k}$, we can derive a new $f_{k}$ according to
    \begin{small}
    \begin{equation}
    f_{k}^{(j+1)} = \vctProj[\mathcal{X}_{f}]{\left(f_{k}^{(j)} - \gamma_{k}\nabla Q_{k}\left(f_{k}^{(j)}\right)\right)}, \forall k\in \mathcal{K},
    \label{eq:projf}
    \vspace{-0.05 in}
    \end{equation}
    \end{small}\noindent
    where $\gamma_{k}>0$ is a constant step size and $\mathcal{X}_{f}$ is the bounded domain constrained by $C_{3}$. Based on the BCD method, we repeat (\ref{eq:projf}) until the derived $f_{k}$ is converged and then update $f_{k}$.
\item Given $f_{k}$ and $B_{k}$, we can derive a new $\hat{s_{k}}$ according to
    \begin{small}
    \begin{equation}
    \hat{s_{k}^{(j+1)}} = \vctProj[\mathcal{X}_{\hat{s}}]{\left(\hat{s_{k}}^{(j)} - \eta_{k}\nabla Q_{k}\left(\hat{s_{k}}^{(j)}\right)\right)}, \forall k\in \mathcal{K},
    \label{eq:projs}
    \vspace{-0.05 in}
    \end{equation}
    \end{small}\noindent
    where $\eta_{k}>0$ is a constant step size and $\mathcal{X}_{\hat{s}}$ is the bounded domain constrained by $\hat{C_{4}}$. Based on the BCD method, we repeat (\ref{eq:projs}) until the derived $\hat{s_{k}}$ is converged and then update $\hat{s_{k}}$.
\item Given $f_{k}$ and $\hat{s_{k}}$, the problem is simplified to
    \begin{small}
    \begin{equation}
    \begin{aligned}
    \min_{\{B_{k},\forall k\in \mathcal{K}\}} \quad & 
    Q = \sum_{k\in \mathcal{K}} \left(E^{k}+ \lambda^{k}_{1}L^{k}-\lambda^{k}_{2}A_{k}\right)\\
    s.t. \quad 
    & C_{1}: \sum_{k\in \mathcal{K}}B_{k} \leq B_{max};\\
    & C_{2}: L^{k} \leq L_{max}^{k}, \forall k\in \mathcal{K};\\
    \end{aligned}
    \vspace{-0.05 in}
    \end{equation}
    \end{small}\noindent
    where constraints $C_{3}$ and $\hat{C_{4}}$ are irrelevant to this problem.
\end{itemize}

The Lagrangian dual decomposition method is utilized to solve the above problem, where the Lagrangian function is
\begin{small}
\begin{equation}
\begin{split}
    \mathcal{L}\left(B_{k},\mu,\bm{\beta}\right) & = \sum_{k\in \mathcal{K}}\left(E^{k}+ \lambda^{k}_{1}L^{k}-\lambda^{k}_{2}A_{k}\right)\\
    &+ \mu\left(\sum_{k\in \mathcal{K}}B_{k} - B_{max}\right) + \sum_{k\in \mathcal{K}}\beta_{k}\left(L^{k}-L_{max}^{k}\right),
\end{split}
%\vspace{-0.15 in}
\end{equation}
\end{small}\noindent
where $\mu$ and $\bm{\beta}$ are the Lagrange multipliers, (i.e., $\bm{\beta}$ is a Lagrange multiplier vector), corresponding to constraints $C_{1}$ and $C_{2}$, respectively. The Lagrangian dual problem can therefore be expressed as
\begin{small}
\begin{equation}
\begin{aligned}
    \max_{\{\mu,\bm{\beta}\}} \quad
    g(\mu,\bm{\beta}) &= \min_{\{B_{k},\forall k\in\mathcal{K}\}} \mathcal{L}(B_{k},\mu,\bm{\beta})\\
    s.t. \quad
    & \mu\geq 0, \bm{\beta}\geq 0.
\end{aligned}
\vspace{-0.05 in}
\end{equation}
\end{small}\noindent
Here, $g(\mu,\bm{\beta})$ is concave with respect to $B_{k}$. 
\begin{myLem}
The problem $\mathscr{P}_{1}$ is convex with respect to $B_{k}$.
\end{myLem}
\begin{proof}
For any feasible $B_{i}, B_{j}, \forall i,j\in\mathcal{K}$, we have
\begin{small}
\begin{equation}
    \pdv{Q}{B_{i}}{B_{j}} =  
    \begin{cases}
        0, & i\ne j,\\
        \Psi_{i}\cdot\pdv{\left(1/r^{max}\right)}{B_{i}}{B_{j}}, & i = j,
    \end{cases}
    \vspace{-0.025 in}
\end{equation}
\end{small}\noindent
where $\Psi_{i} = \frac{\left[fps_{i}(E_{gt}(f_{i})+E_{prv}(f_{i}))+P_{tr}^{i}(0)+P_{bs}(f_{i})+\lambda^{i}_{1}\right]\sigma s_{i}^2}{r_{i}^{*}(f_{i})}$ which is positive, and $\pdv{\left(1/r^{max}\right)}{B_{i}}{B_{j}} = \frac{2}{0.677B_{i}^3}>0$. Thus, the Hessian matrix $\mathbf{H} = \left(\pdv{Q}{B_{i}}{B_{j}}\right)_{K\times K}$ is symmetric and positive definite. Constraint $C_{1}$ is linear and $C_{2}$ is convex with respect to $B_{k}$. Constraints $C_{3}$ and $C_{4}$ are irrelevant to $B_{k}$. Therefore, $\mathscr{P}_{1}$ is strictly convex with respect to $B_{k}$.
\end{proof}
Therefore, based on the Karush-Kuhn-Tucker (KKT) condition \cite{boyd2004convex}, the sufficient and necessary condition of the optimal allocated bandwidth for the $k$th MU can be expressed as
\begin{small}
\begin{equation}
    B_{k}^{*} = \sqrt{\frac{\Phi(f_{k},s_{k},\beta_{k})}{0.677\mu}},
    \vspace{-0.025 in}
\end{equation}
\end{small}\noindent
where $\Phi_{k} = \frac{\left[fps_{i}(E_{gt}(f_{i})+E_{prv}(f_{i}))+P_{tr}^{i}(0)+P_{bs}(f_{i})+\lambda^{i}_{1}+\beta_{k}\right]\sigma s_{i}^2}{r_{i}^{*}(f_{i})}$.

Next, the sub-gradient method \cite{boyd2004convex} is used to solve the dual problem. Based on the sub-gradient method, the dual variables of the $k$th MAR clients in the $(j+1)$th iteration are
\begin{small}
\begin{equation}
\left\{
    \begin{aligned}
    \mu_{k}^{(j+1)} &=  \max\left\{0,\left[\mu^{(j)} + \vartheta_{k}^{\mu}\nabla g(\mu^{(j)})\right]\right\}, \forall k\in\mathcal{K};\\
    \beta_{k}^{(j+1)} &= \max\left\{0,\left[\beta_{k}^{(j)} + \vartheta_{k}^{\beta}\nabla g(\beta_{k}^{(j)})\right]\right\}, \forall k\in\mathcal{K};\\
    \end{aligned}
\right.
\vspace{-0.05 in}
\end{equation}
\end{small}\noindent
where $\vartheta_{k}^{\mu}>0$ and $\vartheta_{k}^{\beta}>0$ are the constant step sizes.

Based on the above mathematical analysis, we propose an MAR optimization algorithm, LEAF, which can dynamically determines the CPU frequency of multiple MAR devices, selects the computation model sizes, and allocates the wireless bandwidth resources. The pseudo code of the proposed LEAF MAR algorithm is presented in Algorithm \ref{alg:leaf}. First, the LEAF is initialized with the lowest CPU frequency, the smallest computation model size, and evenly allocated bandwidth resources among MAR devices.
We then iteratively update $f_{k}$, $\hat{s_{k}}$, and $B_{k}$ until the LEAF converges (i.e., line $7$-$8$ in Algorithm \ref{alg:leaf}). In addition, $\hat{s_{k}}$ is a relaxed value of the computation model size. Thus, it may not match any pre-installed computation model in a real system. In this case, the LEAF selects the computation model size $s_{k}$ that is the closest to the relaxed one $\hat{s_{k}}$ (i.e., line $10$ in Algorithm \ref{alg:leaf}). Since the LEAF MAR algorithm is developed based on the BCD method and follows the convergence results in \cite{grippo2000convergence}, we claim that the LEAF converges to a local optimal solution.

\begin{algorithm}[t]
\DontPrintSemicolon
\footnotesize
\SetNoFillComment
\label{alg:leaf}
\caption{The LEAF MAR Algorithm}
\KwIn{$\lambda^{k}_{1}$, $\lambda^{k}_{2}$, $L_{max}^{k}$, $B_{max}$, $fps_{k}$, and $\tau$, $\forall k \in \mathcal{K}$.}
\KwOut{$f_{k}$, $s_{k}$, and $B_{k}$, $\forall k \in \mathcal{K}$.}
$B_{k} \gets B_{max}/|\mathcal{K}|$,
$\hat{s_{k}} \gets s_{min}$, $\forall k \in \mathcal{K}$, $i\gets 1$;\\
\While{True}{
$f_{k}$ $\gets$ solving $\mathscr{P}_{1}$ with fixed $\hat{s_{k}}$ and $B_{k}$;\\
$\hat{s_{k}} \gets$ solving $\mathscr{P}_{1}$ with fixed $f_{k}$ and $B_{k}$;\\
$B_{k} \gets$ solving $\mathscr{P}_{1}$ with fixed $f_{k}$ and $\hat{s_{k}}$;\\
$Q_{i} \gets \sum_{k\in \mathcal{K}}(E^{k}+ \lambda^{k}_{1}L^{k}-\lambda^{k}_{2}A_{k})$\\
\If{$|(Q_{i}-Q_{i-1})/Q_{i}|\leq \tau$}{
\textbf{break;} \Comment{Converges}\\
}
$i\gets i+1$;\\
}
$s_{k}=\mathop{\arg\min}\limits_{s\in\{s_{min}...s_{max}\}}|s-\hat{s_{k}}|$, $\forall k\in \mathcal{K}$;\\
\Return $f_{k}$, $s_{k}$, and $B_{k}$, $\forall k\in \mathcal{K}$.
\end{algorithm}

\section{Image Offloading Frequency Orchestrator}
\label{sc:tracker}
In this section, an offloading frequency orchestrator with local object tracking is proposed to further reduce the energy consumption and latency of MAR devices by leveraging the model we developed for significant scene change estimation, based on our proposed LEAF algorithm.

\subsection{Edge-based Object Detection vs. Local Object Tracking}

As presented in Section \ref{sc:leaf}, our proposed LEAF is able to guide MAR configuration adaptations and radio resource allocations at the edge server to improve the energy efficiency of executing continuous image offloading and object detection. However, continuous repeated executions of offloading camera image frames to the edge server for object detection are unnecessary. This is because, although the positions of detected objects may slightly change in continuous camera captured frames due to the camera movement or detected object motions, the probability of significant changes to the scene or a new object appearing is low within a very short period. For example, as shown in Fig. \ref{fig:example}, three image frames are extracted in a video stream. From Frame 1 to Frame 10, only the position of the detected dog in the scene changes. Thus, sending every captured frame to the edge server for detecting objects (i.e., locating and recognizing objects in a frame) is extremely inefficient and will cause unnecessary energy expenditure even with our proposed LEAF.

\begin{figure}[t]
\centering
\includegraphics[width=0.48\textwidth]{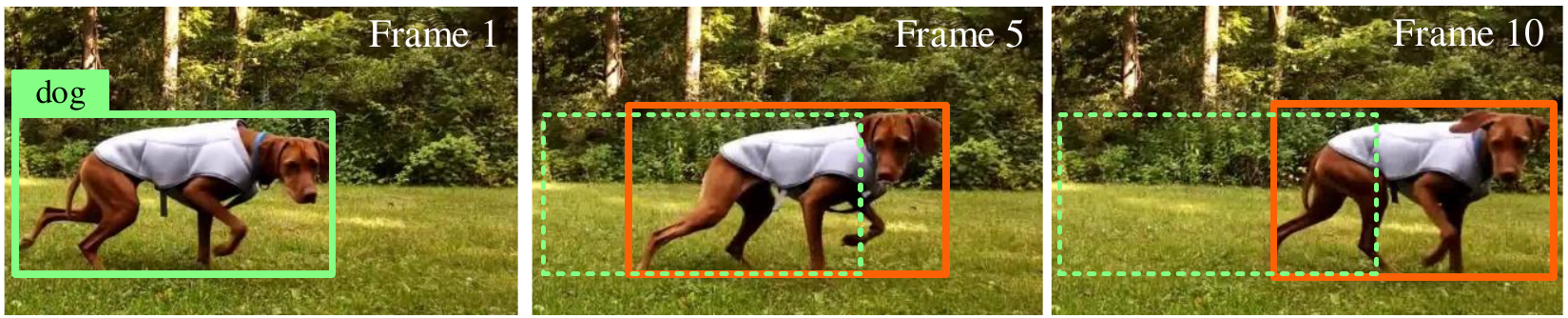}
\caption{Continuous repeated executions of object detection are unnecessary.}
\label{fig:example}
\vspace{-0.15in}
\end{figure}

To reduce the execution of continuous image offloading and object detection, one naive approach is to implement a local lightweight object tracker on the MAR device and invoke the tracker for updating the locations of the objects of interest that are achieved by performing a successful object detection, as is done in several prior works \cite{apicharttrisorn2019frugal,huynh2017deepmon,ran2018deepdecision}. However, three essential questions are brought up here:

\textbf{RQ 1.} How much energy can the local object tracker save for an MAR device compared to performing edge-assisted object detection? It is intuitive that local lightweight object tracker consumes less battery than local CNN-based object detector due to the nature of CNNs, which contains tens to hundreds of computation-intensive layers. But how does it compare to the edge-based detectors, where the MAR device's on-board resource is not consumed by running CNNs?

\textbf{RQ 2.} How does the MAR device's hardware capacity (e.g., CPU frequency) impact the tracking performance and overhead (e.g., tracking delay and energy consumption)? It is critical to have the knowledge that whether the object tracker can help to improve the energy efficiency of MAR devices within the full or only a partial range of CPU frequencies.  

\textbf{RQ 3.} How does the MAR device determine the frequency of the image offloading and object detection? The frequency of executing edge-based object detector is the most essential and challenging parameter of the MAR system. If the edge-based object detector (i.e., image offloading) is executed as often as possible, the MAR device may achieve a high object detection and tracking accuracy but a high energy expenditure. However, if the edge-based object detector is executed with a low frequency, for instance, executing an object detection only once at the beginning of tracking, the MAR device may achieve a high energy efficiency but unacceptable tracking accuracy (e.g., in our experiment, we observe that the tracking accuracy decreases or even the tracker loses objects of interest as the time interval between the current frame and reference frame performed object detection increases).   

To the best of our knowledge, these questions lack pertinent investigations and sophisticated solutions in both academia and industry, such as ARCore \cite{ARCore} and ARKit \cite{ARKit}. To explore these questions, we implement a real-time lightweight object tracker on a Nexus $6$ using JavaCV libraries\footnote{JavaCV is a collection of wrappers in Java for commonly used libraries in the field of computer vision such as OpenCV \cite{opencv}.} \cite{JavaCV}. The implemented lightweight object tracker in this paper is developed based on Kernelized Correlation Filter (KCF) \cite{henriques2014high} which is a tracking framework that utilizes properties of circulant matrix to enhance the processing speed. KCF tracker has achieved impressive tracking delay and accuracy on Visual Tracker Benchmarks \cite{wu2013online}. The latency of performing an object tracking on a single video frame contains (i) the latency of converting a camera captured raw YUV video frame produced by \texttt{\textsc{ImageReader}} to an \texttt{\textsc{Mat}}\footnote{\texttt{\textsc{Mat}} is a basic image container used in OpenCV.} object, (ii) the latency of converting a frame from color to gray scale (i.e., \texttt{COLOR\_BGR2GRAY}), and (iii) the latency of executing the KCF object tracker.

\begin{figure}[t]
\centering
\subfigure[]
{\includegraphics[width=0.24\textwidth]{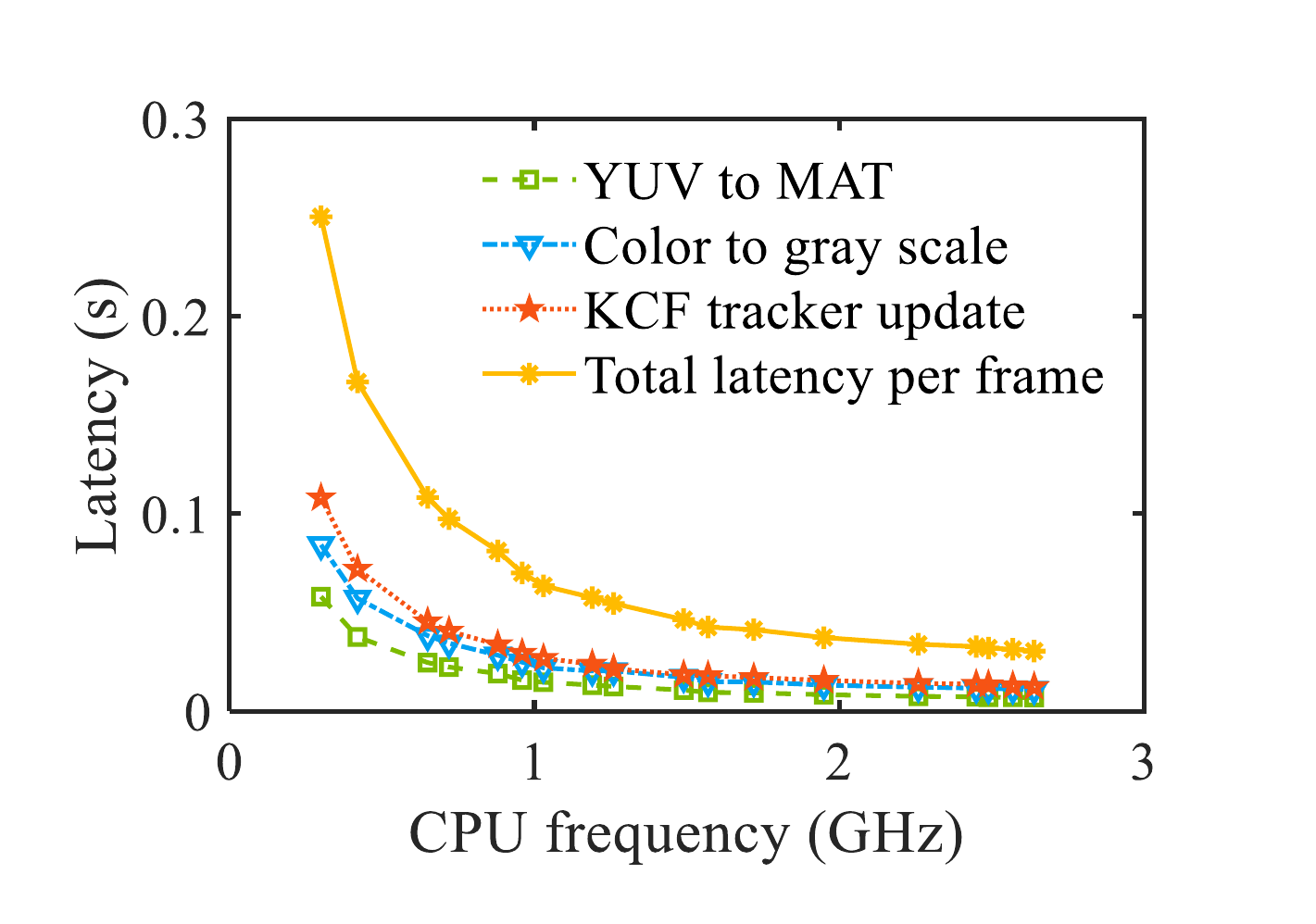}\label{fig:Trackerlatency}}
\subfigure[]
{\includegraphics[width=0.235\textwidth]{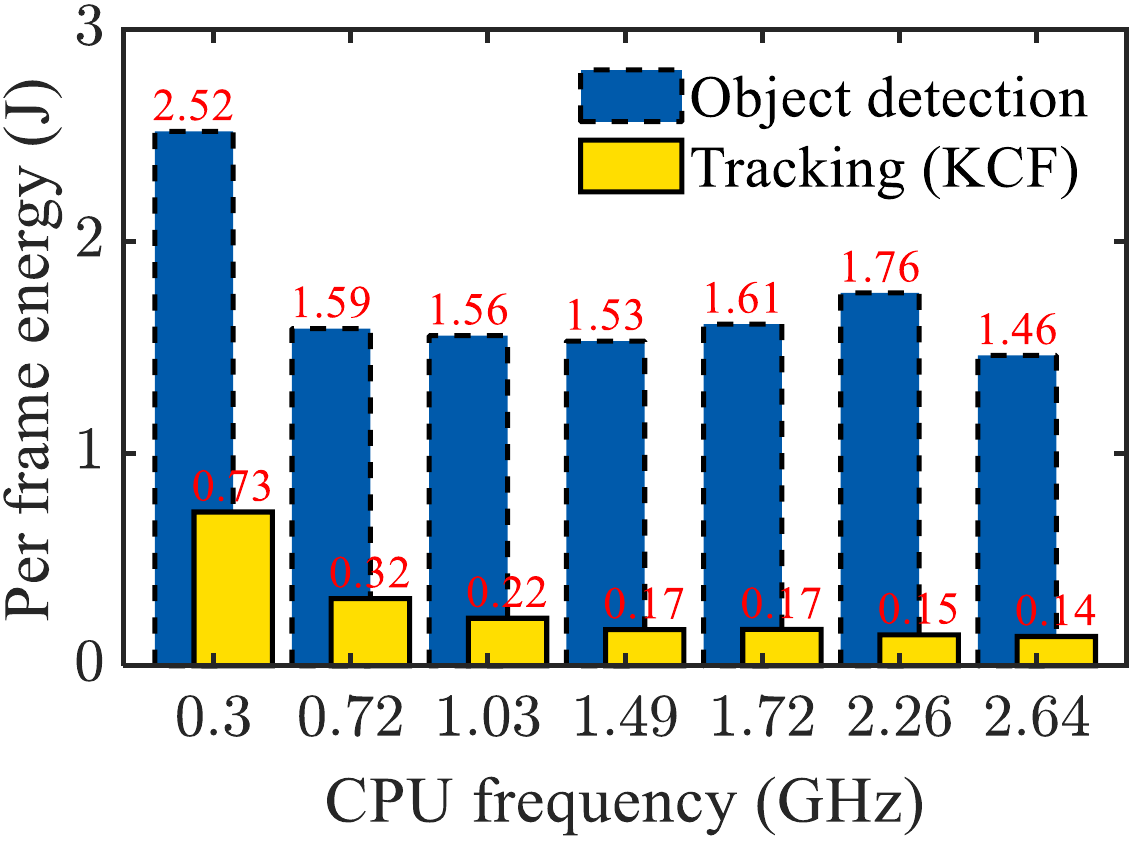}\label{fig:trackvsobject}}
\caption{CPU frequency vs. latency and per frame energy consumption of local object tracking.}
\label{fig:tracker}
\vspace{-0.2in}
\end{figure}

Fig. \ref{fig:tracker} illustrates the average object tracking latency and per frame energy consumption when the MAR device runs on different CPU frequencies. In Fig. \ref{fig:Trackerlatency}, we observe that the average total latency of performing an object tracking is significantly reduced compared to the latency of edge-based object detection, as presented in Section \ref{sc:Experimental Results}. For instance, when the CPU frequency of Nexus $6$ is $1.728$ GHz, the average latency of object detection and tracking are $500$ ms and $40$ ms, respectively (RQ 1). Fig. \ref{fig:trackvsobject} compares the per frame energy consumption of the edge-based object detection and the local object tracking, where we find that the lightweight local object tracker, KCF, can help to improve the energy efficiency of the Nexus $6$ within the full range of CPU frequencies (RQ 2). The per frame energy consumption is decreased by over $80\%$ comparing to the object detection ($s^{2}=320^2$) when the device's CPU frequency is not less than $1.032$ GHz. \textit{Therefore, implementing a lightweight local object tracker will not only help MAR devices to further mitigate the quick battery depletion, but also drop the latency substantially.}

The above experimental result and discussion advocate adding a local lightweight object tracker in our developed edge-based MAR system, as depicted in Fig. \ref{fig:overview}, for further improving the energy efficiency of MAR devices and reducing the latency. However, given the discussion on RQ 3, we argue that naively implementing an object tracker in such a system is inadequate, where an image offloading frequency orchestrator that balances the trade-off between MAR device's energy efficiency and tracking accuracy is essential and nonignorable. We design and implement such an orchestrator,  which coordinates with the proposed LEAF to adaptively and intelligently adjust the image offloading frequency (i.e., execution of edge-based object detection) based on real-time scene change estimations.

\subsection{Image Offloading Frequency Orchestrator}

\begin{figure}[t]
\centering
\includegraphics[width=0.48\textwidth]{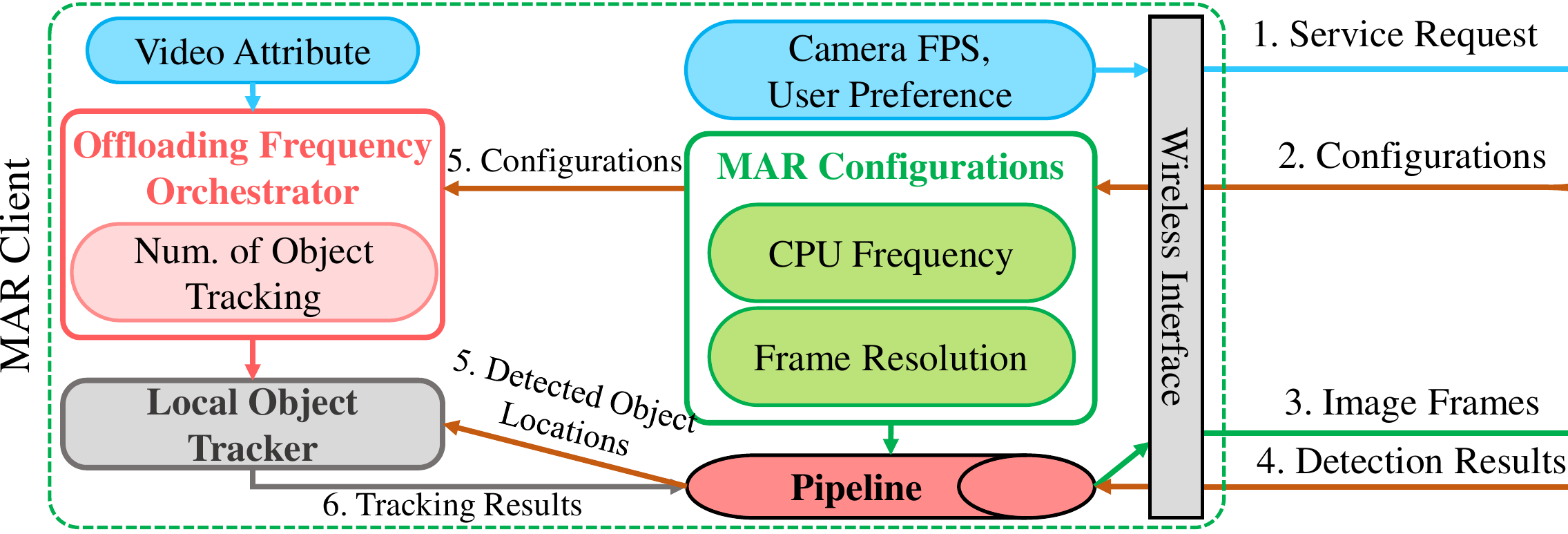}
\caption{Overview of the MAR client deployed with our proposed image offloading frequency orchestrator and how it coordinates with the rest of the edge-based MAR system with LEAF.}
\label{fig:overview2}
\vspace{-0.15 in}
\end{figure}

Fig. \ref{fig:overview2} provides an overview of how our proposed image offloading frequency orchestrator coordinates with the LEAF illustrated in Fig. \ref{fig:overview}. The proposed orchestrator is implemented in MAR devices. The MAR device invokes the orchestrator after it has successfully received the object detection results from the edge server. The inputs of the orchestrator are the optimal MAR configurations (i.e., CPU frequency and frame resolution) obtained from the proposed LEAF. The output is the estimated number of next successive image frames that will perform local object tracking, denoted by $\rho$. For instance, if the output of the orchestrator is $11$, the next $11$ continuous frames will not be eligible for offloading and will be transited to the local object tracker to perform tracking. 

However, designing such an image offloading frequency orchestrator is challenging. Prior work \cite{apicharttrisorn2019frugal} set a single threshold to determine whether the current image frame should be offloaded to the edge. However, (i) the value of the threshold is significantly experience-driven, which is unrealistic to handle all environment conditions with one single threshold; (ii) it lacks exploration of what is the optimal offloading solution for the MAR device in a time period. To tackle these, our orchestrator is designed based on two principles: (i) the detection/tracking decision will be made via a context-aware optimization algorithm, which is developed based on our proposed analytical model and LEAF; (ii) in order to achieve real-time results, considering the restricted computation capability of MAR devices, the designed algorithm should be as lightweight as possible.

To fulfill the first principle, it is necessary to predict how $\rho$ will impact the object tracking accuracy within various scenarios, as tracking is not always accurate with respect to changes in object locations. The larger $\rho$ the orchestrator provides, the less similarity between the current tracked image frame and the frame executed object detection, which raises the probability of tracking accuracy degradation. In addition, the attribute of the scenario (e.g., objects of interest being blurred) also heavily impacts the similarity among continuous tracked frames. To assess the tracking accuracy in terms of $\rho$, we choose to measure the similarity between two images using the peak signal-to-noise ratio (PSNR). PSNR provide a measure of the similarity between two images, $I_1$ and $I_2$, and is given by $PSNR(I_1,I_2) = 20\log_{10}\left(\frac{MAX_{I}}{\sqrt{MSE(I_1,I_2)}}\right)$, where $MAX_{I}$ is the maximum possible pixel value of the image. $MSE(I_1,I_2)$ is the mean squared error of two images and is calculated by $MSE(I_1,I_2)=\frac{1}{MN}\sum^{M}_{i=1}\sum^{N}_{j=1}\left(I^{ij}_{1}-I^{ij}_{2}\right)^{2}$, where $M\times N$ is the image size; and $i, j$ are the pixel locations within the images. Furthermore, the average intersection over union (IOU) is used for estimating the tracking accuracy. The IOU of object $o$ in frame $I$ is $IOU^{I}_{o}=\frac{R^{G}_{o}\cap R^{P}_{o}}{R^{G}_{o}\cup R^{P}_{o}}$, where $R^{G}_{o}$ is the groundtruth region of object $o$, and $R^{P}_{o}$ is the predicted region of object $o$. Therefore, to develop the context-aware optimization algorithm, we need to explore the interactions among $\rho$, $PSNR$, and $IOU$.

To explore such interactions, we leverage an open dataset \cite{WuLimYang13} which contains $100$ videos with nine different scene attributes, such as illumination variation (i.e., the illumination in the target region is significantly changed) and motion blur (i.e., the target region is blurred due to the motion of the target or camera). We measure how IOU varies in terms of PSNR using videos with the same scene attribute, where we gradually increase the frame interval of two measured images, such as Frame 1 to Frame 3, Frame 1 to Frame 4, etc. Figs. \ref{fig:video1} to \ref{fig:video3} depict the measurement of three videos with motion blur. We observe that different videos that have the same major attribute obtain similar shape of $IOU(PSNR)$ (due to the page limit, we only show the results of three videos with motion blur). If we integrate all the samples into one figure, as illustrated in Fig. \ref{fig:video_integration}, we can achieve a regression-based model which describes an object tracking accuracy function in terms of PSNR for a specific scenario (e.g., for motion blur, $IOU(PSNR) = -0.004335PSNR^2 + 0.2411PSNR-2.328$).

\begin{figure*}[t]
\centering
\subfigure[]
{\includegraphics[width=0.24\textwidth]{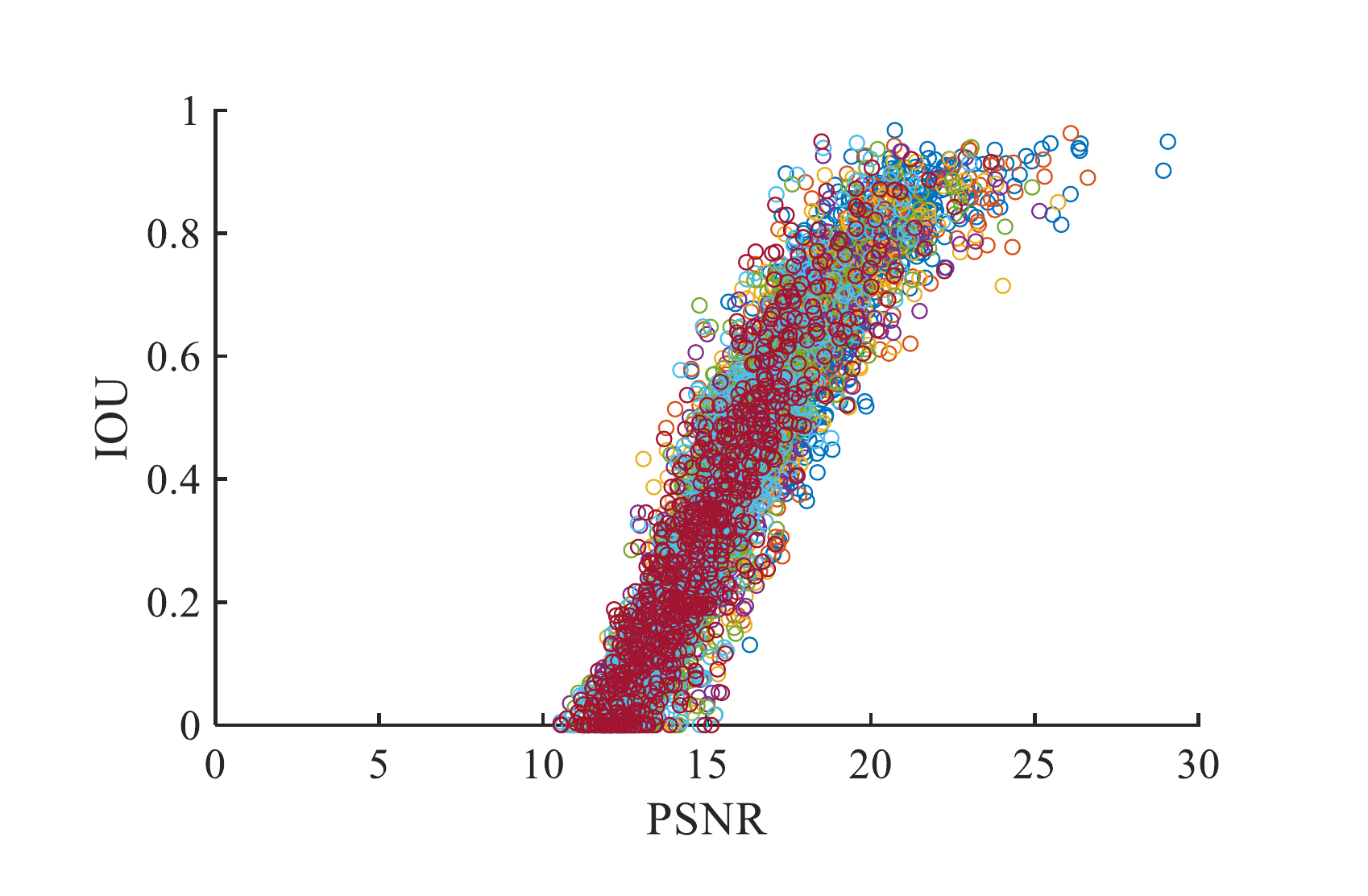}\label{fig:video1}}
\subfigure[]
{\includegraphics[width=0.24\textwidth]{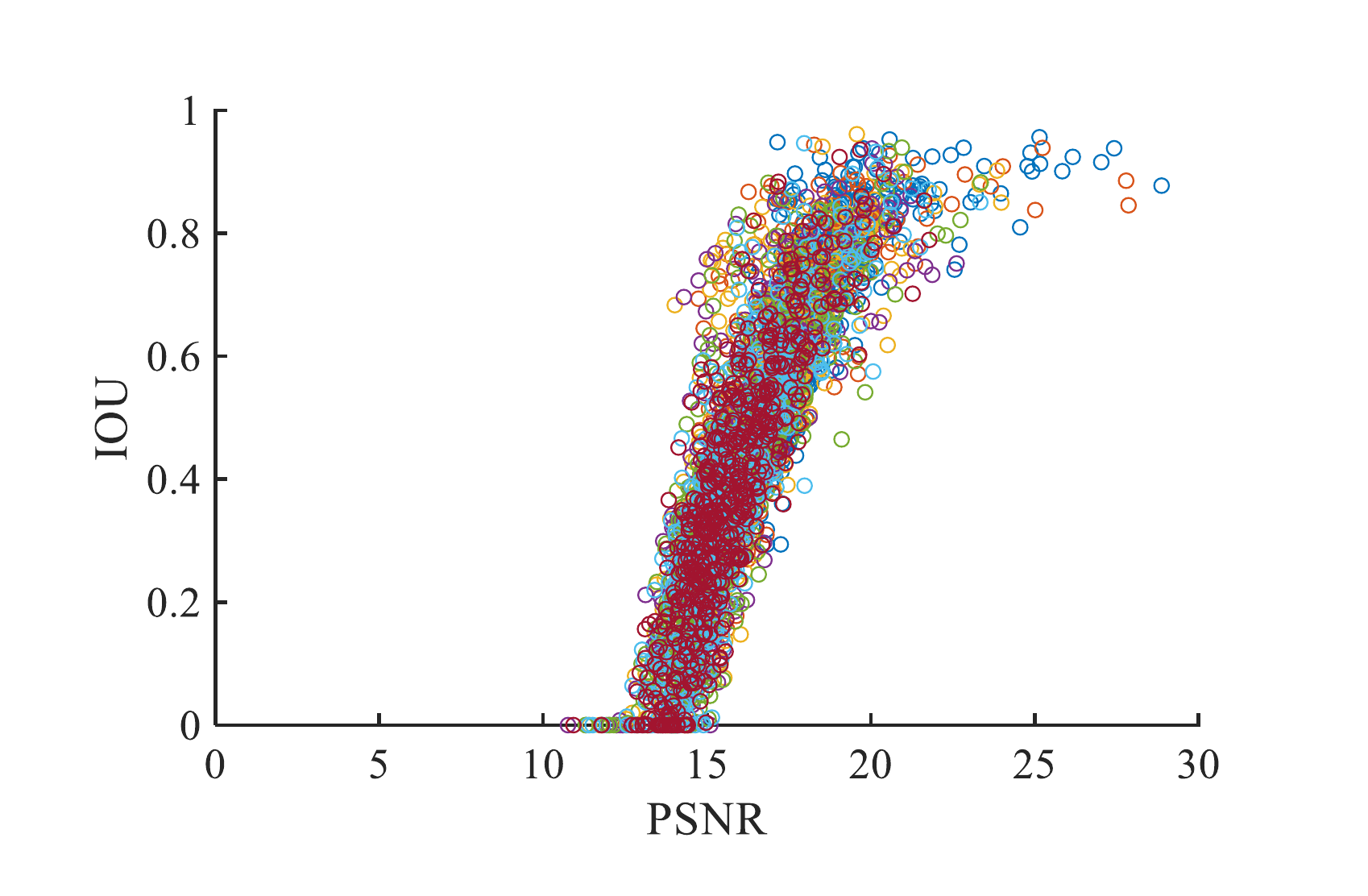}\label{fig:video2}}
\subfigure[]
{\includegraphics[width=0.24\textwidth]{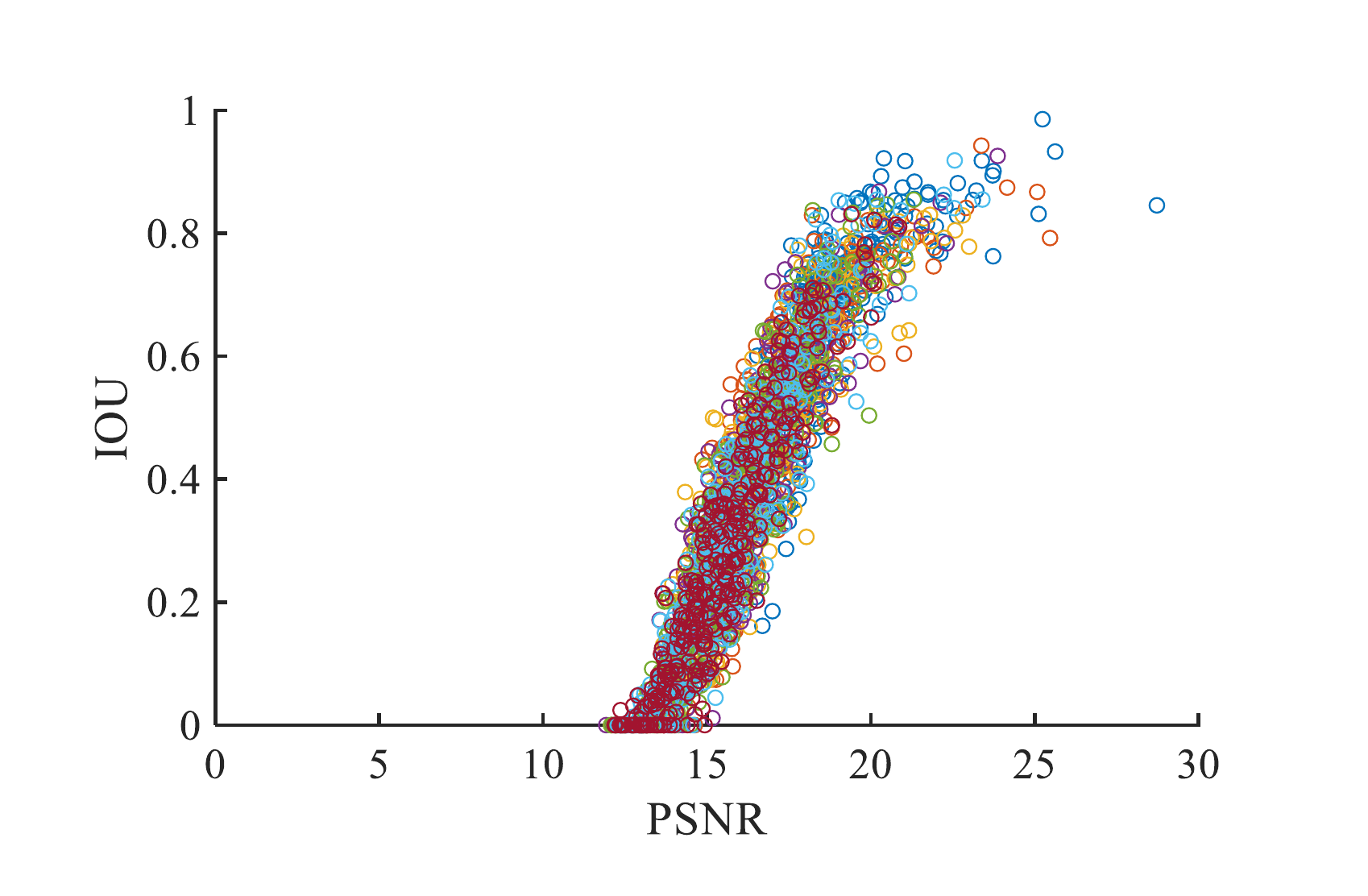}\label{fig:video3}}
\subfigure[]
{\includegraphics[width=0.24\textwidth]{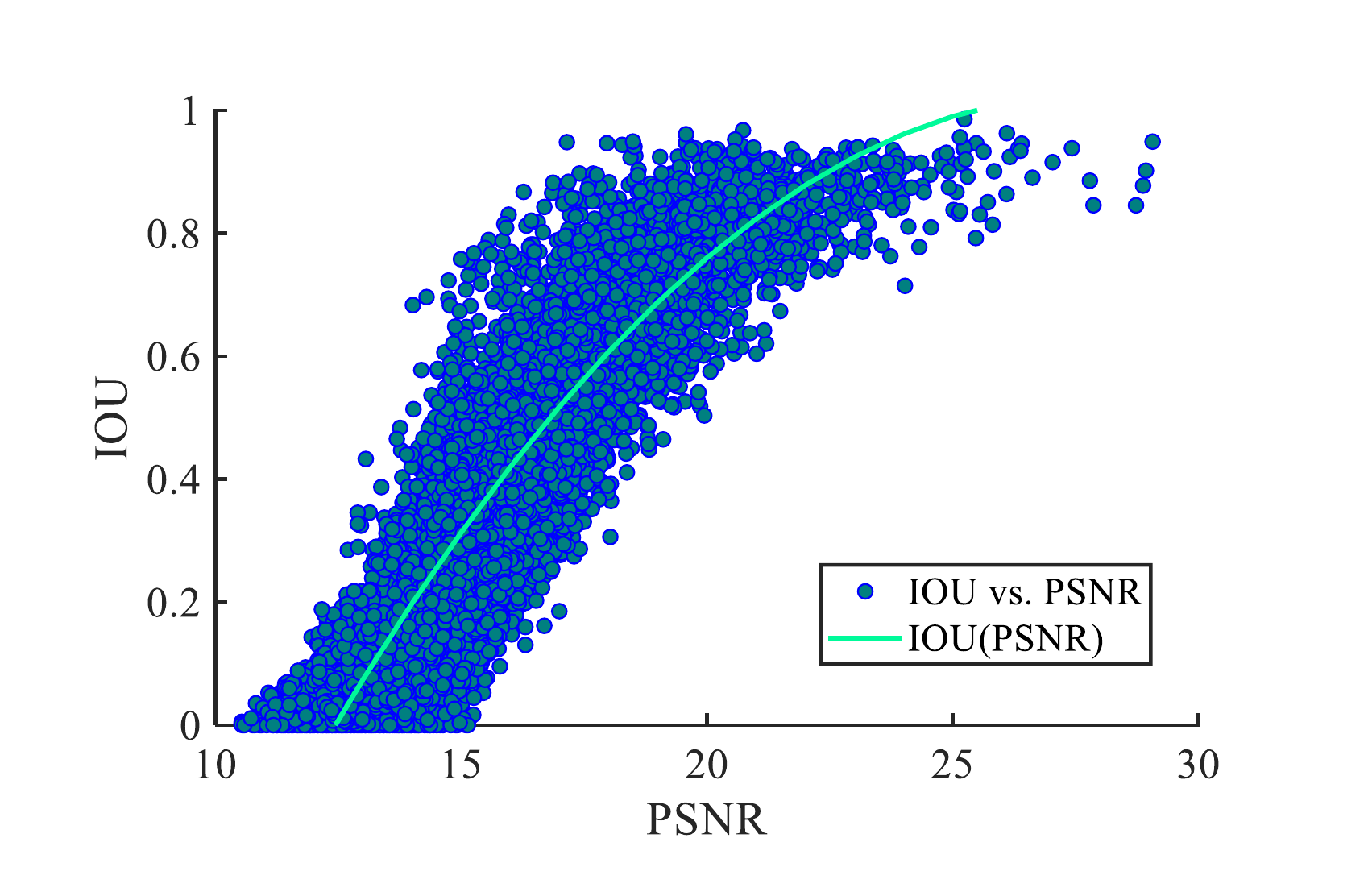}\label{fig:video_integration}}
\vspace{-0.1in}
v\caption{Study on how the PSNR impacts the IOU degradation in videos with motion blurred attribute.}
\label{fig:video}   
\vspace{-0.15in}
\end{figure*}

Given the above discussion and analysis, we formulate the image offloading frequency decision as an optimization problem $\mathscr{P}_{2}$, which aims to achieve an optimal $\rho$ to balance the MAR device's energy consumption and tracking accuracy loss. $E_{ogj}$ and $E_{trk}$ are estimated per frame energy consumption of edge-based object detection and local object tracking based on current MAR system configurations determined by LEAF, respectively. Two positive weight parameters $\theta_{1}$ and $\theta_{2}$ are introduced to characterize the offloading preference of an MAR client. For example, given a larger $\theta_{1}$ and a smaller $\theta_{2}$, the decision made by the offloading frequency orchestrator will be more aggressive on saving MAR device's battery life, and vice versa. Current scene attribute is denoted by $\varpi$. 

\begin{small}
\begin{equation}
\begin{aligned}
\mathscr{P}_{2}: \min_{\{\rho\}} \quad & 
J =  \theta_{1}\frac{E_{obj}+ E_{trk}\rho}{1+\rho} - \theta_{2}IOU^{\varpi}_{trk}(PSNR(\rho)) \\
s.t. \quad 
  & \rho\in \{0, 1, 2, ...\}.\\
\end{aligned}
\end{equation}
\end{small}

\begin{algorithm}[t]
\DontPrintSemicolon
\footnotesize
\SetNoFillComment
\label{alg:offload}
\caption{The AIO Algorithm}
\KwIn{$\rho$, $E_{obj}$, $E_{trk}$, $\theta_{1}$, $\theta_{2}$, $\mathcal{P}$, and $\mathcal{V}$.}
\KwOut{$\rho$, $\mathcal{P}$, and $\mathcal{V}$.}
\If{Object detection = True {\normalfont \textbf{and}} $\rho = 0$}{
$\mathcal{P} \gets \mathcal{P} \cup \{PSNR(I_{i-1},I_{i})\}$;\\ 
$\mathcal{V} \gets \mathcal{V} \cup \{v_{i} =\frac{PSNR(I_{i-1},I_{i})-PSNR(I_{i-2},I_{i-1})}{2}\}$;\\ 
$\Bar{v} \gets \frac{\sum^{i-n}_{i}{v_{i}w_{i}}}{\sum^{i-n}_{i}{w_{i}}}, v_{i}\in \mathcal{V}$;\\
$IOU_{trk}\left(\right) \gets IOU^{\varpi}_{trk}\left(\right)$; \Comment{Scene attribute estimation}\\

$\rho$ $\gets$ solving $\mathscr{P}_{2}$ with $E_{obj}$, $E_{trk}$, $\theta_{1}$, $\theta_{2}$, $\Bar{v}$, and $IOU^{\varpi}_{trk}\left(\right)$;\\
\If{$\rho = 0$}{
Object detection $\gets$ False;\\
}
}
\ElseIf{$\rho \neq 0$}{
$\mathcal{P} \gets \mathcal{P} \cup \{PSNR(I_{i-1},I_{i})\}$;\\ 
$\mathcal{V} \gets \mathcal{V} \cup \{v_{i} =\frac{PSNR(I_{i-1},I_{i})-PSNR(I_{i-2},I_{i-1})}{2}\} $;\\ 
$\rho \gets \rho - 1$;\\
\If{$\rho = 0$}{
Object detection $\gets$ False;
}
}
\Return $\rho$, $\mathcal{P}$, and $\mathcal{V}$.
\end{algorithm}

Based on $\mathscr{P}_{2}$, we develop an \underline{a}daptive \underline{i}mage \underline{o}ffloading (AIO) algorithm implemented in the offloading frequency orchestrator, where the pseudo code of it is presented in Algorithm \ref{alg:offload}. The proposed AIO will be triggered after the MAR device executes a successful object detection and receives the corresponding detection results from the edge server. First, the AIO calculates the $PSNR$ of the current image frame and $v_{i}$ (i.e., transient gradient of $PSNR$) and updates them in sets $\mathcal{P}$ and $\mathcal{V}$, respectively (i.e., line $2$-$3$ in Algorithm \ref{alg:offload}). We then estimate the current scene change rate denoted by $\Bar{v}$. To avoid a transient outlier that impacts the precision of estimation, $\Bar{v}$ is achieved by calculating the weighted mean of elements of $\mathcal{V}$ that are the latest updated in a short range of time (e.g., in the past 2 seconds). In this paper, we use exponential function to calculate the weights, where the element that is updated later will be allocated with a larger weight. The scene attribute estimation is out of the scope of this paper, where we assume that the AIO knows the current $IOU^{\varpi}_{trk}\left(\right)$ when solving $\mathscr{P}_{2}$. In addition, to fulfill the second principle (i.e., reducing the workload on MAR devices), the per frame energy consumption of object detection $E_{obj}$ is estimated by the edge server via our proposed analytical model, presented in Section \ref{sc:models}, based on the LEAF guided configurations, while the per frame energy consumption of object tracking $E_{trk}$ is estimated locally via a preset table. Thus, the AIO can efficiently achieve the value of $E_{trk}$ through the table in terms of its current CPU frequency. Finally, the AIO outputs an optimal $\rho$ and the MAR device will keep performing local object tracking until $\rho$ decreases to $0$.

% \begin{small}
% \begin{equation}
% \begin{aligned}
% \mathscr{P}_{2}: \min_{\{\rho\}} \quad & 
% Q =  \alpha_{1}\frac{E_{obj}+ E_{trk}\rho}{1+\rho} - \alpha_{2}IoU(PSNR(\rho)) \\
% s.t. \quad 
%   & \rho\in \{0, 1, 2, ...\}.\\
% \end{aligned}
% %\vspace{-0.05 in}
% \end{equation}
% \end{small}

\section{Performance Evaluation of the Proposed Analytical Model and MAR System}
\label{sc:evluation}
In this section, we evaluate both the proposed MAR analytical energy model as well as the proposed LEAF and AIO algorithms. We first validate our analytical model by comparing the estimated energy consumption with the real energy measurement (obtained from our developed testbed described in Section \ref{sc:Experimental Results}). The Mean Absolute Percentage Error (MAPE) is used for quantifying the estimation error. Then, we evaluate the per frame energy consumption, service latency, and detection accuracy of the proposed LEAF and AIO algorithms under variant bandwidth and user preferences through data-driven simulations.
%\vspace{-0.05in}
\subsection{Analytical Model Validation}
The measured power and duration of promotion and tail phases in WiFi are shown in Table \ref{tb:prom&tail} (note that LTE has different values \cite{hu2015energy}). As shown in Fig. \ref{fig:modelV}, we validate the proposed analytical model with respect to MAR client's CPU frequency, computation model size, allocated bandwidth, and camera FPS. Each measured data is the average of the per frame energy consumption of $1,000$ image frames. The calculated MAPE of these four cases are $6.1\% \pm 3.4\%$, $7.6\% \pm 4.9\%$, $6.9\% \pm 3.9\%$, and $3.7\% \pm 2.6\%$, respectively. Therefore, our proposed energy model can estimate the MAR per frame energy consumption very well.
%, depicted in Fig. \ref{fig:modelAf} to \ref{fig:modelAfps}, respectively
\begin{table}[t]
\vspace{-0.1in}
 \begin{center}
    \caption{Power and Duration of Promotion \& Tail Phases.}
    \label{tb:prom&tail}
    %\vspace{-0.1in}
  \begin{tabular}{ |c|c|c|c| }
    \hline
    $P_{pro}$ (W) & $t_{pro}$ (s) & $P_{tail}$ (W) & $t_{tail}$ (s)\\ \hline
    $1.97\pm 0.08$ & $0.034\pm 0.004$ & $1.61\pm 0.15$ & $0.21\pm 0.02$\\ \hline
  \end{tabular}
   \end{center}
\vspace{-0.2in}
\label{tb:promo&tail}
\end{table}

\begin{figure*}[t]
\centering
\subfigure[]
{\includegraphics[width=0.24\textwidth]{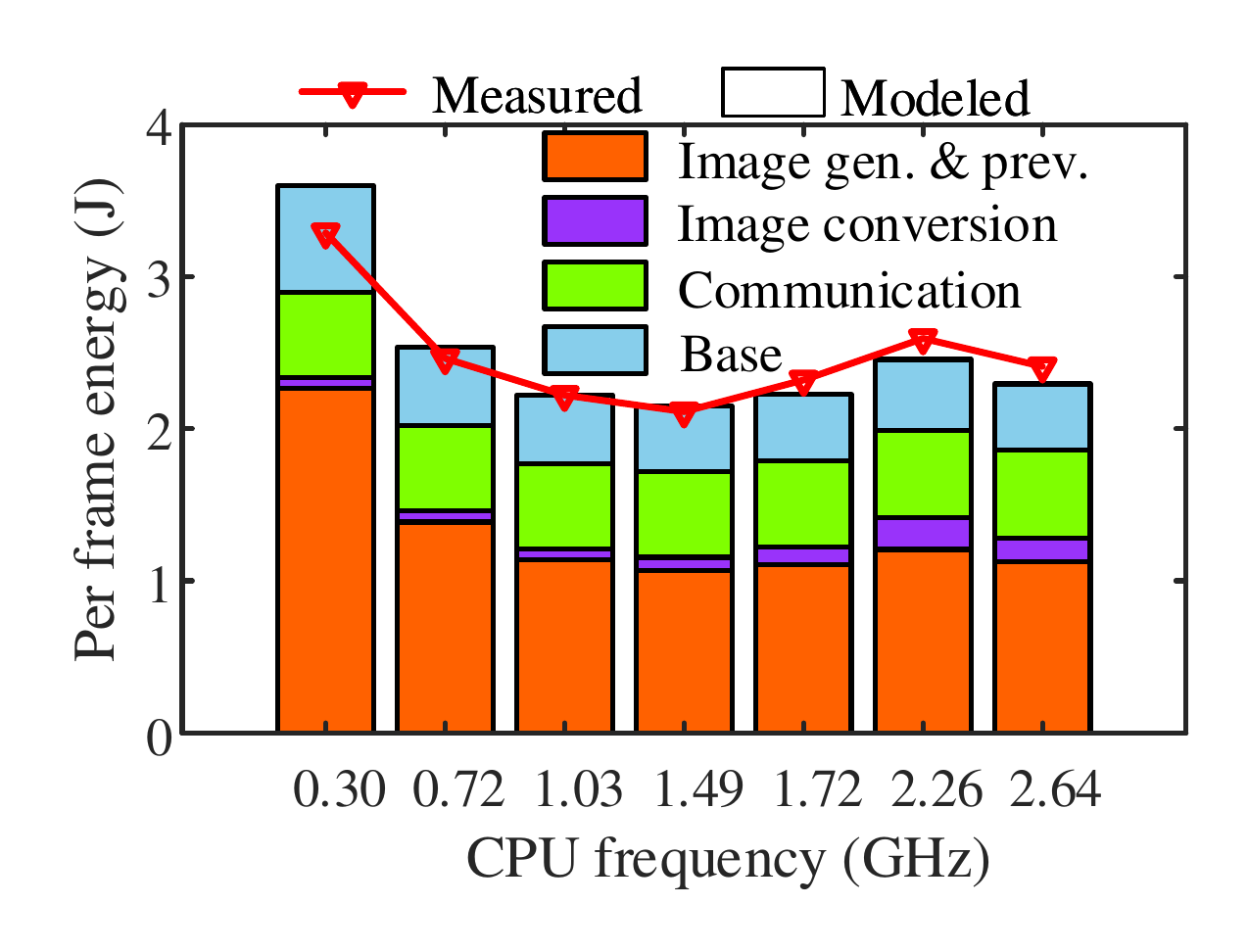}\label{fig:modelAf}}
\subfigure[]
{\includegraphics[width=0.24\textwidth]{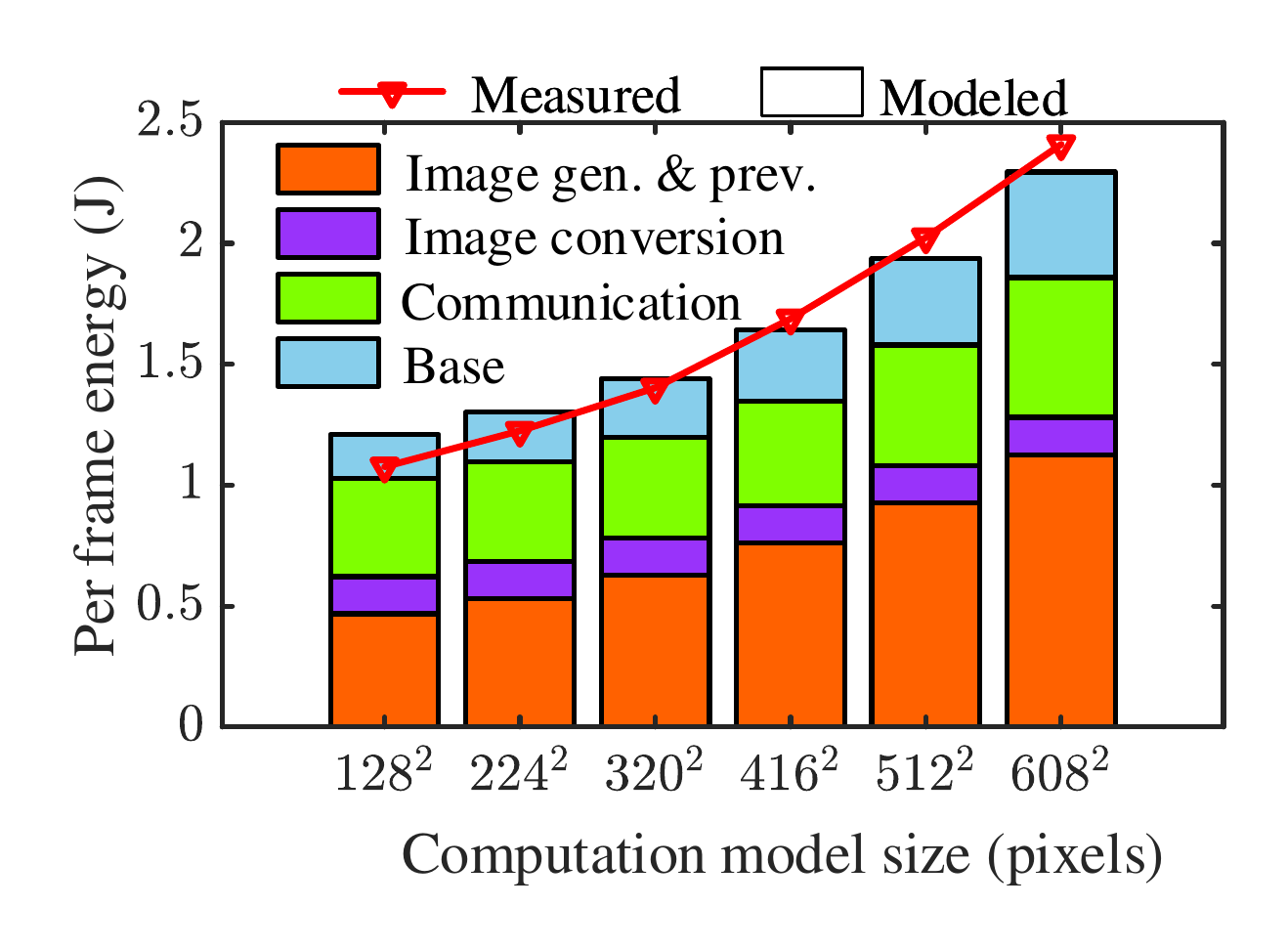}\label{fig:modelAs}}
\subfigure[]
{\includegraphics[width=0.24\textwidth]{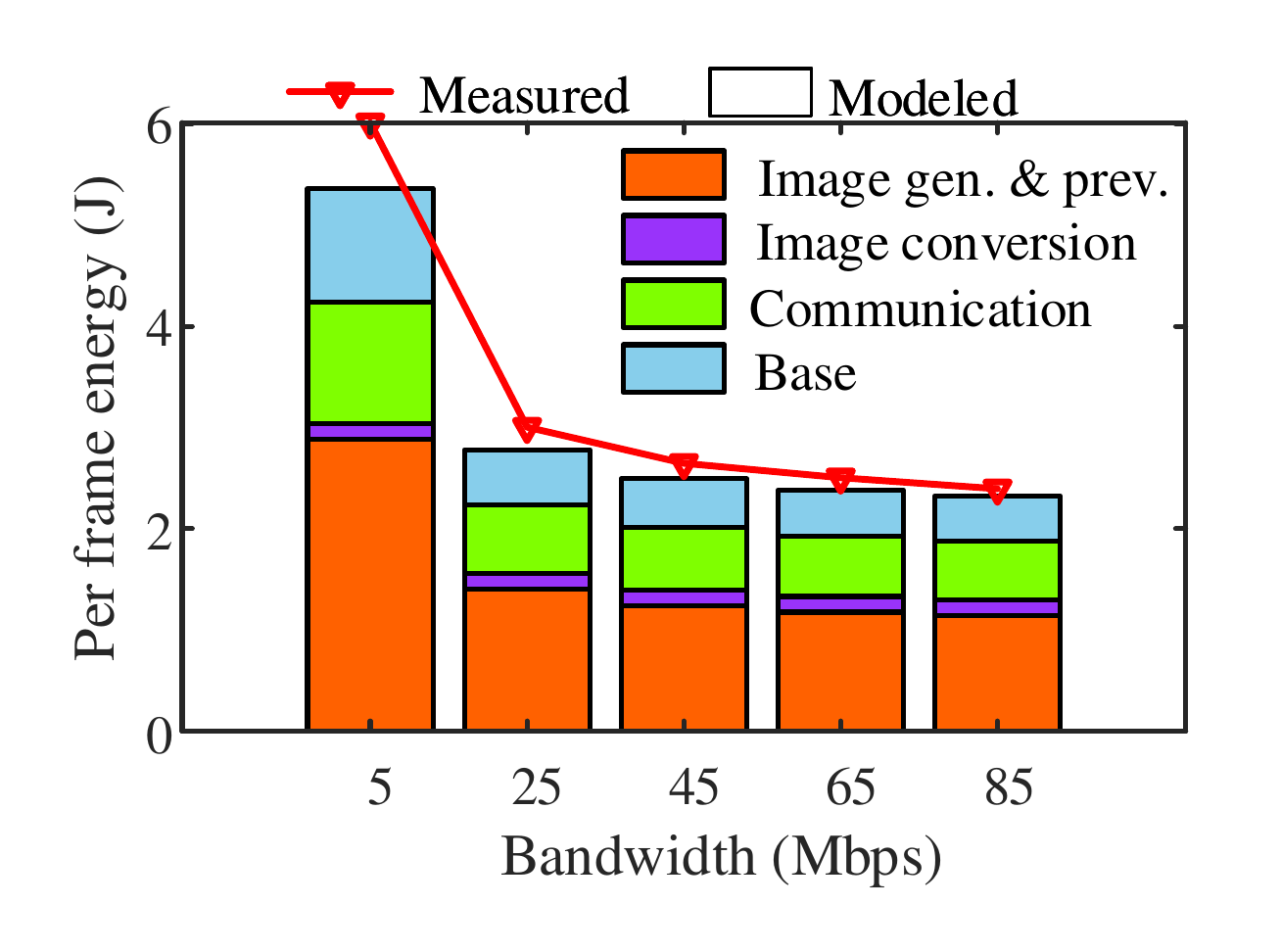}\label{fig:modelAb}}
\subfigure[]
{\includegraphics[width=0.24\textwidth]{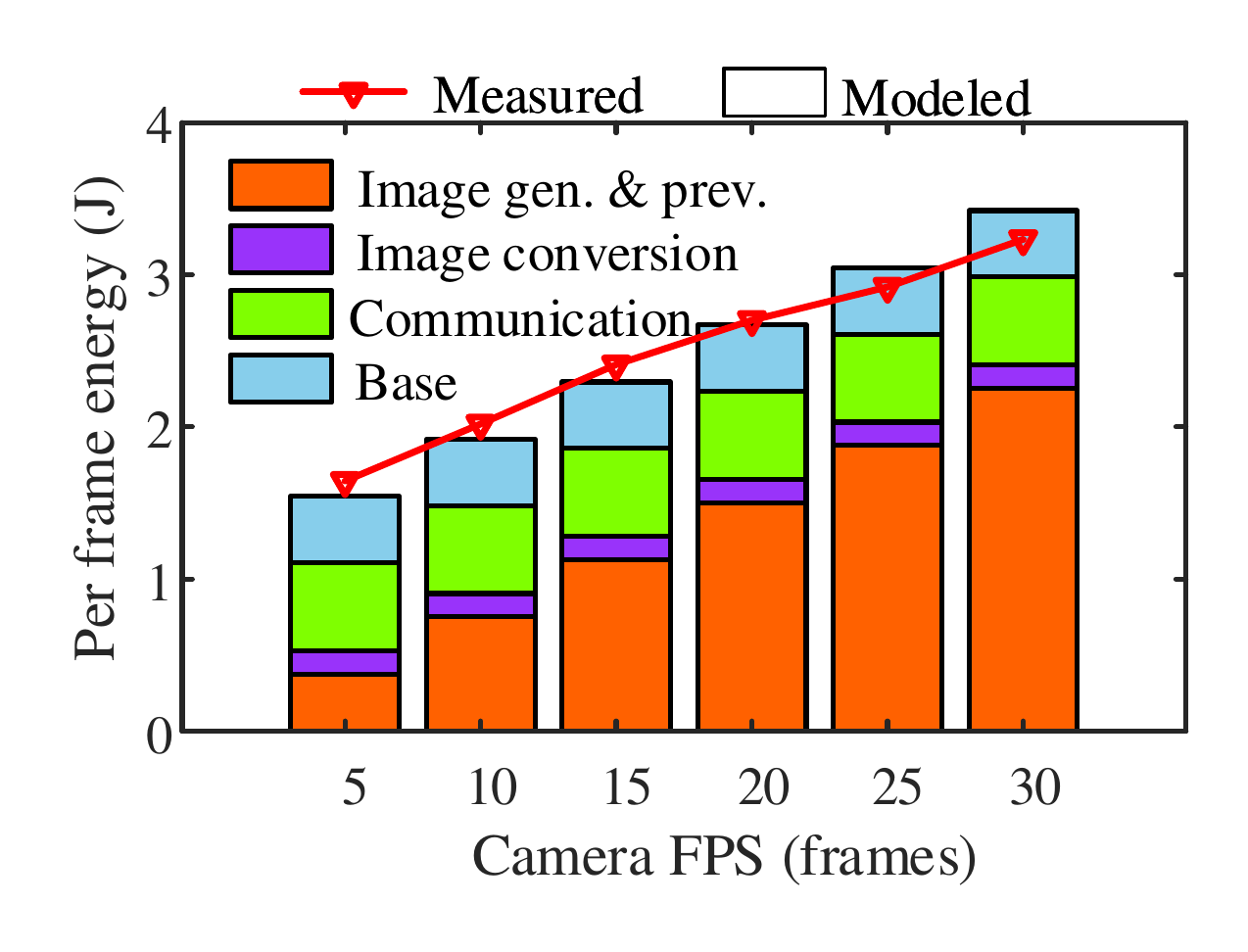}\label{fig:modelAfps}}
\vspace{-0.1in}
\caption{Measured data vs. estimated data from our proposed analytical model.}
\label{fig:modelV}   
\vspace{-0.15in}
\end{figure*}

\begin{figure}[t]
\centering
\subfigure[\first{$Q$ vs. Max. bandwidth}]
{\includegraphics[width=0.235\textwidth]{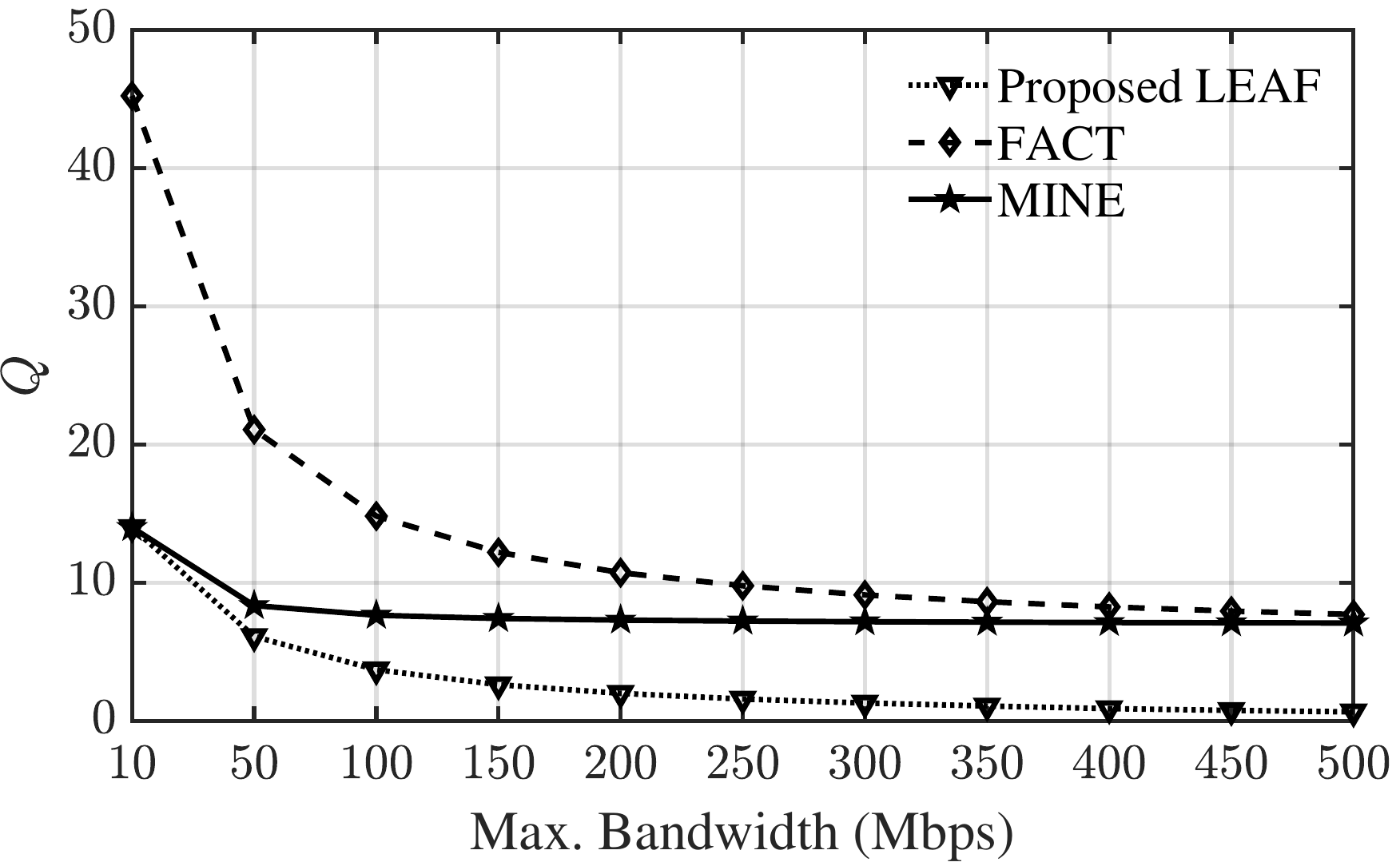}\label{fig:Q1}}
\subfigure[\first{$Q$ vs. user preference}]
{\includegraphics[width=0.24\textwidth]{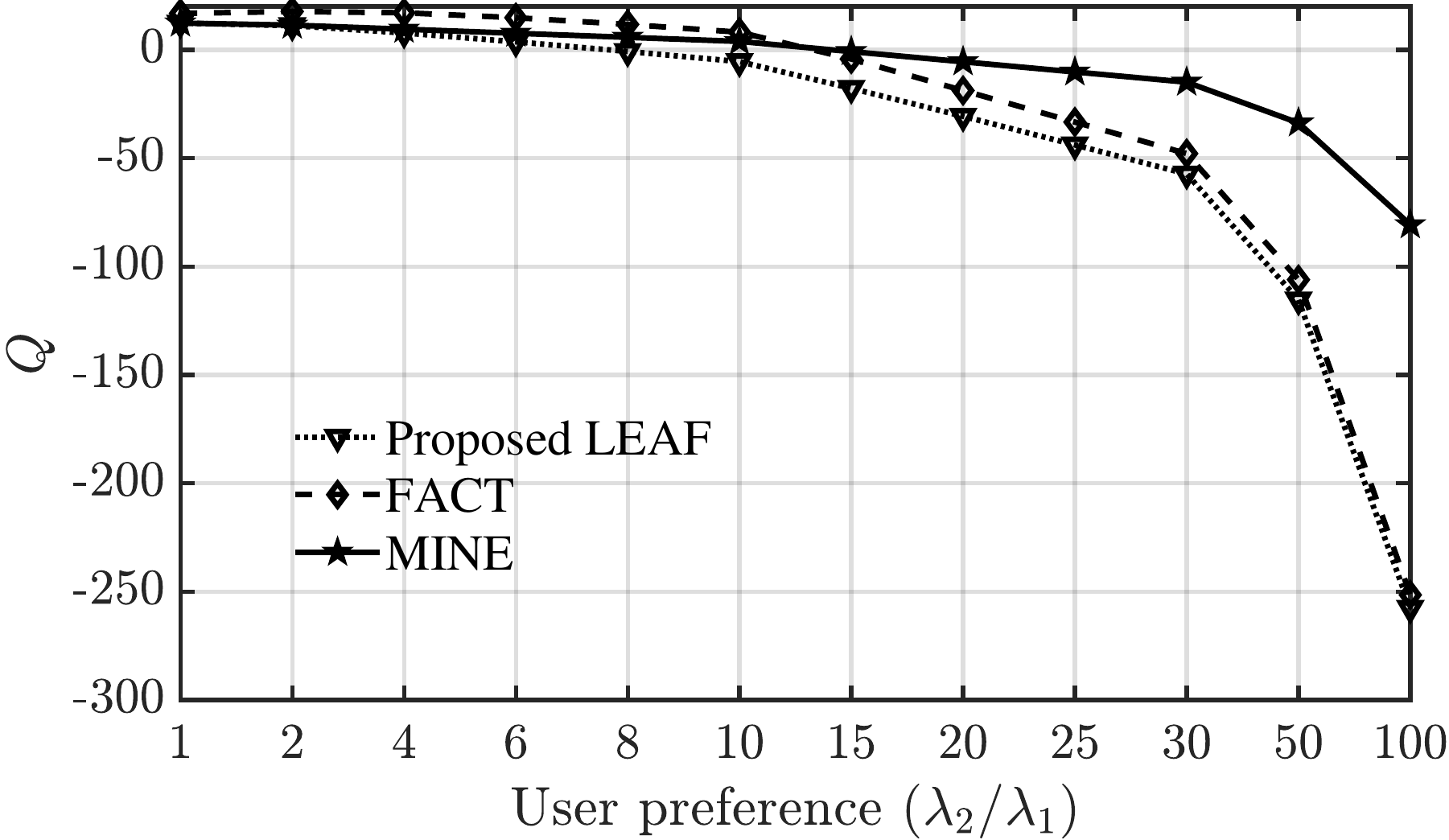}\label{fig:Q2}}
%\vspace{-0.1in}
\caption{\first{Optimality.}}
\label{fig:QPevaluation}   
\vspace{-0.15in}
\end{figure}

\begin{figure}[t]
\centering
\subfigure[]
{\includegraphics[width=0.225\textwidth]{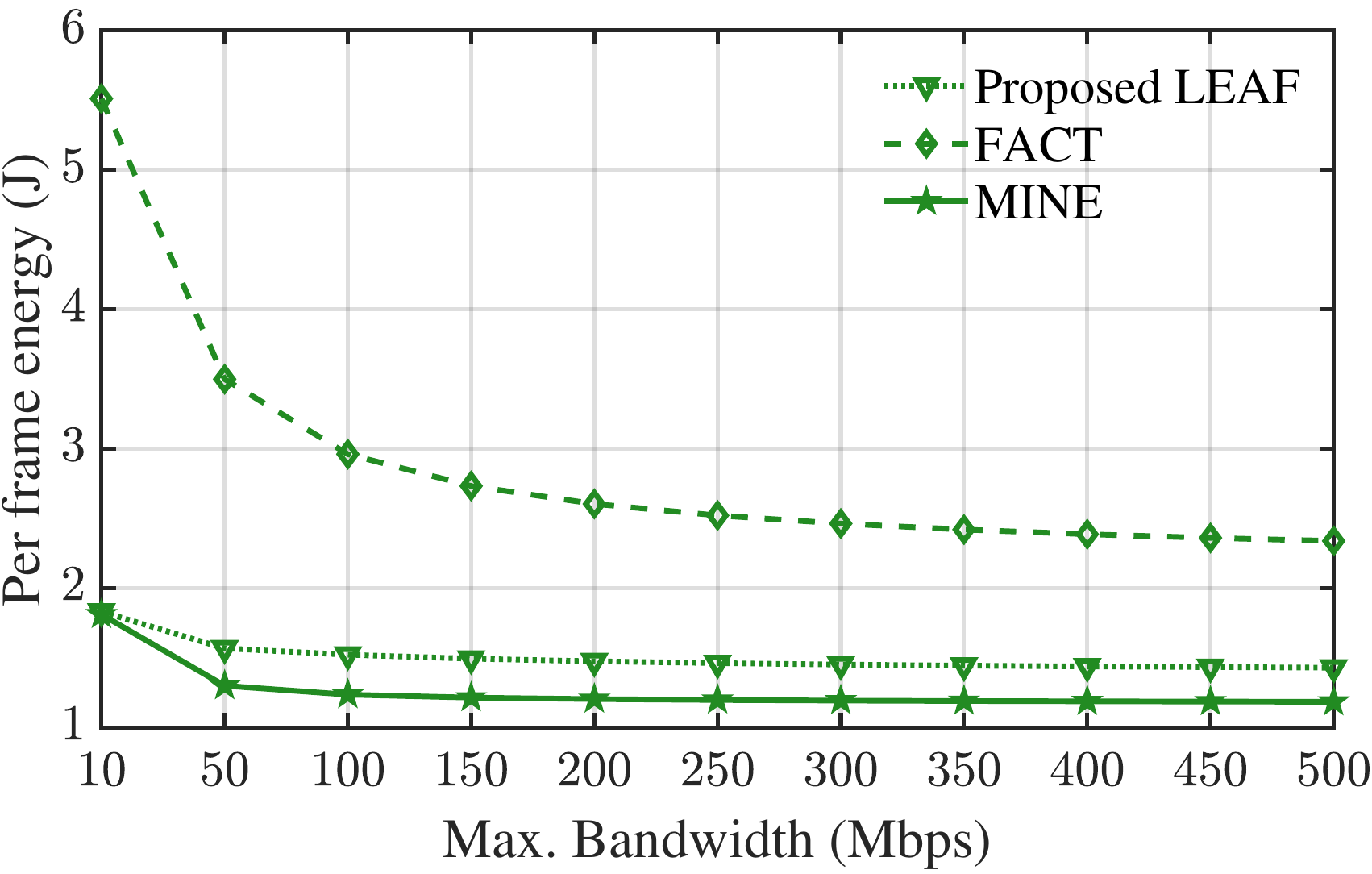}\label{fig:Eb}}
\subfigure[]
{\includegraphics[width=0.252\textwidth]{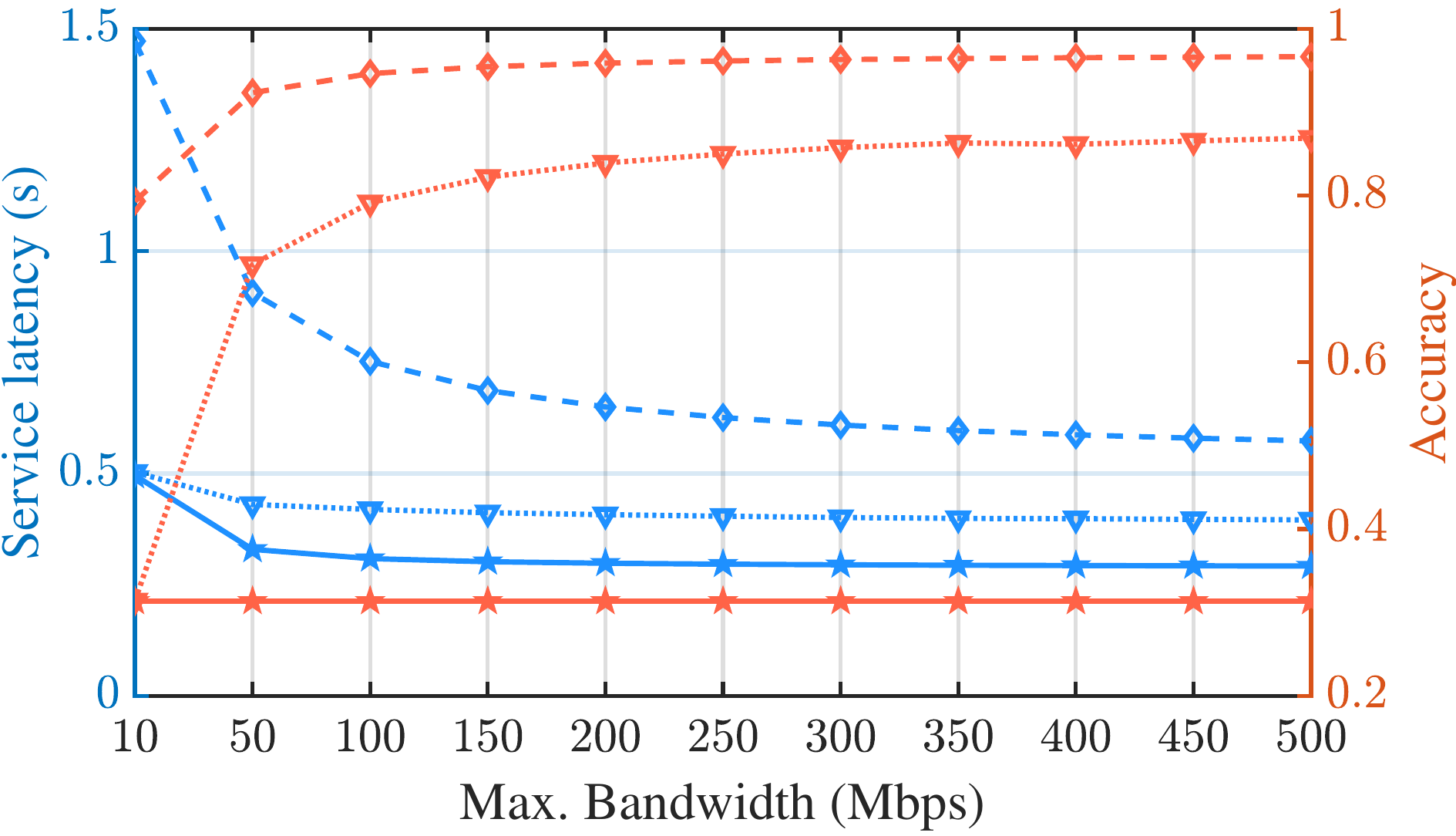}\label{fig:LAb}}
%\vspace{-0.15in}
\caption{\first {System performance vs. Max. bandwidth.}}
\label{fig:BPevaluation}   
\vspace{-0.15in}
\end{figure}

\begin{figure}[t]
\centering
\subfigure[]
{\includegraphics[width=0.23\textwidth]{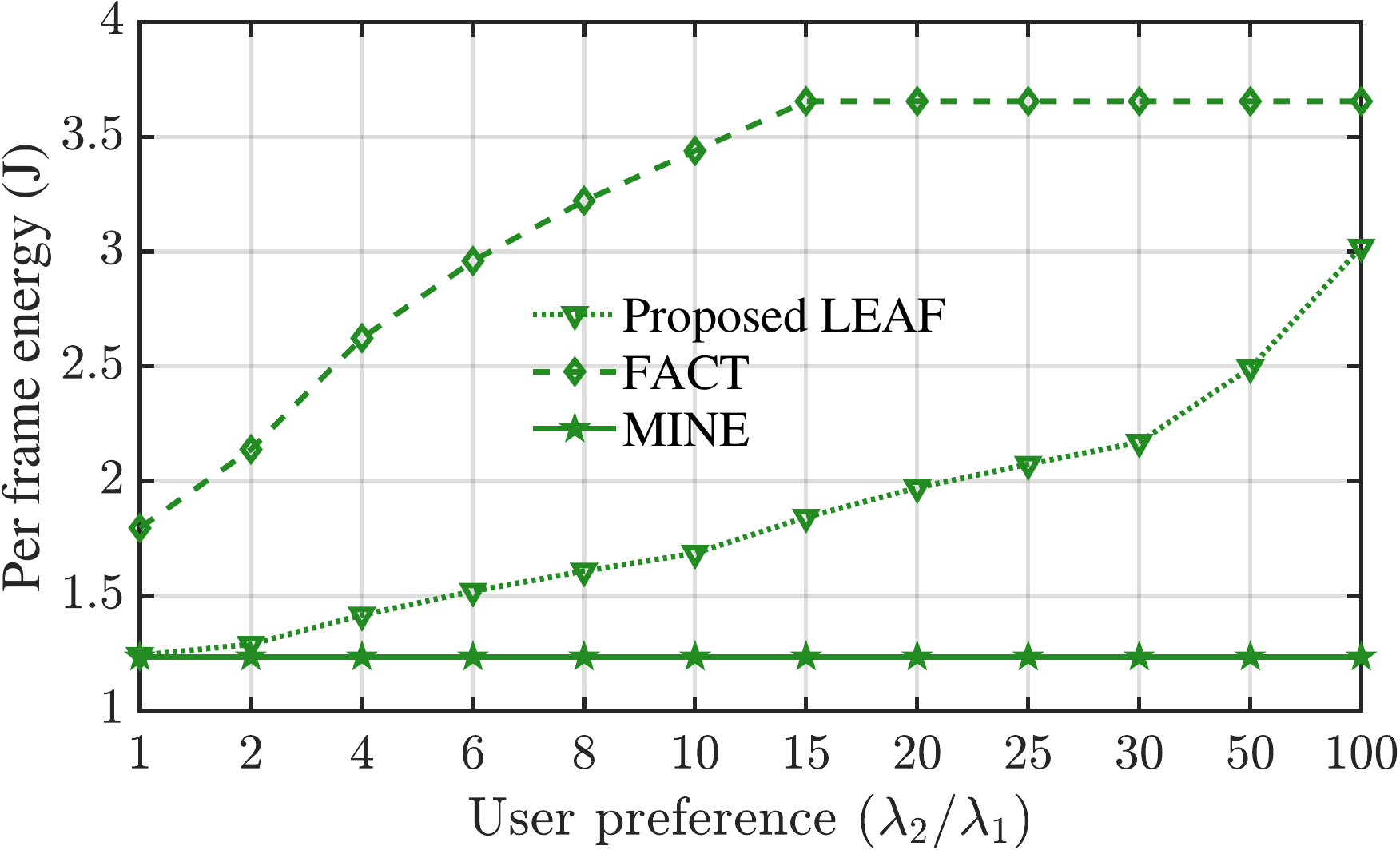}\label{fig:Ep}}
\subfigure[]
{\includegraphics[width=0.25\textwidth]{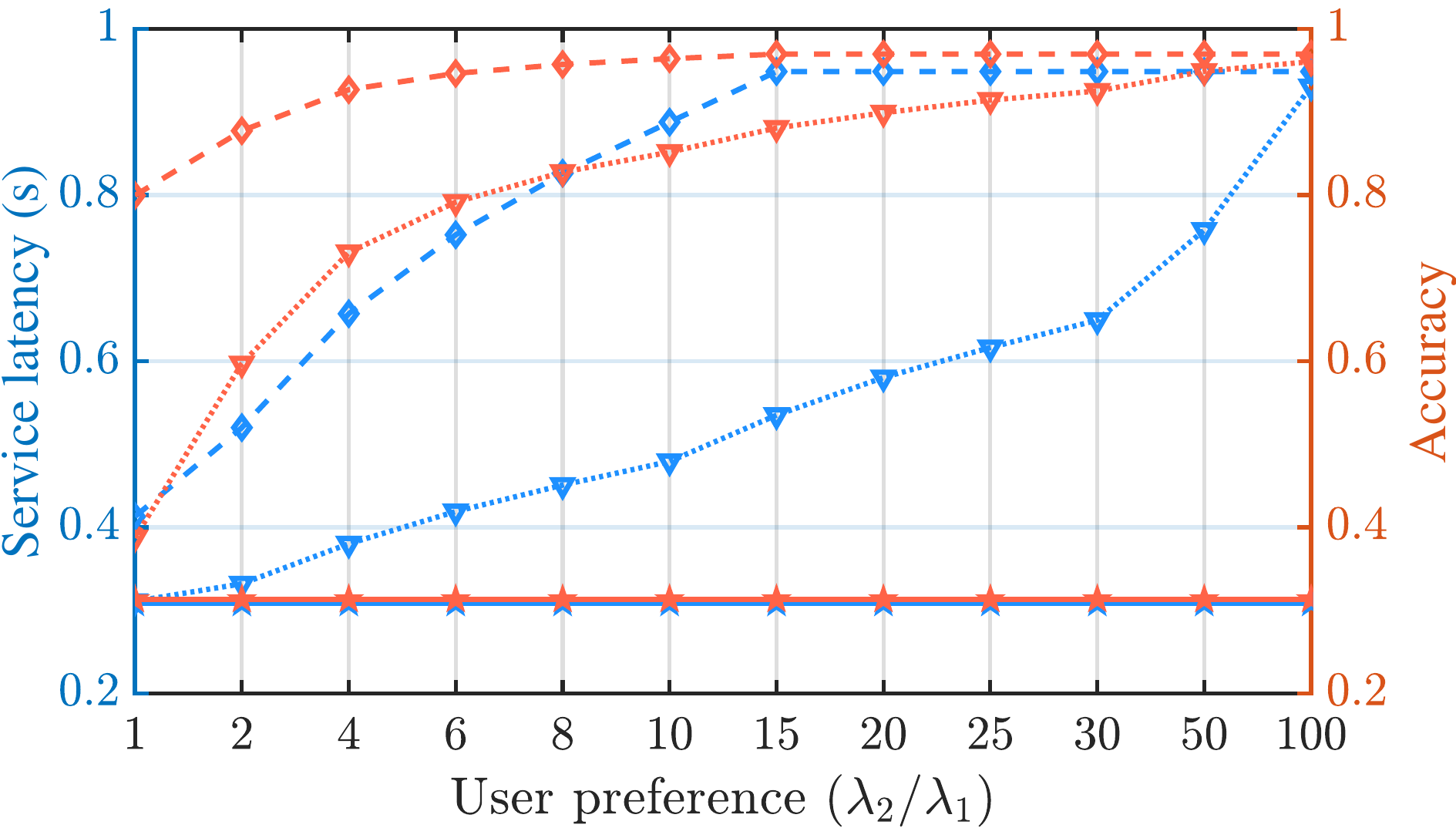}\label{fig:LAp}}
%\vspace{-0.15in}
\caption{\first {System performance vs. user preference.}}
\label{fig:UPevaluation}   
\vspace{-0.15in}
\end{figure}

\subsection{Performance Evaluation of LEAF}
We simulate an edge-based MAR system with an edge server and multiple MAR clients. Each MAR client may select a different camera FPS, which is obtained randomly in the range of $\left[1,30\right]$ frames\footnote{Ten MAR devices are implemented in our simulation, where their camera FPS are 9, 30, 16, 23, 14, 17, 13, 2, 19, 5.}. The default user preference is $\lambda_{1} = 0.3$ and $\lambda_{2} = 1.8$. We compare our proposed LEAF algorithm with two other algorithms summarized as follows:
\begin{itemize}
    \item \textbf{FACT + Interactive:} It uses the FACT algorithm \cite{liu2018edge} to select the computation model size, which is optimized for the tradeoff between the service latency and the detection accuracy. As FACT does not consider the MAR client's CPU frequency scaling and radio resource allocation at the edge server, we use \textit{Interactive} to conduct CPU frequency scaling and the radio resource is allocated evenly.
    Note that FACT does not consider the energy efficiency of MAR clients either.
    
    \item \textbf{Energy-optimized only solution:} It selects the optimal CPU frequency, computation model size, and bandwidth allocation by minimizing the per frame energy consumption of MAR clients in the system without considering user preferences, which is named as MINE.
\end{itemize}

\textbf{Optimality.} We first validate the optimality of our proposed LEAF algorithm. As shown in Fig. \ref{fig:QPevaluation}, LEAF always obtains the minimal $Q$ compared to the other two algorithms under variant maximum available bandwidth and user preference.

\textbf{Comparison under Variant Max. Bandwidth.} We then evaluate the impact of the maximum available bandwidth on the performance of the proposed LEAF. As presented in Section \ref{ssc:formulation}, in practical environments, the maximum bandwidth at an edge server for serving its associated MAR clients may vary with the user distribution. For each MAR client, the value of the allocated bandwidth directly impacts not only the service latency and the per frame energy consumption but also the detection accuracy. The evaluation results are depicted in Fig. \ref{fig:BPevaluation}. (i) Compared to FACT, the proposed LEAF decreases up to $40\%$ per frame energy consumption and $35\%$ service latency with less than $9\%$ loss of object detection accuracy when the max. bandwidth is $300$ Mbps. The performance gap between LEAF and FACT is due to the gain derived through optimizing the clients' CPU frequency and the server radio resource allocation. (ii) Compared to MINE, the proposed LEAF significantly improves the detection accuracy at the cost of a slightly increase of the service latency and per frame energy. The performance gap between LEAF and MINE reflects the gain derived through considering the user preference.

\textbf{Comparison under Variant User Preferences.}
Finally, we evaluate the impact of the user preference on the performance of the proposed LEAF by varying the value of $\lambda_{2}/\lambda_{1}$, as shown in Fig. \ref{fig:UPevaluation}. User preference impacts the tradeoffs among the per frame energy consumption, service latency, and detection accuracy. When $\lambda_{2}/\lambda_{1}$ grows, the MAR client emphasizes on the detection accuracy by trading the service latency and per frame energy. Since MINE does not consider the user preference, the variation of $\lambda_{2}/\lambda_{1}$ does not change its performance. (i) Compared to FACT, the proposed LEAF reduces over $20\%$ per frame energy consumption while maintaining the same detection accuracy ($\lambda_{2}/\lambda_{1}=100$). (ii) Compared to MINE, the proposed LEAF is able to enhance over $50\%$ accuracy while ensuring similar per frame energy and service latency ($\lambda_{2}/\lambda_{1}=2$). Fig. \ref{fig:UPevaluation} also shows that, as compared to FACT, the proposed LEAF offers more fine-grained and diverse user preference options for MAR clients.

\begin{figure*}[t]
\centering
\subfigure[]
{\includegraphics[width=0.24\textwidth]{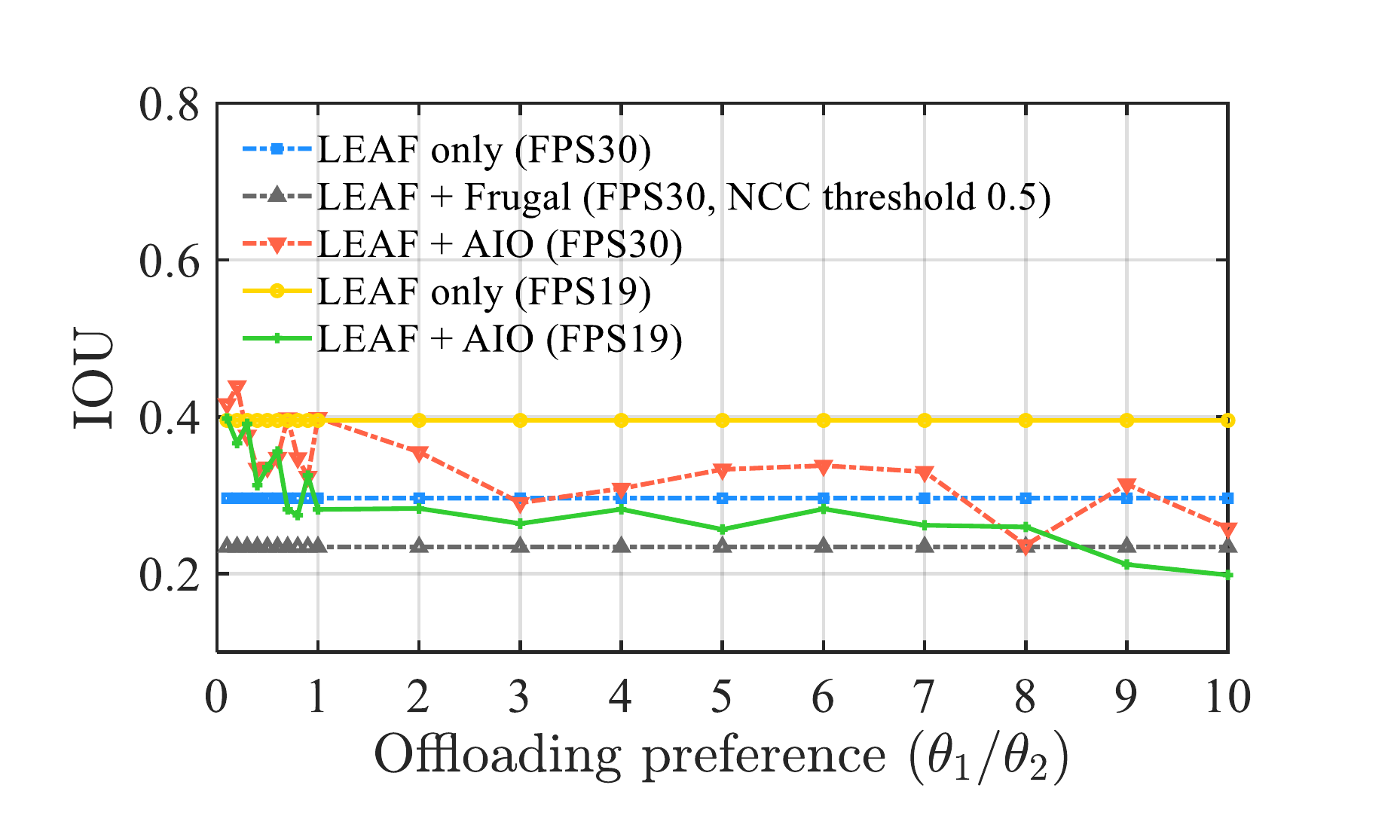}\label{fig:IOUvsa1a2}}
\subfigure[]
{\includegraphics[width=0.24\textwidth]{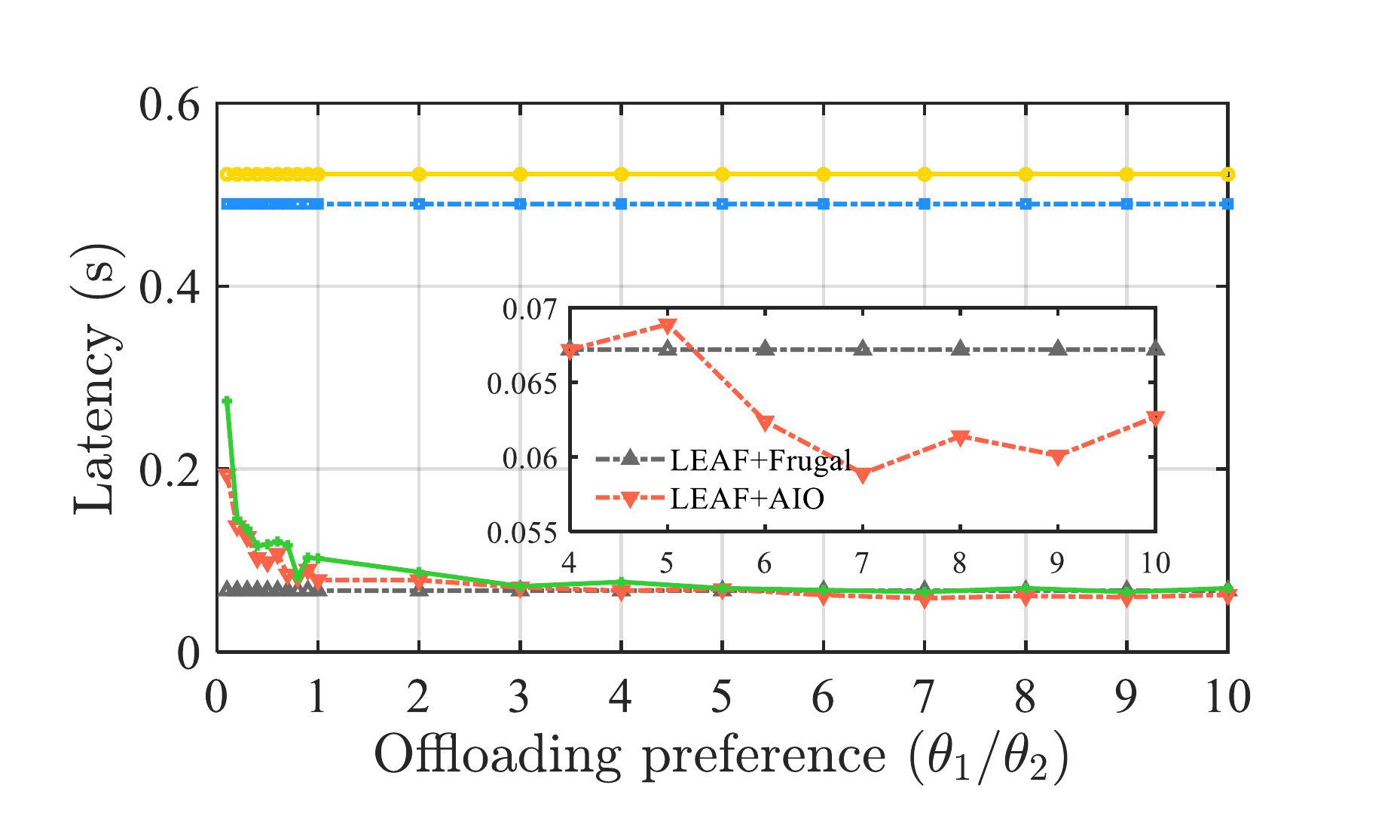}\label{fig:latencyvsa1a2}}
\subfigure[]
{\includegraphics[width=0.23\textwidth]{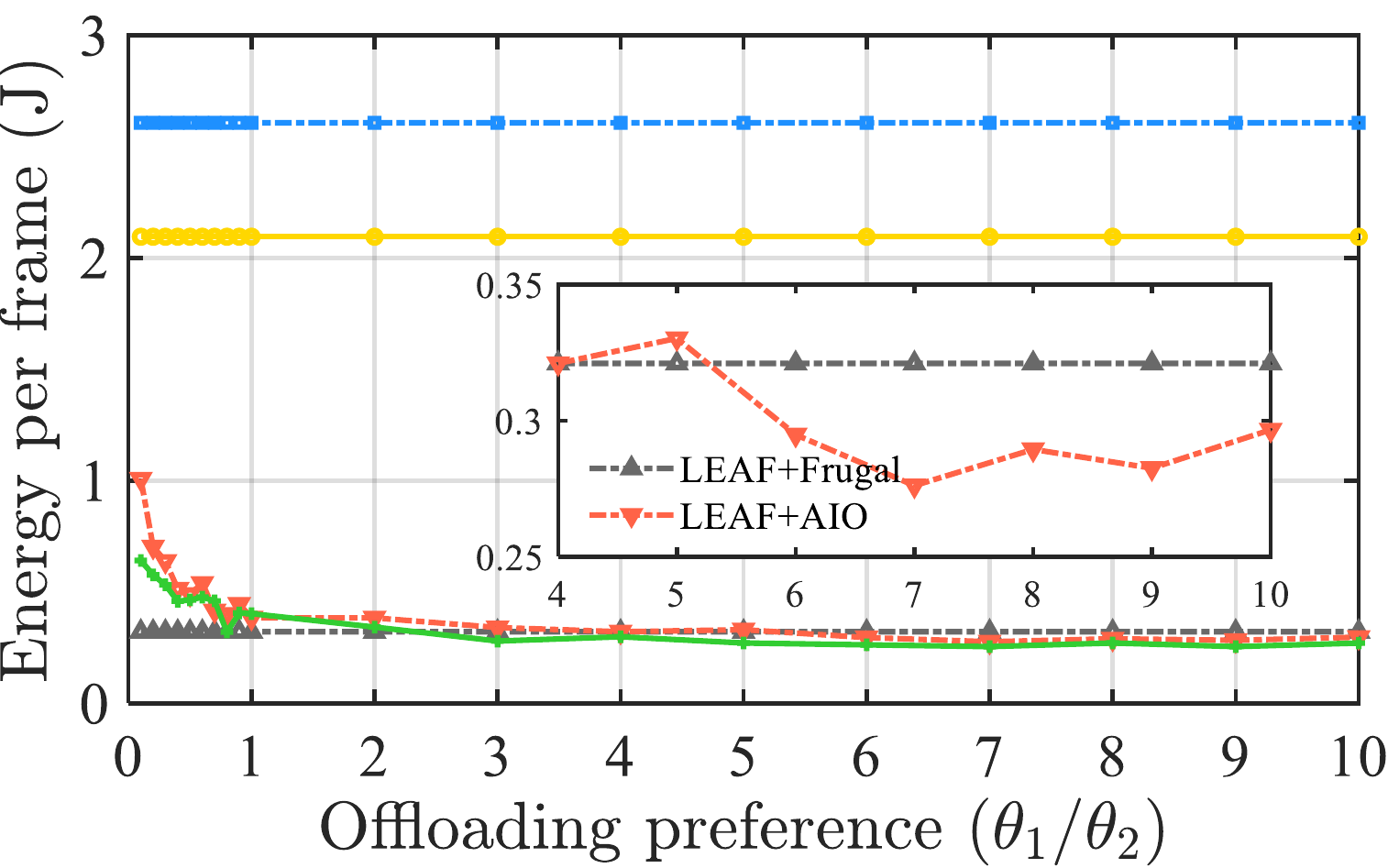}\label{fig:energyvsa1a2}}
\subfigure[]
{\includegraphics[width=0.23\textwidth]{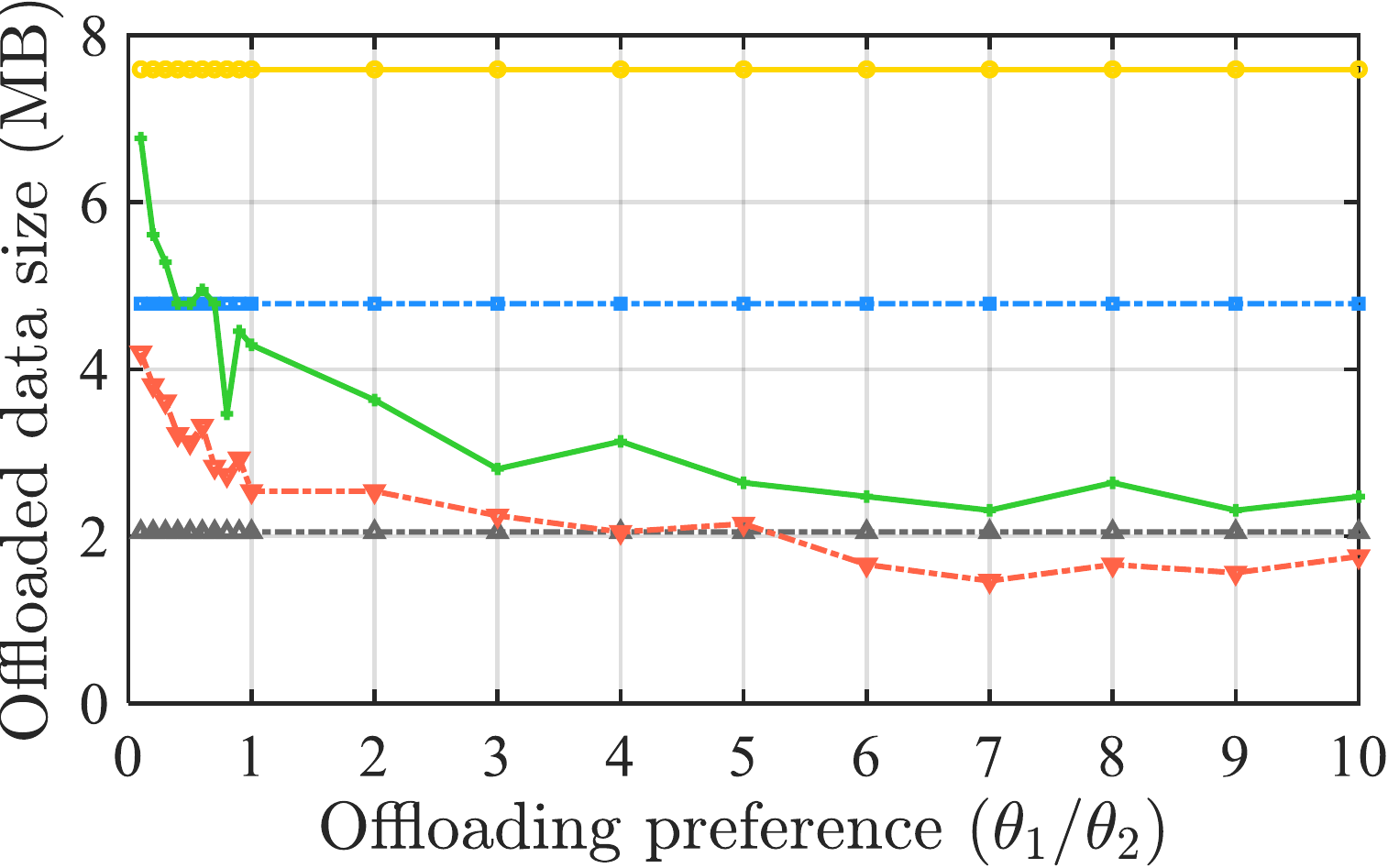}\label{fig:datasizevsa1a2}}
\vspace{-0.1in}
\caption{System performance vs. offloading preference.}
\label{fig:AIOevaluation}   
\vspace{-0.15in}
\end{figure*}

\subsection{Performance Evaluation of LEAF + AIO}
We implement our proposed image offloading orchestrator with the AIO in a device-to-device testbed that consists of an MAR device and an edge server. The MAR device works with the configurations achieved from the edge server and optimized by the proposed LEAF. In the experiment, we choose two MAR clients with camera FPS $30$ and $19$ for the evaluation. To make the experiment repeatable, we leverage the video frames from the open dataset \cite{WuLimYang13} as the source of data ingestion. In addition, we compare our proposed AIO algorithm integrated with the LEAF with two other baselines summarized as follows:
\begin{itemize}
    \item \textbf{LEAF + Frugal} \cite{apicharttrisorn2019frugal}\textbf{:} It uses a preset normalized cross-correlation (NCC) threshold to trigger object detection invocations. The value of NCC threshold is set to $0.5$ which is experience-driven.  
    \item \textbf{LEAF only:} The MAR device offloads its camera captured image frames as many as possible and no local object tracker is deployed. 
\end{itemize}

\textbf{Comparison under Variant Offloading Preferences.} We evaluate the impact of the offloading preference on the performance of the proposed AIO by varying the value of $\theta_{1}/\theta_{2}$, as illustrated in Fig. \ref{fig:AIOevaluation}. Offloading preference influences the tradeoffs between the perception accuracy and energy efficiency of MAR devices. When $\theta_{1}/\theta_{2}$ increases, the image offloading orchestrator emphasizes on the energy efficiency by trading the perception accuracy. As Frugal sets the same trigger value for all scenarios and LEAF only does not consider the adaptive image offloading decision, the variation of $\theta_{1}/\theta_{2}$ does not change their performance. (i) Compared to Frugal, our proposed AIO improves the average IOU by $43\%$ while decreasing the average service latency and per frame energy consumption by $12.3\%$ and $13.9\%$, respectively ($\theta_{1}/\theta_{2} = 7$). (ii) Compared to the LEAF only, our integration system not only significantly drops the latency and offloaded data size but also further improves IOU and energy efficiency of MAR devices.

\section{Conclusion}
\label{sc:conclusion}
In this paper, we proposed a user preference based energy-aware edge-based MAR system for object detection that can reduce the per frame energy consumption of MAR clients without compromising their user preferences by dynamically selecting the optimal combination of MAR configurations and radio resource allocations according to user preferences, camera FPS, and available radio resources at the edge server. To the best of our knowledge, we built the first analytical energy model for thoroughly investigating the interactions among MAR configuration parameters, user preferences, camera sampling rate, and per frame energy consumption in edge-based MAR systems. Based on the proposed analytical model, we proposed the LEAF optimization algorithm to guide the optimal MAR configurations and resource allocations. The performance of the proposed analytical model is validated against real energy measurements from our testbed and the LEAF algorithm is evaluated through extensive data-driven simulations. Additionally, we studied and implemented object tracking to further improve the energy efficiency of our proposed edge-based MAR system.

\IEEEpeerreviewmaketitle

\bibliographystyle{IEEEtran}
\bibliography{references}

\begin{IEEEbiography}[{\includegraphics[width=1in,height=1.25in,clip,keepaspectratio]{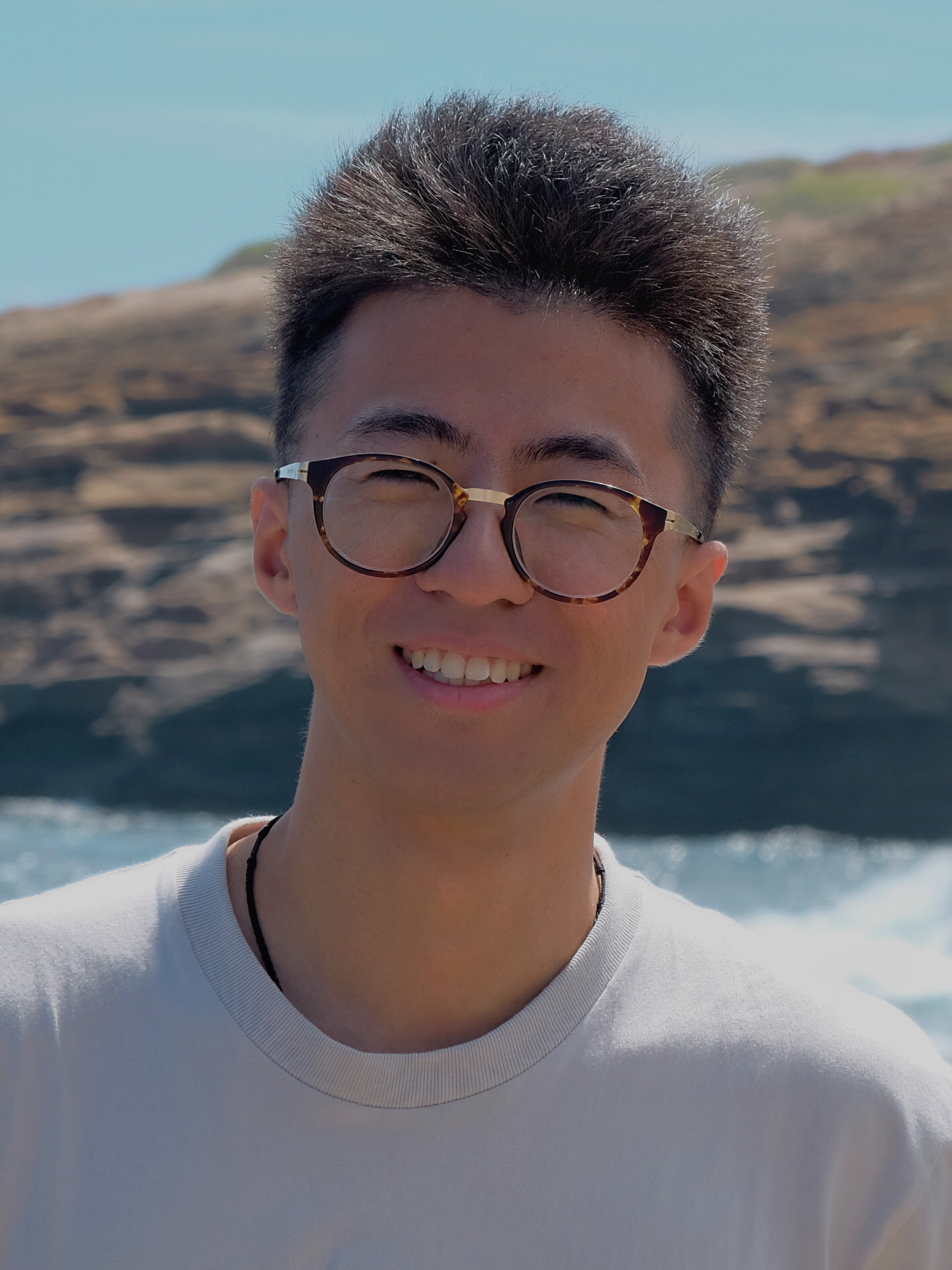}}]{Haoxin Wang}
(S'15-M'21) received the Ph.D. degree in electrical and computer engineering from The University of North Carolina at Charlotte in 2020, and the B.S. degree in control science and engineering from Harbin Institute of Technology in China in 2015. He is currently a research scientist at Toyota Motor North America, InfoTech Labs, where he leads the ``Edge Computing'' project. His research interests include edge computing for connected and autonomous vehicles, applied machine learning for intelligent systems, and energy-efficient mobile computing systems.
\end{IEEEbiography}

\begin{IEEEbiography}[{\includegraphics[width=1in,height=1.25in,clip,keepaspectratio]{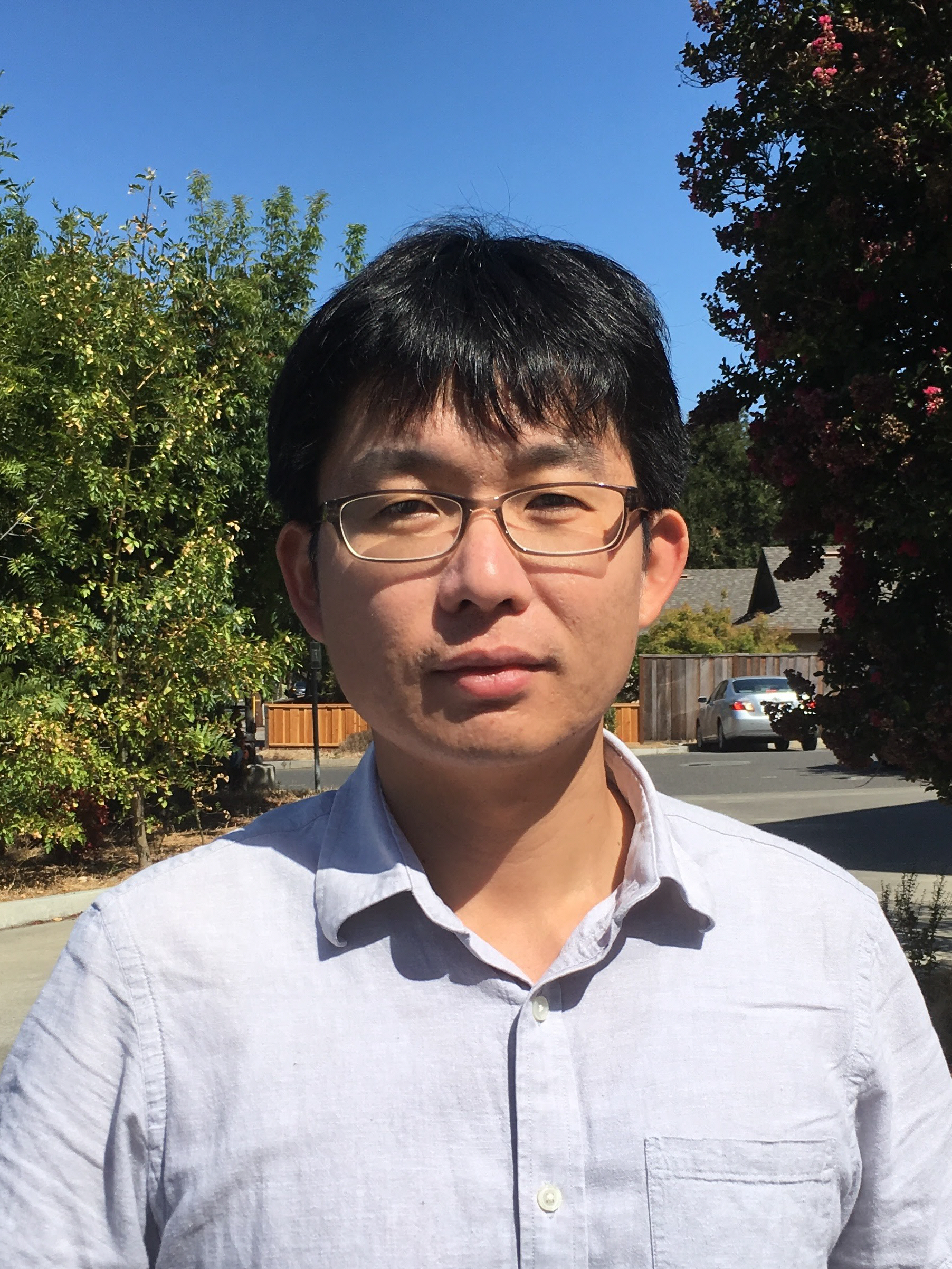}}]{BaekGyu Kim}
earned B.S. and M.S. in electrical engineering and computer science from Kyungpook National University in South Korea, and earned Ph.D in computer science from University of Pennsylvania. He is currently a principal researcher in Toyota Motor North America, InfoTech Labs, and his research area includes software platform technologies for connected cars, and model based software development for high-assurance systems.
\end{IEEEbiography}

\begin{IEEEbiography}[{\includegraphics[width=1in,height=1.25in,clip,keepaspectratio]{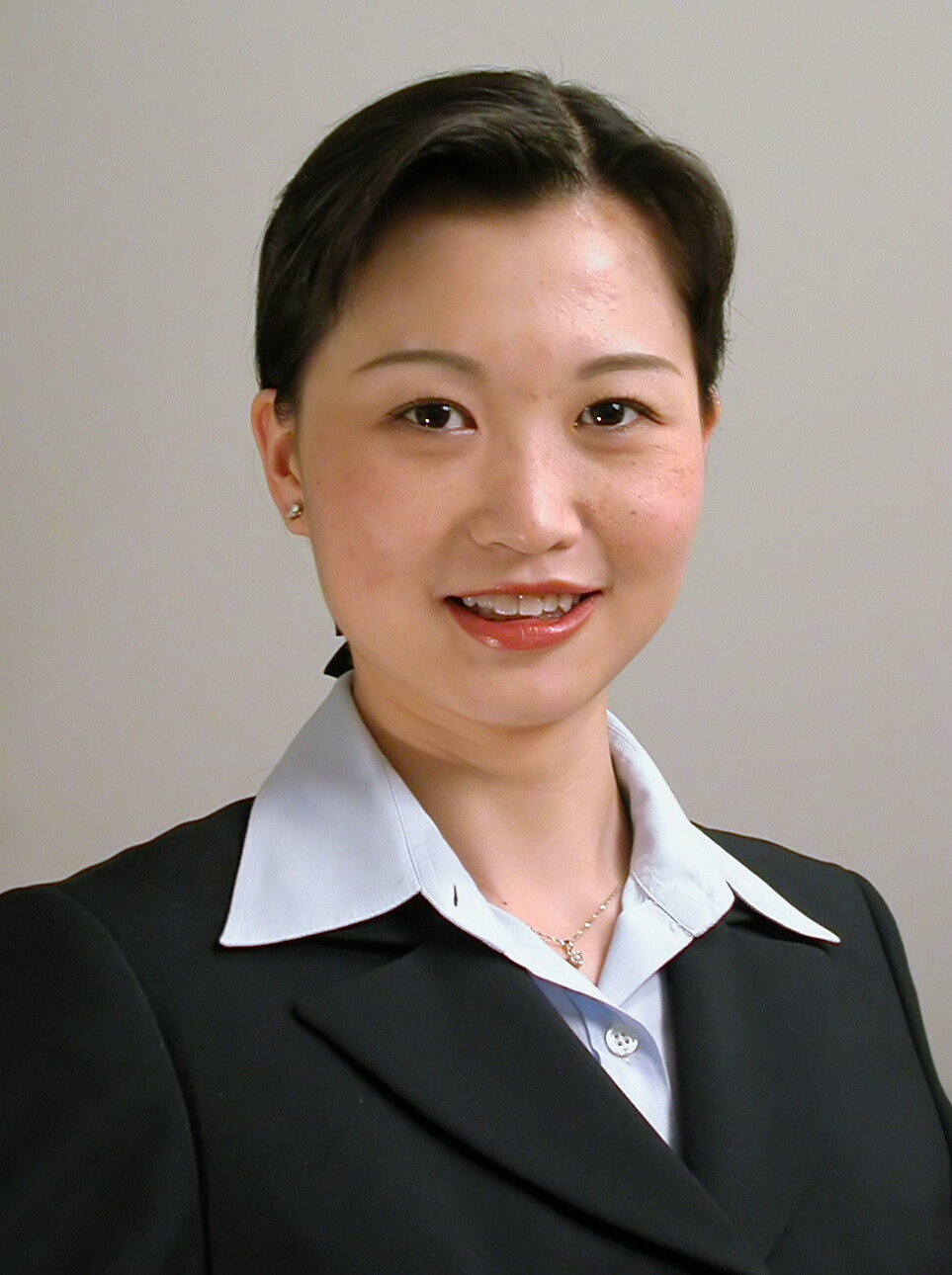}}]{Jiang Xie}
received the B.E. degree from Tsinghua University, Beijing, China, the MPhil degree from the Hong Kong University of Science and Technology, and the M.S. and PhD degrees from Georgia Institute of Technology, all in electrical and computer engineering. She joined the Department of Electrical and Computer Engineering at the University of North Carolina at Charlotte (UNC Charlotte) as an Assistant Professor in August 2004, where she is currently a Full Professor. Her current research interests include resource and mobility management in wireless networks, mobile computing, Internet of Things, and cloud/edge computing. She is on the Editorial Boards of the IEEE Transactions on Wireless Communications, IEEE Transactions on Sustainable Computing, and Journal of Network and Computer Applications (Elsevier). She received the US National Science Foundation (NSF) Faculty Early Career Development (CAREER) Award in 2010, a Best Paper Award from IEEE Global Communications Conference (Globecom 2017), a Best Paper Award from IEEE/WIC/ACM International Conference on Intelligent Agent Technology (IAT 2010), and a Graduate Teaching Excellence Award from the College of Engineering at UNC Charlotte in 2007. She is a fellow of the IEEE and a senior member of the ACM.
\end{IEEEbiography}

\begin{IEEEbiography}[{\includegraphics[width=1in,height=1.25in,clip,keepaspectratio]{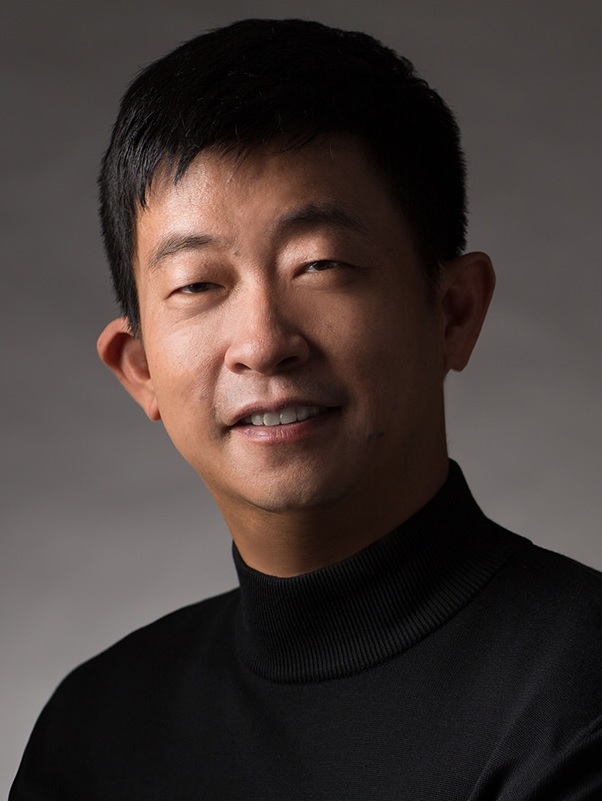}}]{Zhu Han}
(S’01–M’04-SM’09-F’14) received the B.S. degree in electronic engineering from Tsinghua University, in 1997, and the M.S. and Ph.D. degrees in electrical and computer engineering from the University of Maryland, College Park, in 1999 and 2003, respectively. 

From 2000 to 2002, he was an R\&D Engineer of JDSU, Germantown, Maryland. From 2003 to 2006, he was a Research Associate at the University of Maryland. From 2006 to 2008, he was an assistant professor at Boise State University, Idaho. Currently, he is a John and Rebecca Moores Professor in the Electrical and Computer Engineering Department as well as in the Computer Science Department at the University of Houston, Texas. His research interests include wireless resource allocation and management, wireless communications and networking, game theory, big data analysis, security, and smart grid.  Dr. Han received an NSF Career Award in 2010, the Fred W. Ellersick Prize of the IEEE Communication Society in 2011, the EURASIP Best Paper Award for the Journal on Advances in Signal Processing in 2015, IEEE Leonard G. Abraham Prize in the field of Communications Systems (best paper award in IEEE JSAC) in 2016, and several best paper awards in IEEE conferences. Dr. Han was an IEEE Communications Society Distinguished Lecturer from 2015-2018, AAAS fellow since 2019, and ACM distinguished Member since 2019. Dr. Han is a 1\% highly cited researcher since 2017 according to Web of Science. Dr. Han is also the winner of the 2021 IEEE Kiyo Tomiyasu Award, for outstanding early to mid-career contributions to technologies holding the promise of innovative applications, with the following citation: ``for contributions to game theory and distributed management of autonomous communication networks."
\end{IEEEbiography}
\end{document}